\documentclass[twoside,11pt]{article}

\usepackage{subfigure}
\usepackage{float}
\usepackage[all]{xy}
\usepackage{enumitem}
\setlist[enumerate]{itemsep=0mm}
\usepackage{color}
\usepackage{jmlr2e}
\usepackage{graphics}
\usepackage{epsfig}
\usepackage[ruled]{algorithm2e}
\usepackage{booktabs}
\usepackage{url}
\usepackage{psfrag}
\usepackage{relsize}
\usepackage{amsmath,amsfonts,mathrsfs,mathtools,amssymb}
\usepackage[colorlinks,linkcolor=blue,anchorcolor=green,citecolor=green]{hyperref}
\usepackage{tikz,pgfplots}
\usepackage{tkz-tab}
\usepackage[format=hang,justification=justified,singlelinecheck=false,font=footnotesize]{caption}

\usepackage[T1]{fontenc}
\usepackage[utf8]{inputenc}
\usepackage[font=small,labelfont=bf,tableposition=top]{caption}
\usepackage{booktabs}
\usepackage{threeparttable} 
\usepackage{multirow}

\newcommand{\eins}{\boldsymbol{1}}

\DeclareSymbolFont{wideparensymbol}{OMX}{yhex}{m}{n}
\DeclareMathAccent{\wideparen}{\mathord}{wideparensymbol}{"F3}

\newcommand{\argmin}{\operatornamewithlimits{arg \, min}}

\newcommand{\supp}{\mathrm{supp}}

\makeatletter

\newcommand{\Rmnum}[1]{\expandafter\@slowromancap\romannumeral #1@}
\makeatother

\jmlrheading{}{}{~}{~}{~}{2019 Hanyuan Hang,  Zhouchen Lin, Xiaoyu Liu, and Hongwei Wen}

\usepackage{tikz,pgfplots}
\pgfplotsset{compat=newest}
\pgfplotsset{plot coordinates/math parser=false,trim axis left}
\newlength\figureheight
\newlength\figurewidth

\begin{document}

\title{Histogram Transform Ensembles \\ for Large-scale Regression}

\author{\name Hanyuan Hang \email hanyuan.hang@samsung.com \\
\name Zhouchen Lin \email zhouchen.lin@samsung.com \\
\name Xiaoyu Liu \email  xiaoyu1.liu@partner.samsung.com \\
\name Hongwei Wen \email hongwei.wen@partner.samsung.com \\
\addr AI Lab \\
Samsung Research China - Beijing\\
100028 Beijing, China 
}


\allowdisplaybreaks

\maketitle


\begin{abstract}We propose a novel algorithm for large-scale regression problems named histogram transform ensembles (HTE), composed of random rotations, stretchings, and translations. 
First of all, we investigate the theoretical properties of HTE when the regression function lies in the H\"{o}lder space $C^{k,\alpha}$, $k \in \mathbb{N}_0$, $\alpha \in (0,1]$. In the case that $k=0, 1$, we adopt the constant regressors and develop the na\"{i}ve histogram transforms (NHT). Within the space $C^{0,\alpha}$, although almost optimal convergence rates can be derived for both single and ensemble NHT, we fail to show  the benefits of ensembles over single estimators theoretically. In contrast, in the subspace $C^{1,\alpha}$, we prove that if $d \geq 2(1+\alpha)/\alpha$, the lower bound of the convergence rates for single NHT turns out to be worse than the upper bound of the convergence rates for ensemble NHT. In the other case when $k \geq 2$, the NHT may no longer be appropriate in predicting smoother regression functions. Instead, we apply kernel histogram transforms (KHT) equipped with smoother regressors such as support vector machines (SVMs), and it turns out that both single and ensemble KHT  enjoy almost optimal convergence rates. 
Then we validate the above theoretical results by numerical experiments.
On the one hand, simulations are conducted to elucidate that ensemble NHT outperform single NHT. 
On the other hand, the effects of bin sizes on accuracy of both NHT and KHT also accord with theoretical analysis.
Last but not least, in the real-data experiments, comparisons between the ensemble KHT, equipped with adaptive histogram transforms, and other state-of-the-art large-scale regression estimators verify the effectiveness and accuracy of our algorithm. 
\end{abstract}


\section{Introduction}

In the era of big data, with rapid development of information technology, especially the processing power and memory storage in automatic data generation and acquisition, the size and complexity of data sets are constantly advancing to a unprecedented degree \citep{Zhou2014Big}. In this context, from a real-world applicable perspective, learning algorithms that not only maintain desirable prediction accuracy but also achieve high computational efficiency are urgently needed \citep{Wen2018ThunderSVM, Guo2018Efficient, Philipp2017Spatial, Cho2014Divide}. Among common machine learning tasks, in this paper, we are interested in the large-scale nonparametric regression problem which aims at inferring the functional relation between input and output. One major challenge, however, is that existing learning algorithms turn out to be unsuitable for dealing with the regression problems conducted on large-volume data sets. To tackle this difficulty, some approaches for generating more satisfactory algorithms have been introduced in the literature such as the efficient decomposition algorithm SVMTorch proposed in \cite{Suisse2001Support} and the randomized sketching algorithm for least-squares problems presented in \cite{Raskutti2016Sketching}. In particular, the mainstream solutions fall into two categories, the \emph{horizontal methods} and the \emph{vertical methods}. The former one, also known as a kind of distributed learning, consists of three steps, firstly partitioning the data set into several disjoint subsets, then implementing a certain learning algorithm to each data subset to obtain a local predictor, and finally synthesizing a global output by utilizing some average of the individual functions. By taking full advantage of the first step, horizontal methods gain its popularity on account of the ability to significantly reduce computing time and lower single-machine memory requirements. Unfortunately, although the effectiveness of distributed regression can be verified to some degree through theoretical results, for example, optimal convergence rates under certain restrictions (see e.g. \cite{Lin2017Distributed, Chang2017Distributed, Guo2017Distributed}), this approach suffers from its own inherent disadvantages. Mathematically speaking, for a single data block, the output function is obtained using trade off between the bias and variance, however, the variance of the average estimator in distributed learning actually shrinks as the number of blocks increases while the bias keeps unchanging which leads to the undesirable bias-denominating case. Therefore, distributed learning prefers algorithms in possession of the function with small bias and the optimal choice for a single block is not necessarily optimal for distributed learning. In this manner, the learning approach stands a good chance of creating local predictors quite different from the desired global predictor, not to mention the synthesized final predictor.

Other than partitioning the original data sets, another popular type of approaches, named vertical methods, chooses to divide the feature space into multiple non-overlapping blocks instead and to apply individual regression strategies on each resulting cell. In the literature, efforts have been made to propose innovative partition methods such as subsampling algorithms \citep{suykens2002least,Espinoza2006Fixedsize}, decision tree-based approaches \citep{Bennett1998decisiontree, Wu99largemargin, Chang2010Tree}, and so on. In addition, various kinds of embedded regressors are then applied to train local predictors such as Gaussian process (GP) regression, support vector machines, just to name a few. Although not suffering from the undesirable bias-denominating case, vertical methods have their own drawbacks, for example, the long-standing boundary discontinuities. Since the discontinuity impacts greatly on accuracy, literature has committed to tackle this problem. Under the same condition on partitioned input domain and GP regression, \cite{park11a} firstly imposes equal boundary constraints merely at a finite number of locations which actually cannot essentially solve the boundary discontinuities. Following on, \cite{park16a} extends this predictive means restriction to all neighboring regions. Nevertheless, the optimization-based formulations make this improved method infeasible to derive the marginal likelihood and the predictive variances in closed forms. In contrast, without imposing any further assumptions on the nature of the GPs, \cite{park2018patchwork} presents a simple and natural way to enforce continuity by creating additional pseudo-observations around the boundaries. However, on the one hand, this approach is defective for not benefiting from the desirable global property of GPs as well as suffering from the curse of dimensionality; on the other hand, artificially determined decomposition process brings a great impact on the final predictor, which inspires us to adopt more reasonable partition-based learning methods to gain smoothness from the randomness of partition and the nature of ensembles. Over past decades, a wealth of literature is pulled into exploring desirable partitions such as dyadic partition polynomial estimators \citep{Binev2005Universal, Binev2007Universal} and the Voronoi partition support vector machine \citep{meister16a}. However, to the best of our knowledge, although satisfactory experimental performance and optimal convergence rates are established, they fail to explain the benefits of ensembles for asymptotic smoothness from the theoretical perspective.

This study is conducted under such background, aiming at solving these tough problems mentioned earlier. To be specific, motivated by the random rotation ensemble algorithms proposed in \cite{Rodr2006Rotation, Ezequiel2013HT, Blaser2016HT}, we investigate a regression estimator based on partitions, induced by histogram transform ensembles, together with embedded individual regressors which takes full advantage of the histogram methods and ensemble learning. Specifically, its merits can be stated as twofold. First, the algorithm can be locally adaptive by applying adaptive stretching with respect to samples of each dimension. Second, the global smoothness of our obtained regression function is attributed to the randomness of different partitions together with the ensemble learning. The algorithm starts with mapping the input space into transformed feature space under a certain histogram transform. Then, the process proceeds by partitioning the transformed space into non-overlapping cells with the unit bin width where the bin indices are chosen as the round points. After obtaining the partition, we apply certain regression strategies such as piecewise constant or SVM to formulate the na\"{i}ve or kernel histogram transform estimator respectively according to the specific assumptions on the target conditional expectation function. Last but not least, by integrating estimators generated by the above procedure, we obtain a regressor ensemble with satisfactory asymptotic smoothness.

The contributions of this paper come from both the theoretical and experimental aspects. \emph{(i)} Our regression estimator varies when the Bayes decision rule $f_{L,\mathrm{P}}^*$ is assumed to satisfy different H\"{o}lder continuity assumptions. To be specific, under the assumption that $f_{L,\mathrm{P}}^*$ resides in $C^{0,\alpha}$ or $C^{1,\alpha}$, we adopt the na\"{i}ve histogram transform (NHT) estimator. By decomposing the error term into approximation error and estimation error, which correspond to data-free and data-dependent error terms, respectively, we prove almost optimal convergence rates for both single NHT and ensemble NHT in the space $C^{0,\alpha}$. In contrast, for the subspace $C^{1,\alpha}$ consisting of smoother functions, we show that the ensemble NHT can attain the convergence rate $O(n^{-(2(1+\alpha))/(4(1+\alpha)+d)})$ whereas the lower bound of the convergence rates for a single NHT is merely of the order $O(n^{-2/(2+d)})$ under certain conditions. As a result, when $d \geq 2(1+\alpha)/\alpha$, the ensemble NHT actually outperforms the single estimator, which illustrates the benefits of ensembles over single NHT. 
Furthermore, if $f_{L,\mathrm{P}}^* \in C^{k,\alpha}$ for $k \geq 2$, although taking fully advantage of the nature of ensembles, constant-embedded regressor is inadequate to achieve good performance. Thus, we turn to apply the kernel histogram transform (KHT) which is verified to have almost optimal convergence rates.  \emph{(ii)} We highlight that all theoretical results in this paper have their one to one corresponding experiment analysis. We design several numerical experiments to verify the study on parameters $\overline{h}_{0,n}$ and $T$.
Firstly, we show that, for NHT, there exists an optimal $\overline{h}_{0,n}$ with regard to the test error, whereas in contrast, KHT with fairly large cells have better performance.
Note that these experimental results coincide with the conclusions about the selection of parameter bin width in order to obtain almost optimal convergence rates, as are shown in Theorems \ref{thm::forest} and \ref{thm::convergencerateforest}. Moreover, in order to give a more comprehensive understanding of the significant benefits of ensemble NHT over single estimator, a simulation corresponding to Theorem \ref{prop::counter} is conducted on synthetic data with different parameter $T$, the number of NHT applied in the regression estimator. To be precise, the slope of mean squared error versus $T$ shows that ensemble NHT outperform single estimator. \emph{(iii)} Experiments conducted on real-data indicate that this method achieves both high precision and great efficiency. Its inherent advantages can be specified as follows. Firstly, the additional advantage of computational efficiency of our histogram transform ensembles mainly benefits from the parallel computation. Secondly, the randomness of partitions coming from the histogram transform together with the nature of ensembles allow us to better access to the unknown data structure as well as the desirable asymptotic smoothness, which greatly improves the progress of prediction. These advantages of our algorithm are fully evidenced by experiments conducted on real data, where we adopt ensemble KHT, equipped with adaptive histogram transforms. Experiments show that on the one hand, our adaptive KHTE outperforms the other state-of-the-art algorithms in terms of accuracy when $T$ is large enough; on the other hand, with much smaller $T$, it enjoys high efficiency by reducing average running time while maintaining satisfactory precision.

This paper is organized as follows. Section \ref{sec::methodology} is a preliminary section covering some required fundamental notations, definitions and technical histogram transform which all contribute to the formulation of both NHT and KHT. Section \ref{sec::Re1} is concerned with theoretical results, that is, the convergence rates, under different h\"{o}lder continuity assumptions on $f_{L,\mathrm{P}}^*$. To be specific, under the condition on the Bayesian decision function $f_{L,\mathrm{P}}^* \in C^{0,\alpha}$, almost convergence rates for both single NHT and ensemble NHT are derived in Section \ref{sec::C0}. In the subspace $C^{1,\alpha}$, we firstly present the convergence rate of ensemble NHT in Section \ref{sec::upper}, then a more complete theory is obtained by establishing the lower bound of single NHT to illustrate the exact benefits of ensembles in Section \ref{sec::lower}. In contrast, for the case where the target function resides in the subspace containing smoother functions $C^{k,\alpha}$, Section \ref{sec::Re2} presents almost optimal convergence rates for both single and ensemble KHT. Some comments and discussions related to the main results will also be presented in this section. Section \ref{sec::Error} provides a detailed analysis of both approximation error and sample error. Numerical experiments are conducted in Section \ref{sec::numerical_experiments} to verify our theoretical results and to further witness the effectiveness and efficiency of our algorithm. More precisely, Section \ref{sec::subsec::parameter} presents the study of parameters which verifies our theoretical results on the parameter selection for bin width $\overline{h}_0$ and ensemble number $T$ in order to achieve optimal convergence rates; Section \ref{sec::subsec::syntheticdata} then establishes a simulation on synthetic data to elucidate the exact benefits of the ensemble estimators over the single one; finally, comparisons between different regression methods on real data sets are provided in Section \ref{sec::subsec::realdata}. For the sake of clarity, we place all the proofs of Section \ref{sec::Re1} and Section \ref{sec::Error} in Section \ref{sec::proofs}. In Section \ref{sec::conclusion}, we close this paper with a conclusive summary, a brief discussion and additional remarks.

\section{Methodology} \label{sec::methodology}

Recall that our study on histogram transform ensembles (HTE) in this paper is initially aiming at addressing the large-scale regression problem. To this end, this section links our HTE algorithm to large-scale data analysis. Firstly, in Section \ref{sec::Notations}, we introduce some preliminaries, containing mathematical notations to be used throughout the entire paper, important basics for the least-square regression frameworks, and the definition of function space $C^{k,\alpha}$ where the target regression function lies in. Then in Section \ref{subsec::HT} we present the so called histogram transform approach through defining every crucial element such as rotation matrix $R$, stretching matrix $S$ and translator vector $b$. Based on the partition of the input space induced by the histogram transforms, we are then able to formulate the HTE for regression within the framework of regularized empirical risk minimization (RERM) in section \ref{sec::Al1}. To be more precise, taking the order of smoothness of the target function $f_{L,\mathrm{P}}^*$ into account, we establish the na\"{i}ve histogram transform ensembles (NHTE) and kernel histogram transform ensembles (KHTE) with $f_{L,\mathrm{P}}^*$ residing in different H\"{o}lder spaces respectively.

\subsection{Preliminaries} \label{sec::Notations}

\subsubsection{Notations}

Throughout this paper, we assume that $\mathcal{X} \subset \mathbb{R}^d$ and $\mathcal{Y} \subset \mathbb{R}$ are compact and non-empty. The goal of a supervised learning problem is to predict the value of an unobserved output variable $Y$ after observing the value of an input variable $X$. To be exact, we need to derive a predictor $f$ which maps the observed input value of $X$ to a prediction $f(X)$ of the unobserved output value of $Y$. The choice of predictor should be based on the training data $D := ((x_1, y_1), \ldots, (x_n, y_n))$ of i.i.d~observations, which are with the same distribution as the generic pair $(X, Y)$, drawn from an unknown probability measure $\mathrm{P}$ on $\mathcal{X} \times \mathcal{Y}$. Moreover, we denote $\mathrm{P}_X, \mathrm{P}_{Y|X}$ as the marginal and conditional distribution respectively.

For any fixed $R > 0$, we denote $B_R$ as the centered ball of $\mathbb{R}^d$ with radius $R$, that is,
\begin{align*}
B_R 
:= [-R, R]^d
:= \{ x = (x_1, \ldots, x_d) \in \mathbb{R}^d : x_i \in [-R, R], i = 1, \ldots, d \},
\end{align*} 
and for any $r \in (0, R)$, we write
\begin{align*}
B_{R,r}^+ := [r, R - r]^d.
\end{align*} 
We further assume that $\mathcal{X} \subset B_R$ for some $R > 0$ and $\mathcal{Y} := [-M, M]$ for some $M > 0$. 
In addition, for a Banach space $(E, \|\cdot\|_E)$, we denote $B_E$ as its unit ball, i.e.,
\begin{align*}
B_E
:= \{ f \in E : \|f\|_E \leq 1 \}.
\end{align*} 
Recall that for $1\leq p < \infty$, the $L_p$-norm of $x=(x_1,\ldots,x_d)$ is defined as $\| x \|_p := (|x_1|^p + \ldots + |x_d|^p)^{1/p}$, and the $L_{\infty}$-norm is defined as $\| x \|_{\infty} := \max_{i=1,\ldots,d} |x_i|$.

In the sequel, we use the notation $a_n \lesssim b_n$ to denote that there exists a positive constant $c$ such that $a_n \leq c b_n$, for all $n \in \mathbb{N}$. In addition, the notation $a_n \asymp b_n$ means that there exists some positive constant $c \in (0, 1)$, such that $a_n \geq c b_n$ and $a_n \leq c^{-1} b_n$, for all $n \in \mathbb{N}$. Moreover, throughout this paper, we shall make frequent use of the following multi-index notations. For any vector $x = (x_i)_{i=1}^d \in \mathbb{R}^d$, we write $\lfloor x \rfloor := (\lfloor x_i \rfloor)_{i=1}^d$, $x^{-1} := (x_i^{-1})_{i=1}^d$, $\log(x) := (\log x_i)_{i=1}^d$, $\overline{x} := \max_{i = 1, \ldots, d} x_i$, and $\underline{x} := \min_{i = 1, \ldots, d} x_i$.

\subsubsection{Least Squares Regression}

According to the learning target of finding the best regression function, it is legitimate to consider the least squares loss
$L = L_{\mathrm{LS}} : \mathcal{X} \times \mathcal{Y} \to [0, \infty)$ defined by $L(x, y, f(x)) := (y - f(x))^2$. Then, for a measurable decision function $f : \mathcal{X} \to \mathcal{Y}$, the risk is defined by
\begin{align*}
\mathcal{R}_{L,\mathrm{P}}(f) 
:= \int_{\mathcal{X} \times \mathcal{Y}} L(x, y, f(x)) \, d\mathrm{P}(x,y), 
\end{align*}
and the empirical risk is defined by
\begin{align*}
\mathcal{R}_{L,\mathrm{D}}(f) 
:= \frac{1}{n} \sum_{i=1}^n L(X_i, Y_i, f(X_i)),
\end{align*}
where $\mathrm{D}:= \frac{1}{n} \sum_{i=1}^n \delta_{(X_i, Y_i)}$ is the empirical measure associated to data and $\delta_{(X_i, Y_i)}$ is the Dirac measure at $(X_i, Y_i)$. The Bayes risk which is the minimal risk with respect to $\mathrm{P}$ and $L$ can be given by
\begin{align*}
\mathcal{R}_{L,\mathrm{P}}^* 
:= \inf \{ \mathcal{R}_{L,\mathrm{P}}(f) \mid f : \mathcal{X} \to \mathcal{Y} \text{ measuarable}\}.
\end{align*}
In addition, a measurable function $f_{L, \mathrm{P}}^* : \mathcal{X} \to \mathcal{Y}$ with $\mathcal{R}_{L, \mathrm{P}} (f_{L, \mathrm{P}}^*) = \mathcal{R}_{L, \mathrm{P}}^*$  is called a Bayes decision function. By minimizing the risk, we can get the Bayes decision function 
\begin{align}\label{eq::bayes}
f_{L, \mathrm{P}}^* = \mathbb{E}_\mathrm{P} (Y \arrowvert X)
\end{align}
which is a $\mathrm{P}_X$-almost surely $[-M, M]$-valued function.
Finally, it is well-known that
\begin{align}\label{equ::prop1}
\mathcal{R}_{L, \mathrm{P}}(f) - \mathcal{R}_{L,\mathrm{P}}^* = \left\|f - f_{L, \mathrm{P}}^*\right\|_{L_2(\mathrm{P}_X)}^2.
\end{align}

In what follows, note that it is sufficient to consider estimators with values in $[-M, M]$ on $\mathcal{X}$. To this end, we introduce the concept of clipping the decision function, see also Definition 2.22 in \cite{StCh08}. Let $\wideparen{t}$ be the clipped value of $t\in \mathbb{R}$ at $\pm M$ defined by 
\begin{align*}
\wideparen{t} := 
\begin{cases}
- M & \text{ if } t < - M,  
\\ 
t & \text{ if } t \in [-M, M], 
\\ 
M & \text{ if } t > M. 
\end{cases} 
\end{align*}
Then, a loss is called \emph{clippable} at $M > 0$ if, for all $(y, t) \in \mathcal{Y} \times \mathbb {R}$, there holds
\begin{align*}
L(x, y, \wideparen{t}) \leq L(x, y, t).
\end{align*}

According to Example 2.26 in \cite{StCh08}, the least square loss $L$ here can be clipped at $M$. Obviously, the latter implies that 
\begin{align*}
\mathcal{R}_{L, \mathrm{P}}(\wideparen{f}) 
\leq \mathcal{R}_{L, \mathrm{P}}(f)
\end{align*}
for all $f : \mathcal{X} \to \mathbb{R}$. In other words, restricting the decision function to the interval $[-M, M]$ cannot worsen the risk, in fact, clipping this function typically reduces the risk. Hence, in the following, we consider the clipped version $\wideparen{f}_{\mathrm{D}}$ of the decision function as well as the risk $\mathcal{R}_{L, \mathrm{P}}(\wideparen{f}_{\mathrm{D}})$ instead of the risk $\mathcal{R}_{L, \mathrm{P}}(f_{\mathrm{D}})$ of the unclipped decision function.

\subsubsection{H\"{o}lder Continuous Function Spaces}

In this paper, we mainly focus on the general function space $C^{k, \alpha}$ consisting of $(k, \alpha)$-H\"{o}lder continuous functions of different smoothness.

\begin{definition}\label{def::Cp}
Let $k \in \mathbb{N}_0 := \mathbb{N} \cup \{ 0 \}$, $\alpha \in (0, 1]$, and $R > 0$. We say that a function $f : B_R \to \mathbb{R}$ is $(k, \alpha)$-H\"{o}lder continuous, if there exists a finite constant $c_L > 0$ such that 
\begin{itemize}
\item[(i)] 
$\| \nabla^{\ell} f \| \leq c_L$ for all $\ell \in \{ 1, \ldots, k \}$;
\item[(ii)] 
$\| \nabla^k f(x) - \nabla^k f(x') \| \leq c_L \| x - x' \|^{\alpha}$ for all $x, x' \in B_R$.
\end{itemize}
The set of such functions is denoted by $C^{k, \alpha} (B_R)$.
\end{definition}

It can be seen from the definition above that functions contained in the space $C^{k,\alpha}$ with larger $k$ enjoy higher level of smoothness. Note that for the special case $k = 0$, the resulting function space $C^{0, \alpha} (B_R)$ coincides with the commonly used $\alpha$-H\"{o}lder continuous function space $C^{\alpha} (B_R)$.

\subsection{Histogram Transform} \label{subsec::HT}

To give a clear description of one possible construction procedure of histogram transforms, we introduce a random vector $(R, S, b)$ where each element represents the rotation matrix, stretching matrix and translation vector, respectively. To be specific, 
\begin{itemize}
\item[$R$] 
denotes the rotation matrix which is a real-valued $d \times d$ orthogonal square matrix with unit determinant, that is
\begin{align}\label{RotationMatrix}
R^{\top} = R^{-1} 
\quad 
\text{ and } 
\quad 
\det(R) = 1.
\end{align}
\item[$S$] 
stands for the stretching matrix which is a positive real-valued $d \times d$ diagonal scaling matrix with diagonal elements $(s_i)_{i=1}^d$ that are certain random variables. Obviously, there holds
\begin{align}\label{StretchingMatrix}
\det(S) = \prod_{i=1}^d s_i.
\end{align}
Moreover, we denote 
\begin{align}\label{equ::s}
{s} = (s_i)_{i=1}^d,
\end{align}
and the bin width vector measured on the input space is given by
\begin{align}\label{equ::h}
h =  s^{-1}.
\end{align}
\item[$b$] 
$\in [0,1]^d$ is a $d$ dimensional vector named translation vector. 
\end{itemize}

Here we describe a practical method for their construction we are confined to in this study. Starting with a $d \times d$ square matrix $M$, consisting of $d^2$ independent univariate standard normal random variates, a Householder $Q R$ decomposition  \cite{Householder1958Unitary} is applied to obtain a factorization of the form $M = R \cdot W$, with orthogonal matrix $R$ and upper triangular matrix $W$ with positive diagonal elements. The resulting matrix $R$ is orthogonal by construction and can be shown to be uniformly distributed. Unfortunately, if $R$ does not feature a positive determinant then it is not a proper rotation matrix according to definition (\ref{RotationMatrix}). However, if this is the case then we can flip the sign on one of the column vectors of $M$ arbitrarily to obtain $M^+$ and then repeat the Householder decomposition. The resulting matrix $R^+$ is identical to the one obtained earlier but with a change in sign in the corresponding column and $\det(R^+) = 1$, as required for a proper rotation matrix. See \cite{Blaser2016HT} for a brief account of the existed algorithms to generate random orthogonal matrices.

After that, we build a diagonal scaling matrix with the signs of the diagonal of $S$ where the elements $s_k$ are the well known Jeffreys prior, that is, we draw $\log (s_i)$ from the uniform distribution over certain interval of real numbers $[\log (\underline{s}_0), \log (\overline{s}_0)]$ for fixed constants $\underline{s}_0$ and $\overline{s}_0$ with $0<\underline{s}_0 < \overline{s}_0 < \infty$. By \eqref{equ::h}, there holds $h_i \in [\overline{s}_0^{-1}, \underline{s}_0^{-1}]$, $i=1,\ldots,d$. For simplicity and uniformity of notations, in the sequel, we denote $\overline{h}_0 = \underline{s}_0^{-1}$ and $\underline{h}_0 = \overline{s}_0^{-1}$, then we can say $h_i \in [\underline{h}_0, \overline{h}_0]$, $i=1,\ldots,d$.

Moreover, the translation vector $b$ is drawn from the uniform distribution over the hypercube $[0,1]^d$.

Based on the above notations, we define the histogram transform $H : \mathcal{X} \to \mathcal{X}$ by
\begin{align}\label{HistogramTransform}
H(x) := R \cdot S \cdot x + b,
\end{align}
which can be seen in Figure \ref{fig:RHT}, and the corresponding distribution by $\mathrm{P}_H := \mathrm{P}_R\otimes \mathrm{P}_S \otimes \mathrm{P}_b$, where $\mathrm{P}_R$, $\mathrm{P}_S$ and $\mathrm{P}_b$ represent the distribution for rotation matrix $R$, stretching matrix $S$ and translation vector $b$ respectively.

\begin{figure}[H]
\begin{minipage}[t]{0.99\textwidth}  
\centering  
\includegraphics[width=\textwidth]{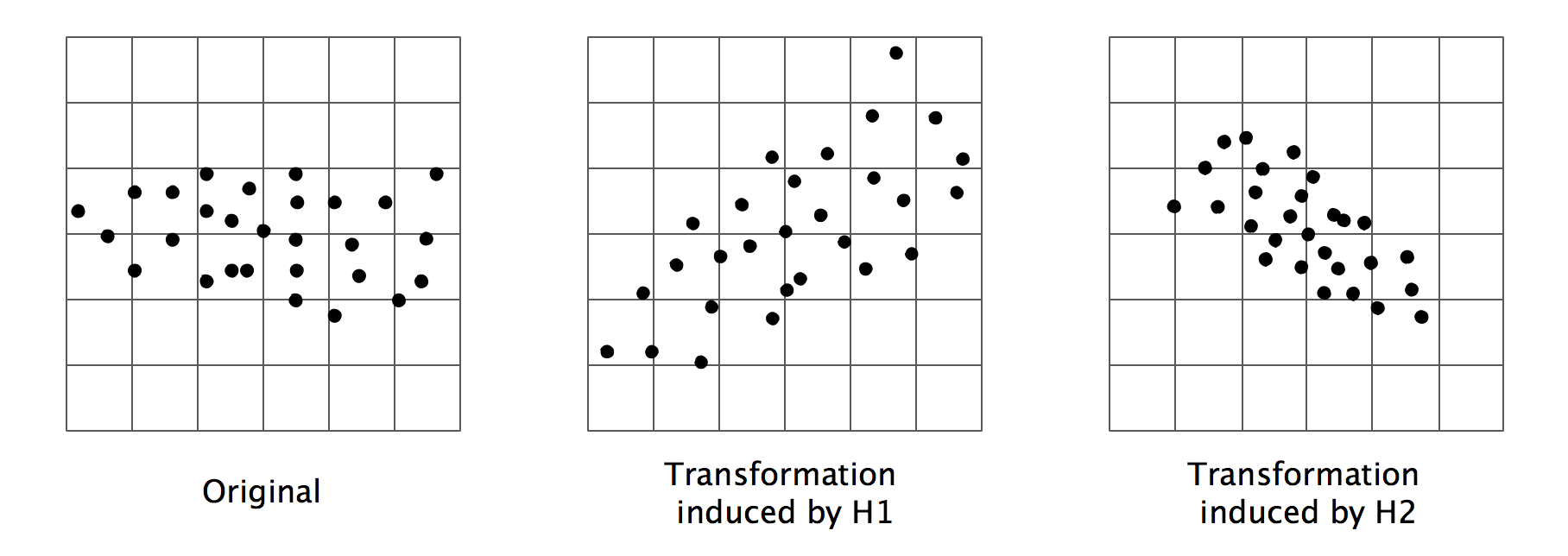}  
\end{minipage}  
\centering  
\caption{This figure shows two-dimensional examples of histogram transforms. The left subfigure is the original data and the other two subfigures are possible  histogram transforms of the original sample space, with obviously different rotating orientations and scales of stretching. }
\label{fig:RHT}
\end{figure}

Moreover, we denote $H'$ as the affine matrix $R \cdot S$, clearly, there holds
\begin{align}
\det(H') = \det(R) \cdot \det(S) = \prod_{i=1}^d s_i.
\end{align} 
The histogram probability $p (x | H', b)$ is defined by considering the bin width $h = 1$ in the transformed space. It is important to note that there is no point in using $h \neq 1$, since the same effect can be achieved by scaling the transformation matrix $H'$. Therefore, let $\lfloor H(x) \rfloor$ be the transformed bin indices, then the transformed bin is given by 
\begin{align}\label{TransBin}
A'_H(x) := \{ H(x') \ | \ \lfloor H(x') \rfloor = \lfloor H(x) \rfloor \}.
\end{align}
The corresponding histogram bin containing $x \in \mathcal{X}$ is 
\begin{align}\label{equ::InputBin}
A_H(x) := \{ x' \ | \ H(x') \in A'_H(x) \}
\end{align}
whose volume is $\mu(A_H(x)) = (\det(H'))^{-1}$.

For a fixed histogram transform $H$, we specify the partition of $B_r$ induced by the histogram rule \eqref{equ::InputBin}. 
Let $(A'_j)$ be the set of all cells generated by $H$, and denote $\mathcal{I}_H$ as the index set for $H$ such that $A'_j \cap B_r \neq \emptyset$ for all $j \in \mathcal{I}_H$. As a result, the set
\begin{align}\label{eq::PiH}
\pi_H 
:= (A_j)_{j \in \mathcal{I}_H}
:= (A'_j \cap B_r)_{j \in \mathcal{I}_H}
\end{align}
forms a partition of $B_r$. For notational convenience, if we substitute $A_0$ for $B_r^c$, then 
\begin{align*}
\pi'_H := (A_j)_{j \in \mathcal{I}_H \cup \{ 0 \}}
\end{align*}
forms a partition of $\mathbb{R}^d$.

\subsection{Histogram Transform Ensembles (HTE) for Regression} \label{sec::Al1}

Having developed the partition process induced by the histogram transforms, in this section, we formulate our histogram transform regressors, namely, the Na\"{i}ve histogram transform ensembles (NHTE) and kernel histogram transform ensembles (KHTE) using support vector machines.

In order to find an appropriate regressor under histogram transform $H$, we conduct our analysis under the framework of regularized empirical risk minimization (RERM). To be specific, let $L : \mathcal{X} \times \mathcal{Y} \times \mathbb{R} \to [0, \infty)$ be a loss and $\mathcal{F} \subset \mathcal{L}_0(\mathcal{X})$ be a non-empty set, where $\mathcal{L}_0(\mathcal{X})$ is the set of measurable functions on $\mathcal{X}$ and we let $\Omega : \mathcal{F} \to [0, \infty)$ be a penalty function. We denote regularized empirical risk minimization (RERM) as the learning method whose decision function $f_{\mathrm{D}}$ satisfying
\begin{align*}
f_{\mathrm{D}} = \argmin_{f \in \mathcal{F}} \mathcal{R}_{L, \mathrm{D}}(f) + \Omega (f)
\end{align*}
for all $n \geq 1$ and $D \in (\mathcal{X} \times \mathcal{Y})^n$.

\subsubsection{Na\"{i}ve Histogram Transform Ensembles (NHTE)}

In what follows, we define two ways to formulate NHTE, where the latter, with all single estimators sharing the same bin width $\underline{h}_0$, can be viewed as a special case of the former one. With the Bayesian decision function $f_{L,\mathrm{P}}^*$ lying in the space $C^{0,\alpha}$, we adopt the former one, for its generality, whereas for $f_{L,\mathrm{P}}^*$ in $C^{1,\alpha}$, we adopt the latter formulation, for the convenience of proving.

First, we illustrate the former and more general formulation. We define a function set $\mathcal{F}_H$ induced by histogram transform $H$, and then construct each single estimator by solving an optimization problem, with regard to bin width and this function set. Finally, the NHTE $f_{\mathrm{D},T}$ is obtained by performing the average of all single estimators.

To be specific, recall that for a given histogram transform $H$, the set $\pi_H = (A_j)_{j \in \mathcal{I}_H}$ forms a partition of $B_R$. We consider the function set $\mathcal{F}_H$ defined by
\begin{align} \label{mathcalFH}
\mathcal{F}_H := \biggl\{ \sum_{j \in \mathcal{I}_H} c_j \eins_{A_j} \ : \ c_j \in [-M,M], M > 0 \biggr\}.
\end{align}
Moreover, the bin width $h$ of the partition $\pi_H$ defined by \eqref{equ::h} is what we should penalize on. By penalizing on $h$, we are able to give some constraints on the complexity of the function set so that the set will have a finite VC dimension \citep{vapnik1971on}, and therefore make the algorithm PAC learnable \citep{valiant1984learnable}. In addition, it can also refrain the learning results from overfitting by avoiding the size of histogram bin to be too small.
With the data set $D$, the above RERM problem with respect to each function set $\mathcal{F}_H$ turns into
\begin{align}\label{equ::fDH}
(f_{\mathrm{D},H}, \underline{h}_0) 
:= \argmin_{\underline{h}_0}\argmin_{f \in \mathcal{F}_H} \lambda \underline{h}_0^{-2d} + \mathcal{R}_{L, \mathrm{D}}(f)
\end{align}
and its population version is presented by
\begin{align}\label{equ::fPH}
(f_{\mathrm{P},H}, \underline{h}_0^*) 
:= \argmin_{\underline{h}_0}\argmin_{f \in \mathcal{F}_H} \lambda \underline{h}_0^{-2d} + \mathcal{R}_{L, \mathrm{P}}(f).
\end{align}
It is well worth mentioning that the regularization term $\lambda \underline{h}_0^{-2d}$ is chosen from the following two aspects. Firstly, for simplicity of computation, we adopt the isotropic penalty for each dimension, that is to say, we penalize $\underline{h}_0$ rather than each elements $h_1, \ldots, h_d$. Secondly, take $C^{0,\alpha}$ as an example, as long as the peeling method (see Theorem 7.7 in \cite{StCh08}) holds, the exponent of $\underline{h}_0^{-1}$ will not have influence on the performance of convergence rate, therefore, we penalize on $\underline{h}_0^{-2d}$ which ensures the peeling method.

Let $\{ H_t \}_{t=1}^T$ be $T$ histogram transform independently drawn from distribution $\mathrm{P}_H$ and $\{ f_{\mathrm{D},H_t} \}_{t=1}^T$ be corresponding optimization solutions given by \eqref{equ::fDH}. We perform average of $f_{\mathrm{D},H_t}$ to obtain the na\"{i}ve histogram transform ensembles 
\begin{align}\label{eq::f_dt}
f_{\mathrm{D}, T} := \frac{1}{T} \sum_{t=1}^T f_{\mathrm{D},H_t}.
\end{align}

Next, we turn to the second formulation of NHTE, to be used in the theoretical analysis in the space $C^{1,\alpha}$. Herein we directly consider the algorithm in the sense of ensembles.

To this end, let $\{ H_t \}_{t=1}^T$ be $T$ histogram transforms induced by the same bin width $h$ and the function set $\mathcal{F}_h^T$ be defined by
\begin{align*}
\mathcal{F}_h^T := \biggl\{ \frac{1}{T} \sum_{t=1}^T f_t : f_t \in \mathcal{F}_{H_t}, t = 1, \ldots, T \biggr\},
\end{align*}
where the function sets $\{ \mathcal{F}_{H_t} \}_{t=1}^T$ are defined in the same way as \eqref{mathcalFH}. Then na\"{i}ve histogram transform ensembles are obtained within the RERM framework with respect to the function set $\mathcal{F}_h^T$ as 
\begin{align}\label{equ::fDE}
(f_{\mathrm{D},\mathrm{E}}, \underline{h}_{\mathrm{E}}) := \argmin_{\underline{h}_0} \argmin_{f \in \mathcal{F}_h^T} \lambda \underline{h}_0^{-2d} + \mathcal{R}_{L,\mathrm{D}}(f).
\end{align}
Moreover, its population version is given by  
\begin{align}\label{equ::fPE}
(f_{\mathrm{P},\mathrm{E}}, \underline{h}_{\mathrm{E}}^*) := \argmin_{\underline{h}_0} \argmin_{f \in \mathcal{F}_h^T} \lambda \underline{h}_0^{-2d} + \mathcal{R}_{L,\mathrm{P}}(f).
\end{align}

\subsubsection{Kernel Histogram Transform Ensembles (KHTE)} \label{sec::Al2}

The fomulation of KHTE is similar to that of NHTE in the space $C^{0,\alpha}$, but with a kernel-based function set.
Recall that $H$ is a histogram transform defined as in Section \ref{subsec::HT} and $\pi_H = (A_j)_{j \in \mathcal{I}_H}$ forms a partition of $B_R$ induced by the transform $H$ under the histogram rule \eqref{equ::InputBin}. The basic idea of our KHT approach is to consider for each bin $A_j$ of the partition an individual kernel regressor. To describe this approach in a mathematically rigorous way, we have to introduce some more notations. Let the index set
\begin{align*}
\mathcal{I}_j := \{i \in \{1, \ldots, n\} : x_i \in A_j\}, 
\qquad 
j \in \mathcal{I}_H,
\end{align*}
indicates the samples of $D$ contained in $A_j$, as well as the corresponding data set
\begin{align*}
D_j :=  \{(x_i, y_i) \in D : i \in \mathcal{I}_j\}, 
\qquad 
j \in \mathcal{I}_H.
\end{align*}
Moreover, for every $j \in \mathcal{I}_H$, we define a local loss $L_j : \mathcal{X} \times \mathcal{Y} \times \mathbb{R} \to [0, \infty)$ by
\begin{align*}
L_j (x, y, t):= \eins_{A_j}(x)L(x, y ,t)
\end{align*}
where $L : \mathcal{X} \times \mathcal{Y} \times \mathbb{R} \to [0, \infty)]$ is the least square loss that corresponds to our learning problem at hand. We further assume that $\mathcal{H}_j$ is a Reproducing Kernel Hilbert Space(RKHS) over $A_j$ with kernel $k_j : A_j \times A_j \to \mathbb{R}$. Here, every function $f\in \mathcal{H}_j$ is only defined on $A_j$. To this end, for $f \in \mathcal{H}_j$, we define the zero-extension $\widehat{f} : \mathcal{X} \to \mathbb{R}$ by
\begin{align*}
\widehat{f}(x) := 
\begin{cases} 
f(x), & \text{ if } x \in A_j, 
\\ 
0, & \text{ if }  x \notin A_j.
\end{cases} 
\end{align*}
Then, the extended space
\begin{align}\label{entendedRKHS}
\widehat{\mathcal{H}}_j := \{ \widehat{f} : f \in \mathcal{H}_j \}
\end{align}
equipped with the norm
\begin{align*}
\|\widehat{f}\|_{\widehat{\mathcal{H}}_j} := \|f\|_{\mathcal{H}_j}, \qquad \widehat{f} \in \widehat{\mathcal{H}}_j
\end{align*}
is an RKHS on $\mathcal{X}$, which is isometrically isomorphic to $\mathcal{H}_j$, see e.g. Lemma 2 in \cite{meister16a}.

Based on the preparations above, we are now able to construct an RKHS by a direct sum. 
To be specific, for $A, B \subset \mathcal{X}$ such that $A \cap B = \emptyset$ and $A \cup B \subset \mathcal{X}$, let $\mathcal{H}_A$ and $\mathcal{H}_B$ be RKHSs of the kernels $k_A$ and $k_B$ over $A$ and $B$, respectively. Furthermore, let $\widehat{\mathcal{H}}_A$ and $\widehat{\mathcal{H}}_B$ be the RKHSs of all functions of $\mathcal{H}_A$ and $\mathcal{H}_B$ extended to $\mathcal{X}$ in the sense of \eqref{entendedRKHS}. Then, $\widehat{\mathcal{H}}_A \cap \widehat{\mathcal{H}}_B = \{0\}$ and hence the direct sum
\begin{align}\label{DirectSum}
\mathcal{H} := \widehat{\mathcal{H}}_A + \widehat{\mathcal{H}}_B
\end{align}
exists. For $\lambda_A, \lambda_B > 0$ and $f \in \mathcal{H}$, let $\widehat{f}_A \in \widehat{\mathcal{H}}_A$ and $\widehat{f}_B \in \widehat{\mathcal{H}}_B$ be the unique functions such that $f = \widehat{f}_A + \widehat{f}_B$. 
Then, we define the norm $\|\cdot\|_{\mathcal{H}}$ by
\begin{align}\label{DirectSumNorm}
\|f\|_{\mathcal{H}}^2 
:= \lambda_A \|\widehat{f}_A\|_{\widehat{\mathcal{H}}_A}^2 
+ \lambda_B \|\widehat{f}_B\|_{\widehat{\mathcal{H}}_B}^2
\end{align}
and $\mathcal{H}$ equipped with the norm $\|\cdot\|_{\mathcal{H}}$ is again an RKHS for which
\begin{align*}
k(x, x') := \lambda_A^{-1} \widehat{k}_A (x, x') + \lambda_B^{-1} \widehat{k}_B(x, x'), 
\qquad 
x, x' \in \mathcal{X}
\end{align*}
is the reproducing kernel.

Note that in this paper, we only consider RKHSs of Gaussian RBF kernels. For this purpose, we summarize some notions and notations for the Gaussian case of RKHSs. For every $j \in \mathcal{I}_H$, let $k_{\gamma_j} : A_j \times A_j \to \mathbb{R}$ be the Gaussian kernel with width $\gamma_j > 0$, defined by
\begin{align}\label{equ::GaussianKernel}
k_{\gamma_j}(x, x') := \exp (-\gamma_j^{-2} \|x-x'\|_2^2),
\end{align}
with corresponding RKHS $\mathcal{H}_{\gamma_j}$ over $A_j$. According to the the discussion above, we define the extended RKHS by $\widehat{\mathcal{H}}_{\gamma_j}$ and the joint extended RKHS over $\mathcal{X}$ by $\mathcal{H} := \bigoplus_{j\in \mathcal{I}_H} \widehat{\mathcal{H}}_{\gamma_j}$.
We now formulate our kernel histogram transform ensembles in Gaussian RKHSs. To this end, we firstly consider the function space
\begin{align*}
\mathcal{H} &:= \bigg\{\sum_{j\in \mathcal{I}_H} f_{D_j, \gamma_j} : f_{D_j, \gamma_j} \in \widehat{\mathcal{H}}_{\gamma_j} \bigg\},
\end{align*}
and the KHT by solving the following optimization problem
\begin{align}
(f_{D,\gamma,H}, \underline{h}_0^*) 
& := \argmin_{\underline{h}_0} \argmin_{f\in\mathcal{H}} \lambda_1 \underline{h}_0^{q}+ \lambda_2\|f\|_{\mathcal{H}}^2+\frac{1}{n}\sum_{i=1}^n L(x_i,y_i, f(x_i))
\nonumber
\\
&= \argmin_{\underline{h}_0} \argmin_{f_j \in \widehat{\mathcal{H}}_{\gamma_j}} \lambda_1 \underline{h}_0^{q}+ \sum_{j\in \mathcal{I}_H } \lambda_{2,j} \|f\|_{\widehat{\mathcal{H}}_{\gamma_j}}^2 +\frac{1}{n}\sum_{i=1}^n \sum_{j\in \mathcal{I}_H} L_j(x_i,y_i, f(x_i)),
\label{equ::SVM}
\end{align}
where $\lambda_1>0$, $\lambda_{2,j} > 0$, and $\gamma_j > 0$.  Moreover, let $\{H_t, t=1,\ldots,T\}$ be $T$ histogram transforms and $f_{\mathrm{D}, \lambda, \gamma, H_t}$ be the $t$-th corresponding regularized histogram rule derived by \eqref{equ::SVM}, then we perform average to obtain the kernel histogram transform ensembles as 
\begin{align}\label{LSVMEHR}
f_{\mathrm{D},\gamma,\mathrm{E}} := \frac{1}{T} \sum_{t=1}^T f_{\mathrm{D}, \gamma, H_t}.
\end{align}

\subsubsection{Main Algorithm}\label{sec::subsec::Algorithm}

Our NHTE and KHTE can fit into the same algorithm, for they both share the basic structure of ensemble learning. Note that for NHTE, we adopt a so called \emph{best-scored} method, in the consideration of empirical performances. That is, for each single estimator, a certain number of candidate histogram transforms are generated under various hyperparameters $\underline{h}_0$ and $\overline{h}_0$, only the best one participates in constructing the final predictor. 
For KHTE, on the other hand, we skip the \emph{best-scored} operation. However, we can still exert the full use of them by means of parameter selections. Only the optimal $\underline{h}_0$ and $\overline{h}_0$ are universal for all component regressors of the ensemble estimator.

In Algorithm \ref{alg::HTE}, we show a general form of algorithm for HTE. 
We mention that for kernel HTE, i.e., HTE using support vector machines as local regressors, we simply choose $M = 1$.

\begin{algorithm}[H]
\caption{Histogram Transform Ensembles (HTE)}
\label{alg::HTE}
\KwIn{
Training data $D:=((X_1, Y_1), \ldots, (X_n,Y_n))$;
\\
\quad\quad\quad\quad Number of histogram transforms $T$;
\\
\quad\quad\quad\quad Bandwidth parameters $\{\underline{h}_0^i\}_{i=1}^M,\{\overline{h}_0^i\}_{i=1}^M$.
\\
} 
\For{$t =1 \to T$}{

\For{$i = 1 \to M$}{
	Generate random affine transform matrix
	$H_t^i=R_t\cdot S_t^i$;\\
	Apply data independent splitting to the transformed sample space; \\
	Apply constant functions or support vector machines to each cell; \\ 
	Compute the histogram regression mapping $f_{\mathrm{D},H_t^i}(x)$ induced by $H_t^i$.
}
Select the best mapping $f_{\mathrm{D},H_t}(x)$ with the minimal error.
}
\KwOut{The histogram transform ensemble for regression is
\begin{align*}
f_{\mathrm{D},\mathrm{E}}(x)
= \frac{1}{T} \sum_{t=1}^T f_{\mathrm{D},H_t}(x).
\end{align*}
}
\end{algorithm}

\section{Theoretical Results and Statements}\label{sec::Re1}

As mentioned above, our study on HTE in this paper differs when the Bayes decision rule $f_{L,\mathrm{P}}^*$ is assumed to have different smoothness, where mathematically speaking, the target function $f_{L,\mathrm{P}}^*$ resides in $C^{k,\alpha}$ with different $k \geq 0$, defined by Definition \ref{def::Cp}. 
In this section, we present main results on the convergence rates of our empirical decision function $f_{\mathrm{D},H}$ and $f_{\mathrm{D},\mathrm{E}}$ or $f_{\mathrm{D}, \gamma, H} $ and $f_{\mathrm{D}, \gamma,\mathrm{E}}$ to the Bayes decision function $f_{L, \mathrm{P}}^*$ of different smoothness.

This section is organized as follows. In Section \ref{sec::assumptions}, we firstly introduce some fundamental assumptions to be utilized in the theoretical analysis. Then under the assumption that $f_{L,\mathrm{P}}^* \in C^{0,\alpha}$, we prove almost optimal convergence rates for both single and ensemble NHT in Section \ref{sec::C0}, whereas in Section \ref{subsec::C1}, for the subspace $C^{1,\alpha}$ consisting of smoother functions, the lower bound of the single estimator illustrates the benefits of ensembles over single NHT. 
Moreover, if $k \geq 2$, although taking fully advantage of the nature of ensembles, as a constant-embedded regressor, 
NHT ensembles fail to attain the satisfactory convergence rates. Considering both theoretical and experimental performance, we are inspired to explore the kernel-embedded regressor KHT ensembles which is then verified to have almost optimal convergence rates in Section \ref{sec::Re2}. We also present some comments and discussions on the obtained main results as is shown Section \ref{sec::comments}.

\subsection{Fundamental Assumptions}\label{sec::assumptions}

To demonstrate theoretical results concerning convergence rates, fundamental assumptions are required respectively for the Bayesian decision function $f_{L,\mathrm{P}}^*$ and the bin width $h$ of stretching matrix $S$.

First of all, we assume the Bayesian decision function $f_{L,\mathrm{P}}^*$ lies in the function space $C^{k,\alpha}$.

\begin{assumption}
Let the Bayesian decision function $f_{L,\mathrm{P}}^*$ be defined in \eqref{eq::bayes}, assume that $f_{L,\mathrm{P}}^* \in C^{k,\alpha}$, where $\alpha \in (0,1]$ and $k \geq 0$. To be specific, 
we assume that
\begin{enumerate}
\item[(i)] 
for NHT, $f_{L,\mathrm{P}}^* \in C^{k,\alpha}$, where $\alpha \in (0,1]$ and $k = 0$;
\item[(ii)] 
for NHT, $f_{L,\mathrm{P}}^* \in C^{k,\alpha}$, where $\alpha \in (0,1]$ and $k = 1$;
\item[(iii)] 
for KHT, $f_{L,\mathrm{P}}^* \in C^{k,\alpha}$, where $\alpha \in (0,1]$ and $k \geq 2$.
\end{enumerate}
\end{assumption}

Then we assume the upper and lower bounds of the bin width $h$ are of the same order, that is, in a specific partition, the extent of stretching in each dimension cannot vary too much. Mathematically, we assume that the stretching matrix $S$ is confined into the class with width satisfying the following conditions.

\begin{assumption}\label{assumption::h}
Let the bin width $h \in [\underline{h}_0, \overline{h}_0]$ be defined as in \eqref{equ::h}, assume that there exists some constant $c_0 \in (0,1)$ such that 
\begin{align*}
c_0 \overline{h}_0 \leq \underline{h}_0 \leq c_0^{-1}\overline{h}_0.
\end{align*}
In the case that the bin width $h$ depends on the sample size $n$, that is, $h_n \in [\underline{h}_{0,n}, \overline{h}_{0,n}]$, assume that
there exist constants $c_{0,n} \in (0,1)$ such that 
\begin{align*}
c_{0,n} \overline{ {h}}_{0,n} \leq  \underline{ {h}}_{0,n} \leq c_{0,n}^{-1} \overline{ {h}}_{0,n}.
\end{align*}
\end{assumption}

\subsection{Results for NHT in the space $C^{0,\alpha}$}\label{sec::C0}

This section delves into proving almost optimal convergence rate for both single and ensemble NHT under the assumption that the Bayes decision function $f_{L,\mathrm{P}}^* \in C^{0,\alpha}$. 
Note that for the sake of the simplicity and uniformity of notations, we omit the index $t$ for a fixed $t \in \{1,\ldots,T\}$ and substitute $f_{\mathrm{D},H_n}$ for $f_{\mathrm{D},H_{t,n}}$.  
Moreover, for the sake of convenience, we write $\nu_n := \mathrm{P}^n \otimes \mathrm{P}_H$.

\subsubsection{Convergence Rates for Single NHT}

We now state our main result on the learning rates for single na\"{i}ve histogram transform regressor $f_{\mathrm{D},H_n}$ based on the established oracle inequality.

\begin{theorem}\label{thm::tree}
Let the histogram transform $H_n$ be defined as in \eqref{HistogramTransform} with bin width $h_n$ satisfying Assumption \ref{assumption::h}, and $f_{\mathrm{D}, H_n}$ be defined in \eqref{equ::fDH}. 
Furthermore, suppose that the Bayes decision function $f_{L, \mathrm{P}}^* \in C^{0,\alpha}$. Moreover, for all $\delta \in (0,1)$ let $(\lambda_n)$ and $(\overline{h}_{0,n})$ be defined by
\begin{align*}
\lambda_n := n^{-\frac{2(\alpha+d)}{2\alpha(1+\delta)+d}}, \qquad \overline{h}_{0,n} := n^{-\frac{1}{2\alpha(1+\delta)+d}}
\end{align*}
Then for all $\tau > 0$ and any $\xi > 0$, we have
\begin{align*}
\mathcal{R}_{L, \mathrm{P}}(f_{\mathrm{D},H_n}) - \mathcal{R}_{L,\mathrm{P}}^* 
\leq c \cdot n^{-\frac{2 \alpha}{2 \alpha + d}+\xi},
\end{align*}
holds with probability $\nu_n$ at least $1 - 3e^{-\tau}$, 
where $c$ is some constant depending on $\delta$, $d$, $M$, and $R$.
\end{theorem}

\subsubsection{Convergence Rates for Ensemble NHT}

The following theorem establishes the convergence rate for histogram transform ensembles $f_{\mathrm{D},T}$ based on \eqref{eq::f_dt}.

\begin{theorem}\label{thm::forest}
Let the histogram transform $H_n$ be defined as in \eqref{HistogramTransform} with bin width $h_n$ satisfying Assumption \ref{assumption::h}, and $f_{\mathrm{D},T}$ be defined in \eqref{eq::f_dt}.
Furthermore, suppose that the Bayes decision function $f_{L, \mathrm{P}}^* \in C^{0,\alpha}$. 
Moreover, for all $\delta \in (0,1)$, let $(\lambda_n)$ and $(\overline{h}_{0,n})$ be defined by
\begin{align*}
\lambda_n := n^{-\frac{2(\alpha+d)}{2\alpha(1+\delta)+d}}, 
\qquad 
\overline{h}_{0,n} := n^{-\frac{1}{2\alpha(1+\delta)+d}}.
\end{align*}
Then for all $\tau > 0$ and any $\xi > 0$, we have
\begin{align*}
\mathcal{R}_{L, \mathrm{P}}(f_{\mathrm{D},T}) - \mathcal{R}_{L,\mathrm{P}}^* 
\leq c \cdot n^{-\frac{2 \alpha}{2 \alpha + d}+\xi},
\end{align*}
holds with probability $\nu_n$ at least $1 - 3 e^{-\tau}$,
where $c$ is some constant depending on $\delta$, $d$, $M$, $R$, and $T$.
\end{theorem}

As shown in Theorems \ref{thm::tree} and \ref{thm::forest}, 
when the Bayesian decision function $f_{L,\mathrm{P}}^*$ lying in the space $C^{0,\alpha}$, 
the single and ensemble NHT both attain almost optimal learning rates,
if we choose the bin width of the order $\overline{h}_{0,n} = n^{-1/(2\alpha(1+\delta)+d)}$.
However, we fail to show the benefits of ensembles over single estimators. Therefore, to study the advantage of ensemble NHT in a learning rate point of view, we turn to the subspace $C^{1,\alpha}$.

\subsection{Results for NHT in the space $C^{1,\alpha}$}\label{subsec::C1}

In this subsection, we provide a result that illustrates the benefits of histogram transform ensembles over single histogram transform regressor by assuming that the Bayes decision function $f \in C^{1, \alpha}$. To this end, we firstly present the convergence rates of ensemble NHT when $T_n$, $\lambda_n$ and $\overline{h}_{0,n}$ are chosen appropriately in Theorem \ref{thm::optimalForest}. Then we obtain the lower bound of the single NHT to show that single histogram transform regressor does not benefit the additional smoothness assumption and fail to achieve the same convergence rates. We underline that the following theorem is conducted under certain conditions on the partial derivative of the decision function $f_{L,\mathrm{P}}^*$. Also, all theoretical results including both parameter selection for $\overline{h}_{0,n}$ and the lower bound, which establishes the exact difference of the convergence rate between the ensemble and single NHT, are verified experimentally in Section \ref{sec::subsec::parameter} and \ref{sec::subsec::syntheticdata}.

\subsubsection{Upper Bound of Convergence Rates for Ensemble NHT}\label{sec::upper}

\begin{theorem}\label{thm::optimalForest}
Let the histogram transform $H_n$ be defined as in \eqref{HistogramTransform} with bin width $h_n$ satisfying Assumption \ref{assumption::h}
and $T_n$ be the number of single estimators contained in the ensembles. 
Furthermore, let $f_{\mathrm{D},\mathrm{E}}$ be defined in \eqref{equ::fDE} and
suppose that the Bayes decision function $f_{L, \mathrm{P}}^* \in C^{1,\alpha}$
and $\mathrm{P}_X$ is the uniform distribution. 
Moreover, let $L_{\overline{h}_0}(x,y,t)$ be the least squares loss function restricted to $B_{R, \sqrt{d} \cdot \overline{h}_0}^+$, that is, 
\begin{align} \label{RestrictedLS}
L_{\overline{h}_0}(x,y,t) := L_{B_{R,\sqrt{d} \cdot \overline{h}_0}^+}(x,y,t) := \eins_{B_{R,\sqrt{d} \cdot \overline{h}_0}^+}(x) L(x,y,t),
\end{align}
where $L(x,y,t)$ is the least squares loss. 
Let the sequences $(T_n)$, $(\lambda_n)$ and $(\overline{h}_{0,n})$ be chosen as
\begin{align} \label{ParameterSelection}
\lambda_n := n^{-\frac{1}{2(1+\alpha)+2d}}, \qquad \overline{h}_{0,n}:=n^{-\frac{1}{2(1+\alpha)(2-\delta)+d}}, \qquad T_n := n^{\frac{2\alpha}{2(1+\alpha)(2-\delta)+d}},
\end{align}
where $\delta := 1/(8(c_d R/\underline{h}_{0,n})^d+1)$.
Then,  for all $\tau > 0$, the na\"{i}ve histogram transform ensemble regressor satisfies
\begin{align} \label{UpperBoundEnsemble}
\mathcal{R}_{L_{\overline{h}_0},\mathrm{P}}(f_{\mathrm{D},\mathrm{E}}) - \mathcal{R}_{L_{\overline{h}_0},\mathrm{P}}^* 
\lesssim n^{-\frac{2(1+\alpha)}{2(1+\alpha)(2-\delta)+d}}
\end{align}
with probability $\mathrm{P}^n$ not less than $1 - 4 e^{-\tau}$
in expectation with respect to $\mathrm{P}_H$.
\end{theorem}

Note that as $n \to \infty$, we have $\underline{h}_{0,n} \to 0$ and thus $\delta \to 0$. Therefore,
the upper bound \eqref{UpperBoundEnsemble} of
our ensemble NHT attains 
asymptotically
a convergence rate which is slightly faster than 
\begin{align} \label{UpperBoundEnsemble2}
n^{-\frac{2(1+\alpha)}{4(1+\alpha)+d}},
\end{align}
if we choose the bin width as
$\overline{h}_{0,n} = n^{-1/(4(1+\alpha)+d)}$.
That is, if the bin width is larger or smaller than the optimal oder $\underline{h}_{0,n}$, our NHTE have inferior empirical performance. In contrast, the excess risk decreases as $T_n$ increases at the beginning, and when $T_n$ achieves a certain level, the learning rate ceases to improve and attains the optimal. 
Finally, we mention that the theoretical results \eqref{ParameterSelection} on the parameter selection of $\overline{h}_{0,n}$ and $T_n$ will be experimentally verified in Section \ref{sec::subsec::parameter}.

\subsubsection{Lower Bound of Convergence Rates for Single NHT}\label{sec::lower}

As mentioned at the beginning of this subsection, we now present the lower bound of the single NHT to illustrate the benefit of ensembles. Just to make it clear, the following theorem establishes a worse convergence rate in contrast to one shown in Theorem \ref{thm::optimalForest}.

\begin{theorem}\label{prop::counter}
Let the histogram transform $H$ be defined as in \eqref{HistogramTransform} with bin width $h$ satisfying Assumption \ref{assumption::h} with $\overline{h}_0 \leq 1$.
Moreover, let the regression model defined by 
\begin{align} \label{regressionmodel}
Y := f(X) +\varepsilon,
\end{align}
where $\varepsilon$ is independent of $X$ such that $\mathbb{E}(\varepsilon | X) = 0$ and $\mathrm{Var}(\varepsilon|X) =: \sigma^2 < \infty$.
Assume that $f \in C^{1,\alpha}$ and for a fixed constant $\underline{c}_f \in (0, \infty)$, let $\mathcal{A}_f$ denote the set 
\begin{align} \label{DegenerateSetF}
\mathcal{A}_f := 
\big\{ x \in \mathbb{R}^d : 
\|\nabla f\|_{\infty} \geq \underline{c}_f  
\big\}.
\end{align}
Then, for all $n > N'$ with
\begin{align} \label{N'}
N' := \min \biggl\{ n \in \mathbb{N} : \overline{h}_{0,n} \leq \frac{R}{4\sqrt{d}} \biggr\},
\end{align}
by choosing
\begin{align*}
\overline{h}_{0,n} := n^{-\frac{1}{2+d}},
\end{align*}
there holds
\begin{align}\label{eq::lowerbound}
\mathcal{R}_{L,\mathrm{P}}(f_{\mathrm{D},H_n})-\mathcal{R}_{L,\mathrm{P}}^*
\gtrsim n^{-\frac{2}{2+d}}
\end{align}
in expectation with respect to $\nu_n$.
\end{theorem}

Note that for any $\alpha \in (0, 1]$, if $d \geq 2(1+\alpha)/\alpha$, then the upper bound of the convergence rate of ensemble NHT \eqref{UpperBoundEnsemble} or \eqref{UpperBoundEnsemble2}
will be smaller than the lower bound of single NHT \eqref{eq::lowerbound}.
This exactly illustrates the benefits of ensemble NHT over single estimators. 
Moreover, the assumption \eqref{DegenerateSetF} on the derivative of $f$ is quite reasonable and intuitive: if $\mathrm{P}(\mathcal{A}_f) = 0$, then the decision function degenerates into a constant, which can be fitted perfectly by single NHT, and the ensemble procedure is no longer meaningful.

\subsection{Results for KHT in the Space $C^{k,\alpha}$} \label{sec::Re2}

When the regression function resides in the H\"{o}lder space $C^{k,\alpha}$ with large $k$, which contains smoother functions,  the NHTE may not be appropriate anymore. Thus, we consider applying kernel regressors such as support vector machines to achieve kernel HTE. Similar to what we obtain for NHT before, in this section, we aim to develop the learning theory analysis for KHTE in the space $C^{k,\alpha}$ which explores the convergence rates of this estimator resulted from the RERM approach formulated in \eqref{equ::SVM}. Throughout this section, let $\mathrm{P}$ be a distribution on $\mathbb{R}^d \times \mathcal{Y}$, denote the marginal distribution of $\mathrm{P}$ onto $\mathbb{R}^d$ by $\mathrm{P}_X$, write $\mathcal{X} := \supp(\mathrm{P}_X)$, and assume $\mathrm{P}_X(\partial \mathcal{X}) = 0$. Different from the aforementioned conclusion that there exists an optimal parameter $\overline{h}_{0,n}$ with respect to almost optimal convergence rates, in this section, the theoretical results for KHT show that smoother Bayesian decision functions require larger cells. Note that this result is also verified later by the numerical experiments in Section \ref{sec::subsec::realdata}.

\subsubsection{Convergence Rates for Single KHT}

Firstly, we state our main result on the learning rates for single KHT $f_{\mathrm{D}, \gamma_n,H_n}$.

\begin{theorem}\label{thm::convergenceratetree}
Let the histogram transform $H_n$ be defined as in \eqref{HistogramTransform} with bin width $h_n$ satisfying Assumption \ref{assumption::h}, and $f_{\mathrm{D}, \gamma_n,H_n}$ be as in \eqref{equ::SVM}. 
Moreover, let the Bayes decision function satisfy $f_{L, \mathrm{P}}^* \in C^{k,\alpha}$ and
for every $j \in \mathcal{I}_{H_n}$, we choose
\begin{align*}
\lambda_{1,n} := n^{-\frac{1}{2(k+\alpha)+d}}, 
\qquad 
\lambda_{2,n,j} := n^{-1},
\qquad
\gamma_{n,j} := n^{-\frac{1}{2(k+\alpha) + d}}, 
\qquad
\overline{h}_{0,n} := n^0.
\end{align*}
Then, for all $n \geq 1$ and $\xi >0$, there holds
\begin{align*}
\mathcal{R}_{L, \mathrm{P}} (\wideparen{f}_{\mathrm{D}, \gamma_n,H_n}) 
- \mathcal{R}_{L, \mathrm{P}}^*  
\leq c \cdot n^{-\frac{2(k+\alpha)}{2(k+\alpha)+d}+\xi}
\end{align*}
with probability $\nu_n$ not less than $1-3e^{-\tau}$, where $c$ is some constant depending on $M$, $k$, $\alpha$, and $p$, which will be specified in the proof.
\end{theorem}

\subsubsection{Convergence Rates for Ensemble KHT}

We now present the convergence rates for ensemble KHT.

\begin{theorem}\label{thm::convergencerateforest}
Let the histogram transform $H_n$ be defined as in \eqref{HistogramTransform} with bin width $h_n$ satisfying Assumption \ref{assumption::h}, and $f_{\mathrm{D}, \gamma_n,\mathrm{E}}$ be as in \eqref{LSVMEHR}. 
Moreover, let the Bayes decision function satisfy $f_{L, \mathrm{P}}^* \in C^{k,\alpha}$ and
for every $j \in \mathcal{I}_{H_n}$, we choose
\begin{align*}
\lambda_{1,n} := n^{-\frac{1}{2(k+\alpha)+d}}, 
\qquad 
\lambda_{2,n,j} := n^{-1},
\qquad
\gamma_{n,j} := n^{-\frac{1}{2(k+\alpha) + d}}, 
\qquad
\overline{h}_{0,n} := n^{0}.
\end{align*}
Then, for all $n \geq 1$ and $\xi >0$, there holds
\begin{align*}
\mathcal{R}_{L, \mathrm{P}} (\wideparen{f}_{\mathrm{D}, \gamma_n,\mathrm{E}}) 
- \mathcal{R}_{L, \mathrm{P}}^*  
\leq c \cdot n^{-\frac{2(k+\alpha)}{2(k+\alpha)+d}+\xi}
\end{align*}
with probability $\nu_n$ not less than $1-3e^{-\tau}$, where $c$ is some constant depending on $M$, $k$, $\alpha$, $p$, and $T$, which will be specified in the proof.
\end{theorem}

As shown in Theorems \ref{thm::convergenceratetree} and \ref{thm::convergencerateforest}, 
in order to achieve almost optimal convergence rates, the bin width of the histogram transforms
$\overline{h}_{0,n}$ should be selected to be of the constant order. 
This phenomenon will also be experimentally verified in Section \ref{sec::subsec::parameter}.

\subsection{Comments and Discussions}\label{sec::comments}

From the above learning theory analysis, it becomes clear that our study provides an effective solution to large-scale regression problems, i.e., a nonparametric vertical method, built upon the partition induced by histogram transforms together with embedded regressors. We now go further in comparing our work with the existing studies.

Recall that the histogram transform estimator varies when the Bayes decision function $f_{L,\mathrm{P}}^*$ satisfies different $(k,\alpha)$-H\"{o}lder continuous assumptions and theoretical analysis on convergence rates is conducted for different estimators in these spaces respectively. For the space $C^{0,\alpha}$, almost optimal convergence rates $O(n^{-2\alpha/(2\alpha+d)+\xi})$ for both single NHT and ensemble NHT are derived in Theorem \ref{thm::tree} and Theorem \ref{thm::forest}. However, to the best of our knowledge, till now there is no existing literature successfully illustrating the exact benefits of ensembles over single estimators due to the same convergence rates for $f_{\mathrm{D},H}$ and $f_{\mathrm{D},T}$ in the space $C^{0,\alpha}$. Therefore, we turn to the subspace $C^{1,\alpha}$ consisting of a class of smoother functions and verify that ensemble NHT converges faster than single NHT. More precisely, 
Theorem \ref{thm::optimalForest} establishes
convergence rates $n^{-(2(1+\alpha))/(2(1+\alpha)(2-\delta)+d)}$, whereas in contrast, 
Theorem \ref{prop::counter} shows that single NHT fails to achieve this rate whose lower bound is of order $O(n^{-2/(d+2)})$. For the smoother space $C^{k, \alpha}$ with $k \geq 2$, constant regressors are no longer adequate for obtaining satisfactory theoretical results, therefore kernel regression strategy is adopted. We then establish almost optimal convergence rates $O(n^{-2(k+\alpha)/(2(k+\alpha)+d)+\xi})$ for both single KHT and ensemble KHT in Theorem \ref{thm::convergenceratetree} and \ref{thm::convergencerateforest} thanks to the use of some convolution technique that helps bounding the approximation error.

For vertical methods, \cite{meister16a} establishes almost optimal convergence rates $O(n^{-2\alpha/(2\alpha+d)+\xi})$ for VP-SVM when the Bayes decision function is assumed to reside in a Besov space with $\alpha$-degrees of smoothness, which coincides with our theoretical results for the  H\"{o}lder continuous function spaces.

For horizontal methods, \cite{Zhang2015Distributed} randomly partitions a dataset containing $n$ samples into several subsets of equal size, following by providing an independent kernel ridge regression estimator for each subset with a careful choice of the regularization parameter, and then synthesize them by performing a average. With the restriction that the Bayes decision function lies in the corresponding reproducing kernel Hilbert space, convergence rates are then presented with respect to different kernels in the sense of mean-squared error. For example, if the kernel has finite rank $r$, they obtain optimal convergence rates of type $O(r/n)$; for the kernel with $\nu$-polynomial eigendecay, the convergence rates of Fast-KRR algorithms turns out to be $O(n^{-2\nu/2\nu+1})$ which is also optimal, while for a kernel with sub-Gaussian eigendecay, the result turns out to be optimal up to a logarithm term $O(\sqrt{\log n}/ n)$. In a similar way, \cite{Lin2017Distributed} constructs random partition with equal sample size and obtain independent kernel ridge regression, but synthesize them by taking a weighted average rather than simple average. Then, under the smoothness assumption with respect to the $r$-th power of the integral operator $L_k$ and an $\alpha$-related capacity assumption, the convergence rate $O(n^{-2\alpha r/ (4\alpha r+1)})$ is verified to be almost optimal. \cite{Guo2017Distributed} focuses on the distributed regression with bias corrected regularization kernel network and derives the learning rates of order $O(n^{-2r/(2r+\beta)})$, where $\beta$ is the capacity related parameter.

Moreover, rather than the aforementioned two methods, there exist a flurry of studies for localized learning algorithms in the literature aiming at the large-scale regression problem. For example, KNN based methods are trained on $k$ samples which are closest to the testing point. Under some additional assumptions on the loss function, \cite{Hable2013universal} establishes the universal consistency for SVM-KNN considering metrics w.r.t. the feature space. In addition, training data is split into clusters and then an individual SVM is applied on each cluster in \cite{Cheng07localizedsupport, Cheng2010Efficient}. However, the presented results are mainly of experimental character.

\section{Error Analysis}\label{sec::Error}

In this section, we conduct error analysis for the single and ensemble estimators $f_{\mathrm{D}, H}$ and $f_{\mathrm{D}, \mathrm{E}}$ in the H\"{o}lder spaces $C^{k,\alpha}$ with $\alpha \in (0, 1]$ and $k = 0$, $k= 1$, and $k \geq 2$.

\subsection{Analysis for NHT in the space $C^{0,\alpha}$}\label{sec::ErrorCons}

In this subsection, we investigate the convergence property of 
$f_{\mathrm{D},H}$ and $f_{\mathrm{D},\mathrm{E}}$ when the Bayes decision function $f_{L,\mathrm{P}}^* \in C^{0,\alpha}$. Recall that $f_{\mathrm{P},H}$ and $f_{\mathrm{P},\mathrm{E}}$ are the population version of single NHT and NHTE estimators respectively, derived as in \eqref{equ::fPH} and \eqref{equ::fPE} within the RERM framework. To this end, we start with considering the single estimator. More precisely, the convergence analysis is conducted with the help of the following error decomposition. To this end, we define $h_f := L\circ f - L\circ f_{L,\mathrm{P}}^*$ for all measurable $f : \mathcal{X} \to \mathbb{R}$. By the definition of $f_{\mathrm{D},H}$, we have
\begin{align*}
\Omega(f_{\mathrm{D},H})+\mathbb{E}_{\mathrm{D}} h_{\wideparen{f}_{\mathrm{D},H}}
\leq \Omega(f_{\mathrm{P},H})+\mathbb{E}_{\mathrm{D}} h_{f_{\mathrm{P},H}},
\end{align*}
and consequently,
for all $D\in (\mathcal{X}\times \mathcal{Y})^n$, there holds
\begin{align}
&\Omega(f_{\mathrm{D},H})+\mathcal{R}_{L,\mathrm{P}}(\wideparen{f}_{\mathrm{D},H})-\mathcal{R}_{L,\mathrm{P}}^*
\nonumber\\
&=\Omega(f_{\mathrm{D},H})+\mathbb{E}_{\mathrm{P}} h_{\wideparen{f}_{\mathrm{D},H}}
\nonumber\\
&\leq \Omega(f_{\mathrm{P},H})+\mathbb{E}_{\mathrm{D}} h_{f_{\mathrm{P},H}}-\mathbb{E}_{\mathrm{D}} h_{\wideparen{f}_{\mathrm{D},H}}+\mathbb{E}_{\mathrm{P}} h_{\wideparen{f}_{\mathrm{D},H}}
\nonumber\\
&=(\Omega(f_{\mathrm{P},H})+\mathbb{E}_{\mathrm{P}} h_{f_{\mathrm{P},H}})+(\mathbb{E}_{\mathrm{D}} h_{f_{\mathrm{P},H}}-\mathbb{E}_{\mathrm{P}} h_{f_{\mathrm{P},H}})+(\mathbb{E}_{\mathrm{P}} h_{\wideparen{f}_{\mathrm{D},H}}-\mathbb{E}_{\mathrm{D}} h_{\wideparen{f}_{\mathrm{D},H}}).
\label{equ::decomc0}
\end{align}

Note that the first term $\Omega(f_{\mathrm{P},H})+\mathbb{E}_{\mathrm{P}} h_{f_{\mathrm{P},H}}$
in the above inequality \eqref{equ::decomc0} represents the approximation error, which is data independent. 
In contrast, both of the remaining terms $\mathbb{E}_{\mathrm{D}} h_{f_{\mathrm{P},H}}-\mathbb{E}_{\mathrm{P}} h_{f_{\mathrm{P},H}}$ and $\mathbb{E}_{\mathrm{P}} h_{\wideparen{f}_{\mathrm{D},H}}-\mathbb{E}_{\mathrm{D}} h_{\wideparen{f}_{\mathrm{D},H}}$ are sample errors depending on the data $D$.

\subsubsection{Bounding the Approximation Error Term}\label{sec::AError1}

Our first theoretical result on bounding the approximation error term in the sense of least squared loss shows that, the $L_2$ distance between $f_{\mathrm{P},H}$ and $f_{L, \mathrm{P}}^*$ 
behaves polynomial in the regularization parameter $\lambda$, by choosing the bin width $\underline{h}_0$ appropriately.

\begin{proposition} \label{prop::ApproximatiON-ERROR}
Let the histogram transform $H$ be defined as in \eqref{HistogramTransform} with bin width $h$ satisfying Assumption \ref{assumption::h}. Moreover, suppose that the Bayes decision function $f_{L, \mathrm{P}}^* \in C^{0,\alpha}$. 
Then, for any fixed $\lambda > 0$, there holds
\begin{align*}
\lambda (\underline{h}_0^*)^{-2d} + \mathcal{R}_{L,\mathrm{P}}(f_{\mathrm{P},H})-\mathcal{R}_{L,\mathrm{P}}^*
&\leq c \cdot \lambda ^{\frac{\alpha}{\alpha+d}},
\end{align*}
where $c$ is some constant depending on $\alpha$, $d$, and $c_0$ as in
Assumption \ref{assumption::h}.
\end{proposition}

\subsubsection{Bounding the Sample Error Term}\label{sec::SError1}

In order to bound the sample error term, we give four descriptions of the capacity of the function set in Definition \ref{def::VCdimension}, Definition \ref{def::Covering Numbers}, Definition \ref{def::entropy numbers} and Definition \ref{def::RademacherDefinition}.

Firstly, there is a need for some constraints on the complexity of the function set so that the set will have a finite VC dimension \citep{vapnik1971on}, and therefore make the algorithm PAC learnable \citep{valiant1984learnable}, see e.g., \cite[Definition 3.6.1]{Gine2016Infinitedimension}.

\begin{definition}[VC dimension] \label{def::VCdimension}
Let $\mathcal{B}$ be a class of subsets of $\mathcal{X}$ and $A \subset \mathcal{X}$ be a finite set. The trace of $\mathcal{B}$ on $A$ is defined by $\{ B \cap A : B \in  \mathcal{B} \}$. Its cardinality is denoted by $\Delta^{\mathcal{B}}(A)$. We say that $\mathcal{B}$ shatters $A$ if $\Delta^{\mathcal{B}}(A) = 2^{\#(A)}$, that is, if for every $\tilde{A} \subset A$, there exists a $B \subset \mathcal{B}$ such that $\tilde{A} = B \cap A$. For $k \in \mathbb{N}$, let 
\begin{align*}
m^{\mathcal{B}}(k)  := \sup_{A \subset \mathcal{X}, \, \#(A) = k}  \Delta^{\mathcal{B}}(A).
\end{align*}
Then, the set $\mathcal{B}$ is a Vapnik-Chervonenkis class if there exists $k < \infty$ such that $m^{\mathcal{B}}(k) < 2^k$ and the minimal of such $k$ is called the \emph{VC dimension} of $\mathcal{B}$, and abbreviated as $\mathrm{VC}(\mathcal{B})$.
\end{definition}

Recall that $H$ is a histogram transform, $\pi_H := (A_j)_{j\in \mathcal{I}_H}$ is a partition of $B_r$ with the index set $\mathcal{I}_H$ induced by $H$. And let $\Pi_H$ be the gathering of all partitions $\pi_H$, that is, $\Pi_H := \{\pi_{H} : H \sim \mathrm{P}_H \}$. To bound the estimation error, we need to introduce some more notations. To this end, let $\pi_h$ denote the collection of all cells in $\pi_H$, that is, 
\begin{align}  \label{Bh}
\pi_h := \{ A_j :  A_j \in \pi_H \in \Pi_H \}. 
\end{align} 
Moreover, we define
\begin{align} \label{equ::pih}
\Pi_h := \biggl\{ B : B = \bigcup_{j \in I} A_j,  I \subset \mathcal{I}_H, A_j \in \pi_H \in \Pi_H \biggr\}.
\end{align}

The following lemma presents the upper bound of VC dimension for the interested sets $\pi_h$ and $\Pi_h$.

\begin{lemma} \label{VCindex}
Let the histogram transform $H$ be defined as in \eqref{HistogramTransform} with bin width $h$ satisfying Assumption \ref{assumption::h}. Moreover,
let $\pi_h$ and $\Pi_h$ be defined as in \eqref{Bh} and \eqref{equ::pih}, respectively. Then we have
\begin{align*} 
\mathrm{VC}(\pi_h) \leq 2^d + 2
\end{align*} 
and
\begin{align} \label{VCMathcalBh}
\mathrm{VC}(\Pi_h)
\leq \bigl( d (2^d - 1) + 2 \bigr) \bigl( 2 R \sqrt{d} / \underline{h}_0 + 1 \bigr)^d.
\end{align}
\end{lemma}

To bound the capacity of an infinite function set, we need to introduce the following fundamental descriptions which enables an approximation by finite subsets, see e.g. (\cite{StCh08}, Definition 6.19).

\begin{definition}[Covering Numbers]\label{def::Covering Numbers}
Let $(X, d)$ be a metric space, $A \subset X$ and $\varepsilon > 0$. We call $A' \subset A$ an $\varepsilon$-net of $A$ if for all $x \in A$ there exists an $x' \in A'$ such that $d(x, x') \leq \varepsilon$. Moreover, the $\varepsilon$-covering number of $A$ is defined as
\begin{align*}
\mathcal{N}(A, d, \varepsilon)
= \inf \biggl\{ n \geq 1 : \exists x_1, \ldots, x_n \in X \text{ such that } A \subset \bigcup_{i=1}^n B_d(x_i, \varepsilon) \biggr\},
\end{align*}
where $B_d(x, \varepsilon)$ denotes the closed ball in $X$ centered at $x$ with radius $\varepsilon$.
\end{definition}

Let $\mathcal{B}$ be a class of subsets of $\mathcal{X}$, denote $\eins_{\mathcal{B}}$ as the collection of the indicator functions of all $B \in \mathcal{B}$, that is, $\eins_{\mathcal{B}} := \{ \eins_B : B \in \mathcal{B} \}$. Moreover, as usual, for any probability measure $\mathrm{Q}$, $L_2(\mathrm{Q})$ is denoted as the $L_2$ space with respect to $Q$ equipped with the norm $\|\cdot\|_{L_2(\mathrm{Q})}$.

\begin{lemma} \label{ScriptBhCoveringNumber}
Let $\pi_h$ and $\Pi_h$ be defined as in \eqref{Bh} and \eqref{equ::pih}, respectively. Then, for all $0 < \varepsilon < 1$, there exists a universal constant $K$ such that for any probability measure $\mathrm{Q}$, there hold
\begin{align} \label{CollectionCoveringNumber} 
\mathcal{N}(\eins_{\pi_h}, \|\cdot\|_{L_2(\mathrm{Q})}, \varepsilon) 
\leq K (2^d+2) (4e)^{2^d+2} (1/\varepsilon)^{2(2^d+1)} 
\end{align} 
and
\begin{align} \label{BpCoveringNumber}
\mathcal{N}(\eins_{\Pi_h}, \|\cdot\|_{L_2(\mathrm{Q})}, \varepsilon) 
\leq K (c_d R/\underline{h}_0)^d 
(4 e)^{(c_d R/\underline{h}_0)^d} 
(1/\varepsilon)^{2((c_d R/\underline{h}_0)^d - 1)},
\end{align}
where the constant $c_d := 3 \cdot 2^{1+\frac{1}{d}} \cdot d^{\frac{1}{d}+\frac{1}{2}}$.
\end{lemma}

Let us first consider the complexity of the function set of binary value assignment case. To this end, we define
\begin{align}\label{equ::equalBHO}
\mathcal{F}_H^b := \biggl\{ \sum_{j \in \mathcal{I}_H} c_j \eins_{A_j} \ : \ c_j \in \{ -1, 1 \}, A_j \in \pi_H \in \Pi_H \biggr\}.
\end{align}
Note that for all $g \in \mathcal{F}_H^b$, there exists some $B \in \Pi_H \in \Pi_h$ such that $g$ can be expressed as $g = \eins_{B} - \eins_{B^c}$. Therefore, $\mathcal{F}_H^b$ can be equivalently formulated as
\begin{align}\label{equ::equalBH}
\mathcal{F}_H^b  := \{\eins_B - \eins_{B^c} : B\in \Pi_h \}.
\end{align}
The following lemma gives a upper bound for the covering number of $\mathcal{F}_H^b$.

\begin{lemma}\label{lem::CoveringNumberBH}
Let $\mathcal{F}_H^b$ be defined as in \eqref{equ::equalBHO} or \eqref{equ::equalBH}. 
Then for all $\varepsilon \in (0, 1)$, there exists a universal constant $c < \infty$ such that
\begin{align*}
\mathcal{N} (\mathcal{F}_H^b, \|\cdot\|_{L_2(\mathrm{P}_X)}, \varepsilon) 
\leq c (c_d R/\underline{h}_0 + 1)^d
(4 e)^{(c_d R/\underline{h}_0 + 1)^d}
(2/\varepsilon)^{2 ( (c_d R/\underline{h}_0 + 1)^d - 1)},
\end{align*}
where the constant $c_d := 3 \cdot 2^{1+\frac{1}{d}} \cdot d^{\frac{1}{d}+\frac{1}{2}}$.
\end{lemma}

We further need the following concept of entropy numbers to illustrate the capacity of an infinite function set, for more details we refer to A.5.6 in \cite{StCh08}.

\begin{definition}[Entropy Numbers] \label{def::entropy numbers}
Let $(X, d)$ be a metric space, $A \subset X$ and $n \geq 1$ be an integer. The $n$-th entropy number of $(A, d)$ is defined as
\begin{align*}
e_n(A, d) = \inf \biggl\{ \varepsilon > 0 : \exists x_1, \ldots, x_{2^{n-1}} \in X \text{ such that } A \subset \bigcup_{i=1}^{2^{n-1}} B_d(x_i, \varepsilon) \biggr\}.
\end{align*}
\end{definition}
Before we proceed, there is a need to introduce an important conclusion establishing the equivalence of covering number and entropy number. To be specific, entropy and covering numbers are in some sense inverse to each other. For all constants $a > 0$ and $q > 0$, the implication
\begin{align}\label{EntropyCover}
e_i (T, d) \leq a i^{-1/q}, 
\,
\forall i \geq 1 
\quad 
\Longrightarrow 
\quad 
\ln \mathcal{N}(T, d, \varepsilon) \leq \ln(4) (a/\varepsilon)^q, 
\, 
\forall \varepsilon > 0
\end{align}
holds by Lemma 6.21 in \cite{StCh08}. Additionally, Exercise 6.8 in \cite{StCh08} yields the opposite implication, namely
\begin{align}\label{CoverEntropy}
\ln \mathcal{N}(T, d, \varepsilon) < (a/\varepsilon)^q, 
\,
\forall \varepsilon > 0 
\quad 
\Longrightarrow 
\quad 
e_i(T, d) \leq 3^{1/q} a i^{-1/q}, 
\,
\forall i \geq 1.
\end{align}

Now we introduce some notations of the oracle inequality for general $\varepsilon$-CR-ERMs (see also Definition 7.18 in \cite{StCh08}). Denote
\begin{align}\label{equ::r*}
r_b^* 
:= \inf_{f \in \mathcal{F}_H^b} \lambda \underline{h}_0^{-2d} 
+ \mathcal{R}_{L, \mathrm{P}}(f) - \mathcal{R}_{L, \mathrm{P}}^*.
\end{align}
Then for $r > r_b^*$, we write
\begin{align}
\mathcal{F}_r^b 
& := \bigl\{ g \in \mathcal{F}_H^b : \lambda \underline{h}_0^{-2d} 
+ \mathcal{R}_{L, \mathrm{P}}(g) - \mathcal{R}_{L, \mathrm{P}}^* \leq r \bigr\}, 
\label{equ::Ar}
\\
\mathcal{H}_r^b 
& := \{ L \circ g - L \circ f_{L,\mathrm{P}}^* : g \in \mathcal{F}_r^b \},
\label{equ::Hr}
\end{align}
where $L \circ g$ denotes the least squares loss of $g$. Moreover, in a similar way, let
\begin{align}
r^* :=  \inf_{f \in \mathcal{F}_H} \lambda \underline{h}_0^{-2d} + \mathcal{R}_{L, \mathrm{P}}(f) - \mathcal{R}_{L, \mathrm{P}}^*,
\end{align}
and for $r > r^*$, write
\begin{align}
\mathcal{F}_r 
& := \bigl\{ g \in \mathcal{F}_H : \lambda \underline{h}_0^{-2d} 
+ \mathcal{R}_{L, \mathrm{P}}(g) - \mathcal{R}_{L, \mathrm{P}}^* \leq r \bigr\}, 
\label{equ::Atilder}
\\
\mathcal{H}_r 
& := \{ L \circ g - L \circ f_{L,\mathrm{P}}^* : g \in \mathcal{F}_r \},
\label{equ::Htilder}
\end{align}
where $L \circ g$ denotes the least squares loss of $g$.

\begin{lemma}\label{lem::entropynumber}
Let $\mathcal{H}_r^b$ be defined as in \eqref{equ::Hr}. 
Then for all $\delta \in (0,1)$, the $i$-th entropy number of $\mathcal{H}_r^b$ satisfies
\begin{align*}
\mathbb{E}_{D \sim \mathrm{P}^n} e_i(\mathcal{H}_r^b, \|\cdot\|_{L_2(\mathrm{D})}) 
\leq \bigl( 33/(2e\delta) (2 c_d R (r/\lambda)^{1/(2d)})^d \bigr)^{\frac{1}{2\delta}} i^{-\frac{1}{2\delta}}.
\end{align*}
\end{lemma}

The following definition uses Rademacher sequences to introduce a new type of expectation of suprema, see e.g., Definition 7.9 in \cite{StCh08}. This new type will then be used to bound the capacity of function set $\mathcal{H}_r$ with the help of 
the capacity estimate of the binary-valued function set $\mathcal{H}_r^b$.

\begin{definition}[Empirical Rademacher Average] \label{def::RademacherDefinition}
Let $\{\varepsilon_i\}_{i=1}^m$ be a Rademacher sequence with respect to some distribution $\nu$, that is, a sequence of i.i.d.~random variables such that $\nu(\varepsilon_i = 1) = \nu(\varepsilon_i = -1) = 1/2$. The $n$-th empirical Rademacher average of $\mathcal{F}$ is defined as
\begin{align*}
\mathrm{Rad}_\mathrm{D} (\mathcal{F}, n)
:= \mathbb{E}_{\nu} \sup_{h \in \mathcal{F}} 
\biggl| \frac{1}{n} \sum_{i=1}^n \varepsilon_i h(x_i) \biggr|.
\end{align*}
\end{definition}

\begin{lemma}\label{lem::EERad}
Let $\mathcal{H}_r^b$ and $\mathcal{H}_r$ be defined as in \eqref{equ::Hr} and \eqref{equ::Htilder}, respectively. Then for all $\delta \in (0, 1)$, there exist constants $c'_1(\delta)$, $c'_2(\delta)$, $c''_1(\delta)$, and $c''_2(\delta)$ depending on $\delta$ such that
\begin{align*}
\mathbb{E}_{D \sim \mathrm{P}^n} \mathrm{Rad}_{\mathrm{D}}(\mathcal{H}_r^b, n) 
\leq \max \Bigl\{ c'_1(\delta) \lambda^{-\frac{1}{4}} r^{\frac{3-2\delta}{4}} n^{-\frac{1}{2}}, 
c'_2(\delta) \lambda^{-\frac{1}{2(1+\delta)}} r^{\frac{1}{2(1+\delta)}} n^{-\frac{1}{1+\delta}} \Bigr\}
\end{align*}
and 
\begin{align*}
\mathbb{E}_{D \sim \mathrm{P}^n} \mathrm{Rad}_{\mathrm{D}}(\mathcal{H}_r, n) 
\leq \max \Bigl\{ c''_1(\delta) \lambda^{-\frac{1}{4}} r^{\frac{3-2\delta}{4}} n^{-\frac{1}{2}}, 
c''_2(\delta) \lambda^{-\frac{1}{2(1+\delta)}} r^{\frac{1}{2(1+\delta)}} n^{-\frac{1}{1+\delta}} \Bigr\}.
\end{align*}
\end{lemma}

\subsubsection{Oracle Inequality for Single NHT} \label{sec::OracleTreeC0}

Now we are able to establish an oracle inequality for the single na\"{i}ve histogram transform regressor $f_{\mathrm{D},H_n}$ based on the least squares loss and determining rule \eqref{equ::fDH}.

\begin{theorem}\label{thm::OracleTree}
Let the histogram transform $H_n$ be defined as in \eqref{HistogramTransform} with bin width $h_n$ satisfying Assumption \ref{assumption::h}, and $f_{\mathrm{D}, H_n}$ be defined in \eqref{equ::fDH}. 
Then for all $\tau > 0$ and $\delta \in (0,1)$, the single na\"{i}ve histogram transform regressor satisfies
\begin{align*}
& \lambda_n \underline{h}_{0,n}^{-2d} + \mathcal{R}_{L,\mathrm{P}}(f_{\mathrm{D},H_n}) - \mathcal{R}_{L,\mathrm{P}}^* 
\\
& \leq 9 \bigl( \lambda(\underline{h}_{0,n}^*)^{-2d} 
+ \mathcal{R}_{L,\mathrm{P}}(f_{\mathrm{P},H_n}) - \mathcal{R}_{L,\mathrm{P}}^* \bigr)
+ 3 c \lambda_n^{-\frac{1}{1+2\delta}} n^{-\frac{2}{1+2\delta}}
+ 3456 M^2 \tau / n
\end{align*}
with probability $\mathrm{P}^n$ not less than $1-3e^{-\tau}$, where $c$ is some constant depending on $\delta$, $d$, $M$, and $R$ which will be later specified in the proof.
\end{theorem}

Note that the above oracle inequality shows that the excess error can be bounded by approximation error, which is a crucial step in proving the convergence rate.

\subsection{Analysis for NHT in the space $C^{1,\alpha}$}\label{sec::ErrorCons1}

A drawback to the analysis in $C^{0,\alpha}$, as shown in Section \ref{sec::ErrorCons}, is that the usual Taylor expansion involved techniques for error estimation may not apply directly. As a result, we fail to prove the exact benefits of our ensemble estimators over the single one. Therefore, in this subsection, we turn to the function space $C^{1,\alpha}$ consisting of smoother functions. To be specific, we study the convergence rates of $f_{\mathrm{D},\mathrm{E}}$ and  $f_{\mathrm{D},H}$ to the Bayes decision function $f_{L, \mathrm{P}}^* \in C^{1,\alpha}$. To this end, there is a point in introducing some notations.
First of all, for  any fixed $t \in \{ 1, \ldots, T \}$, we define
\begin{align} \label{def::localBayesHt}
f^*_{\mathrm{P},H_t}(x) = \mathbb{E}_{\mathrm{P}} \bigl( f_{L, \mathrm{P}}^*(X) | A_{H_t}(x) \bigr),
\qquad
x \in \mathrm{supp}(\mathrm{P}_X),
\end{align}
where $\mathbb{E}_{\mathrm{P}}(\cdot | A_{H_t}(x))$ denotes the conditional expectation with respect to $\mathrm{P}$ on $A_{H_t}(x)$. With the ensembles of the population version
\begin{align}\label{equ::interes}
f^*_{\mathrm{P}, \mathrm{E}}(x) := \frac{1}{T} \sum^T_{t=1} f^*_{\mathrm{P},H_t}(x),
\end{align}
we make the error decomposition
\begin{align}     \label{equ::L2de}
\mathbb{E}_{\nu_n} \bigl( \mathcal{R}_{L, \mathrm{P}}(f_{\mathrm{D},\mathrm{E}}) - \mathcal{R}_{L, \mathrm{P}}^*\bigr)
&=\mathbb{E}_{\nu_n}\mathbb{E}_{\mathrm{P}_X} \bigl( f_{\mathrm{D},\mathrm{E}}(X) - f_{L, \mathrm{P}}^*(X) \bigr)^2
\nonumber
\\
&= \mathbb{E}_{\nu_n}\mathbb{E}_{\mathrm{P}_X} \bigl( f_{\mathrm{D},\mathrm{E}}(X) - f^*_{\mathrm{P},\mathrm{E}}(X) \bigr)^2
\nonumber
\\
&\phantom{=}
+ \mathbb{E}_{\nu_n}\mathbb{E}_{\mathrm{P}_X} \bigl( f^*_{\mathrm{P},\mathrm{E}}(X) - f_{L, \mathrm{P}}^*(X) \bigr)^2.
\end{align}
In our study, the consistency and convergence analysis of the histogram transform ensembles $f_{\mathrm{D},\mathrm{E}}$
in the space $C^{1,\alpha}$ will be mainly conducted with the help of the decomposition \eqref{equ::L2de}.

In particular, in the case that $T = 1$, i.e., when there is only single na\"{i}ve histogram transform regressor, we are concerned with the lower bound of $f_{\mathrm{D}, H}$ to $f_{L, \mathrm{P}}^*$. With the population version
\begin{align}
f^*_{\mathrm{P},H}(x) := \mathbb{E}_{\mathrm{P}} ( f_{L, \mathrm{P}}^*(X) | A_H(x) ),
\qquad
x \in \mathrm{supp}(\mathrm{P}_X),
\end{align}
we make the error decomposition
\begin{align}     
\mathbb{E}_{\nu_n} \bigl( \mathcal{R}_{L, \mathrm{P}}(f_{\mathrm{D},H}) - \mathcal{R}_{L, \mathrm{P}}^*\bigr)
&=\mathbb{E}_{\nu_n}\mathbb{E}_{\mathrm{P}_X} \bigl( f_{\mathrm{D},H}(X) - f_{L, \mathrm{P}}^*(X) \bigr)^2
\nonumber
\\
&= \mathbb{E}_{\nu_n}\mathbb{E}_{\mathrm{P}_X} \bigl( f_{\mathrm{D},H}(X) - f^*_{\mathrm{P},H}(X) \bigr)^2
\nonumber
\\
&\phantom{=}
+ \mathbb{E}_{\nu_n}\mathbb{E}_{\mathrm{P}_X} \bigl( f^*_{\mathrm{P},H}(X) - f_{L, \mathrm{P}}^*(X) \bigr)^2.
\label{equ::L2Decomposition}
\end{align}
It is important to note that both of the two terms on the right-hand side of \eqref{equ::L2de} and \eqref{equ::L2Decomposition} are data- and partition-independent due to the expectation with respect to $\mathrm{D}$ and $H$. Loosely speaking, the first error term corresponds to the expected estimation error of the estimators $f_{\mathrm{D},\mathrm{E}}$ or $f_{\mathrm{D},H}$, while the second one demonstrates the expected approximation error.

\subsubsection{Bounding the Approximation Error for Ensemble NHT} \label{sec::AError3}

In this subsection, we firstly establish the upper bound for the approximation error term of histogram transform ensembles $f_{\mathrm{P},\mathrm{E}}$ and further find a lower bound of this error for single estimator $f_{\mathrm{P},H}$.

\begin{proposition}\label{prop::biasterm}
Let the histogram transform $H$ be defined as in \eqref{HistogramTransform} with bin width $h$ satisfying Assumption \ref{assumption::h} and $T$ be the number of single estimators contained in the ensembles. 
Furthermore, let $\mathrm{P}_X$ be the uniform distribution and $L_{\overline{h}_0}(x,y,t)$ be the restricted least squares loss defined as in \eqref{RestrictedLS}.
Moreover, let the Bayes decision function satisfy $f_{L, \mathrm{P}}^* \in C^{1,\alpha}$. 
Then for all $\tau > 0$, there holds
\begin{align}\label{eq::PE}
\mathcal{R}_{L_{\overline{h}_0},\mathrm{P}}(f_{\mathrm{P},\mathrm{E}}^*)-\mathcal{R}_{L_{\overline{h}_0},\mathrm{P}}^* 
\leq c_L^2 \overline{h}_0^{2(1+\alpha)} + \frac{1}{T} \cdot d c_L^2 \overline{h}_0^2
\end{align}
in expectation with respect to $\mathrm{P}_H$.
\end{proposition}

\subsubsection{Bounding the Sample Error for Ensemble NHT} \label{sec::SError3}

\begin{lemma}\label{lem::VCreg}
Let the function space $\mathcal{F}_H$ be defined as in \eqref{mathcalFH}.
Then we have
\begin{align*}
\mathrm{VC}(\mathcal{F}_H) \leq 
(2(d+1)(2^d-1)+2) \biggl( \biggl\lfloor \frac{2R\sqrt{d}}{\underline{h}_0} \biggr\rfloor+1 \biggr)^d.
\end{align*}
Moreover, for any probability measure $\mathrm{Q}$ on $X$, there holds
\begin{align*}
\mathcal{N}(\mathcal{F}_H, L_2(\mathrm{Q}), M \varepsilon) \leq 2K(c_d R/\overline{h}_0)^d(16e)^{2(c_d R/\overline{h}_0)^d}(1/\varepsilon)^{4(c_d R/\overline{h}_0)^d}.
\end{align*}
\end{lemma}

\begin{lemma}\label{conv}
Let $\mathrm{Co}(\mathcal{F}_H)$ be the convex hull of $\mathcal{F}_H$, then 
for any probability measure $\mathrm{Q}$ on $X$, there holds
\begin{align*}
\log\mathcal{N}(\mathrm{Co}(\mathcal{F}_H), L_2(\mathrm{Q}), M \varepsilon) 
\leq K \big(1/\varepsilon)^{2 - 2/(8(c_d R/\underline{h}_0)^d+1)}.
\end{align*}
\end{lemma}

\subsubsection{Oracle Inequality for Ensemble NHT} 

\begin{proposition}\label{thm::OracleForest}
Let the histogram transform $H_n$ be defined as in \eqref{HistogramTransform} with bin width $h$ satisfying Assumption \ref{assumption::h} and $\overline{h}_{0,n} \leq 1$. 
Let $f_{\mathrm{D}, \mathrm{E}}$ and $f_{\mathrm{P}, \mathrm{E}}$ be defined in \eqref{equ::fDE} and  \eqref{equ::fPE} respectively. 
Then for all $\tau > 0$ and $\delta \in (0,1)$, the single na\"{i}ve histogram transform regressor satisfies
\begin{align*}
& \lambda_n \underline{h}_{0,n}^{-2d} + \mathcal{R}_{L,\mathrm{P}}(f_{\mathrm{D},\mathrm{E}}) - \mathcal{R}_{L,\mathrm{P}}^* 
\\
& \leq 9 \bigl( \lambda(\underline{h}_{0,n}^*)^{-2d} 
+ \mathcal{R}_{L,\mathrm{P}}(f_{\mathrm{P},\mathrm{E}}) - \mathcal{R}_{L,\mathrm{P}}^* \bigr)
+ 3 c \lambda_n^{-\frac{1}{1+2\delta}} n^{-\frac{2}{1+2\delta}}
+ 3456 M^2 \tau / n
\end{align*}
with probability $\mathrm{P}^n$ not less than $1-3e^{-\tau}$, where $c$ is some constant depending on $\delta$, $d$, $M$, and $R$ which will be later specified in the proof.
\end{proposition}

\subsubsection{Lower Bound of the Approximation Error for Single NHT} \label{sec::AError4}

\begin{proposition}\label{counterapprox}
Let the histogram transform $H$ be defined as in \eqref{HistogramTransform} with bin width $h$ satisfying Assumption \ref{assumption::h} and $\overline{h}_0 \leq 1$.
Moreover, let the regression model defined by \eqref{regressionmodel} with $f \in C^{1,\alpha}$. For a fixed constant $\underline{c}_f \in (0, \infty)$, let $\mathcal{A}_f$ be defined as in \eqref{DegenerateSetF} and $N'$ be defined as in \eqref{N'}.
Then for all $n > N'$,  there holds
\begin{align*}
\mathcal{R}_{L,\mathrm{P}}(f^*_{\mathrm{P},H})
- \mathcal{R}_{L,\mathrm{P}}^*
\geq \frac{d}{12} \biggl( \frac{R}{2} \biggr)^d c_0^2
\mathrm{P}_X(\mathcal{A}_f) \underline{c}_f^2 \cdot
\overline{h}_0^2 
\end{align*}
in expectation with respect to $\mathrm{P}_H$.
\end{proposition}

\subsubsection{Lower Bound of the Sample Error for Single NHT} \label{sec::SError4}

\begin{proposition}\label{counterapprox2}
Let the histogram transform $H$ be defined as in \eqref{HistogramTransform} with bin width $h$ satisfying Assumption \ref{assumption::h}. Let the the regression model be defined as in \eqref{regressionmodel} with $f \in C^{1,\alpha}$. Moreover, assume that $\varepsilon$ is independent of $X$ such that $\mathbb{E}(\varepsilon | X) = 0$ and $\mathrm{Var}(\varepsilon|X) =: \sigma^2 \leq 4 M^2$ hold almost surely for some $M > 0$. Then there holds
\begin{align*}
\mathcal{R}_{L,\mathrm{P}}(f_{\mathrm{D},H})
- \mathcal{R}_{L,\mathrm{P}}(f^*_{\mathrm{P},H})
\geq 4 R^d \sigma^2 (1 - 2 e^{-1}) c_0^d \cdot \overline{h}_0^{-d} \cdot n^{-1}
\end{align*}
in expectation with respect to $\mathrm{P}^n$,
where the constant $c_0$ is as in Assumption \ref{assumption::h}.
\end{proposition}

\subsection{Analysis for KHT in the space $C^{k,\alpha}$}\label{sec::ErrorSVM}

\subsubsection{Bounding the Approximation Error Term}\label{sec::AError2}

Recall that the target function $f_{L,\mathrm{P}}^*$ is assumed to satisfy $(k,\alpha)$-H\"{o}lder continuity condition, to derive the bound for approximation error of KHT, there is a need to introduce another device to measure the smoothness of functions, that is, the modulus of smoothness (see e.g. \cite{DeVore1993Constructive}, p. 44; \cite{DeVore1988Interpolation}, p. 398; as well as \cite{Berens1978Quantitative}, p. 360). 
Denote by $\|\cdot \|_2$ the Euclidean norm and let $\mathcal{X} \subset B_R \subset \mathbb{R}^d$ be a subset with non-empty interior, $\nu$ be an arbitrary measure on $\mathcal{X}$, $p \in (0,\infty]$, and $f: \mathcal{X} \to \mathbb{R}$ be contained in $L_p(\nu)$. Then, for $q \in \mathbb{N}$, the $q$-th modulus of smoothness of $f$ is defined by
\begin{align} \label{q-thModulusSmoothness}
\omega_{q,L_p(\nu)}(f, t) 
:= \sup_{\|h\|_2 \leq t} \|\triangle_h^q(f, \cdot)\|_{L_p(\nu)}, 
\qquad 
t \geq 0,
\end{align}
where $\triangle_h^q(f,\cdot)$ denotes the $q$-th difference of $f$ given by
\begin{align} \label{q-thDifference}
\triangle_h^q(f,x) =
\begin{cases}
\sum_{j=0}^q  \binom{q}{j} (-1)^{q-j} f(x + j h) & \text{ if } x \in \mathcal{X}_{q,h}
\\
0 & \text{ if } x \notin \mathcal{X}_{q,h}
\end{cases}
\end{align}
for $h = (h_1,\ldots, h_d)\in \mathbb{R}^d$ and $\mathcal{X}_{q,h} := \{x\in \mathcal{X}:x+th\in\mathcal{X} \text{ f.a. } t \in [0,q] \}$. Moreover, for fixed $\gamma_j > 0$, we define the function $K_j : \mathbb{R}^d \to \mathbb{R}$ by
\begin{align}\label{Conv}
K_j(x) := \sum_{\ell=1}^{k+1} {k+1 \choose \ell} (-1)^{1-\ell} 
\biggl( \frac{2}{\ell^2 \gamma_j^2 \pi} \biggr)^{d/2} 
\exp \biggl( - \frac{2\|x\|_2^2}{\ell^2 \gamma_j^2} \biggr).
\end{align}
Then we use the convolution with the kernel $K_j$ to approximate the target function $f^*_{L, \mathrm{P}} \in C^{k,\alpha}(B_R)$ in terms of $L_{\infty}$-norm.

\begin{proposition}\label{thm::convolution}
Assume that $\mathrm{P}_X$ is a finite measure on $\mathbb{R}^d$ with $\mathrm{supp}(\mathrm{P}_X) =: \mathcal{X} \subset B_R$. Let $(A'_j)_{j=1,\ldots,m}$ be a partition of $B_R$. Then, $A_j := A'_j \cap \mathcal{X}$ for all $j \in \{1,\ldots,m\}$ defines a partition $(A_j)_{j=1,\ldots,m}$ of $\mathcal{X}$. Furthermore, suppose that $f \in C^{k, \alpha}(\mathcal{X})$. For the functions $K_j$, $j \in \{1, \ldots, m\}$, defined by \eqref{Conv}, where $\gamma_1, \ldots, \gamma_m > 0$, we then have
\begin{align*}
\biggl\| \sum_{j=1}^m \eins_{A_j} \cdot (K_j * f) - f \biggr\|_{L_{\infty}(\nu)} 
\leq c_{k, \alpha} \biggl(\frac{\overline{\gamma}}{\underline{\gamma}}
\biggr)^{\frac{d}{2}} \overline{\gamma}^{s},
\end{align*}
where the constant
$c_{k, \alpha} := c_L \pi^{-\frac{1}{4}} 2^{-\frac{k + \alpha}{2}-\frac{1}{2}} 
d^{\frac{k + \alpha}{2}+1} \Gamma^{\frac{1}{2}} (k + \alpha+\frac{1}{2})$.
\end{proposition}

\subsubsection{Bounding the Sample Error Term}\label{sec::SError2}

In this section, in order to bound the sample error, we derive some results related to the capacity of the function spaces.
First of all, Lemma \ref{lem::CoveringNumber} shows that the covering number of the direct sum of subspaces can be upper bounded by the product of the covering number of these subspaces. Then Lemma \ref{lem::jointspace} establishes the upper bound of the covering number of the composition of two function subspaces of interest, that is, $B_{\mathcal{H}}$ and $\eins_{\pi_h}$. Finally, 
in Proposition \ref{lem::jointspacecoveringnumber},
we give the upper bound on the capacity of the composite function space by means of expected entropy numbers.

\begin{lemma}\label{lem::CoveringNumber}
Let $\mathrm{P}_X$ be a distribution on $\mathcal{X}$ and $A, B \subset \mathcal{X}$ with $A \cap B = \emptyset$. Moreover, let $\mathcal{H}_A$ and $\mathcal{H}_B$ be RKHSs on $A$ and $B$ that are embedded into $L_2(\mathrm{P}_{X|A})$ and $L_2(\mathrm{P}_{X|B})$, respectively. Let the extended RKHSs $\widehat{\mathcal{H}}_A$ and $\widehat{\mathcal{H}}_B$ be defined as in \eqref{entendedRKHS} and denote their direct sum by $\mathcal{H}$ as in \eqref{DirectSum}, where the norm is given by \eqref{DirectSumNorm} with $\lambda_A, \lambda_B > 0$. Then, for the $\varepsilon$-covering number of $\mathcal{H}$ w.r.t.~$\| \cdot \|_{L_2(\mathrm{P}_X)}$, there holds
\begin{align*}
\mathcal{N}(B_{\mathcal{H}}, \|\cdot\|_{L_2(\mathrm{P}_X)}, \varepsilon) 
\leq \mathcal{N}\bigl( \lambda_A^{-1/2} B_{\widehat{\mathcal{H}}_A}, \|\cdot\|_{L_2(\mathrm{P}_{X|A})}, \varepsilon_A \bigr)
\cdot \mathcal{N} \bigl( \lambda_B^{-1/2} B_{\widehat{\mathcal{H}}_B}, \|\cdot\|_{L_2(\mathrm{P}_{X|B})}, \varepsilon_B \bigr),
\end{align*}
where $\varepsilon_A, \varepsilon_B > 0$ and $\varepsilon := (\varepsilon_A^2 + \varepsilon_B^2)^{1/2}$.
\end{lemma}

Recall from \eqref{Bh} that $\pi_h$ is defined as the collection of all cells in $\pi_H$. Therefore, for any $H \sim \mathrm{P}_H$, we have $A_j \in \pi_h$ for all $j \in \mathcal{I}_H$. In what follows, we aim at bounding the complexity of $B_{\mathcal{H}}\circ \eins_{\pi_h}$, that is, the composite space of the partition space $\eins_{\pi_h}$ and RKHS $B_{\mathcal{H}}$.

\begin{lemma}\label{lem::jointspace}
Let $B_{\mathcal{H}}$ be the unit ball of the RKHS $\mathcal{H}$ over $\mathcal{X}$ with the Gaussian kernel. Concerning with the joint space of $B_{\mathcal{H}} \circ \eins_{\pi_h}$, where $B_{\mathcal{H}} \circ \eins_{\pi_h} = \{f\circ g : f \in B_{\mathcal{H}}, g \in \eins_{\pi_h} \}$, there holds
\begin{align*}
\mathcal{N}(B_{\mathcal{H}}\circ \eins_{\pi_h}, \|\cdot\|_{L_2(\mathrm{P}_X)}, 2\varepsilon) \leq \mathcal{N}(\eins_{\pi_h}, \|\cdot\|_{L_2(\mathrm{P}_X)}, \varepsilon) \cdot \mathcal{N}(B_{\mathcal{H}}, \|\cdot\|_{L_2(\mathrm{P}_X)}, \varepsilon).
\end{align*}
\end{lemma}

With the help of the above lemmas we present
the following proposition 
which gives the upper bound for the expected entropy numbers of
the localized RKHS of Gaussian RBF kernels.

\begin{proposition} \label{lem::jointspacecoveringnumber}
Let $A_j \subset \mathcal{X}, j \in \mathcal{I}_H$ be pairwise disjoint partitions induced by the histogram transform $H$. For $j \in \mathcal{I}_H$, let $\mathcal{H}_j$ be a separable RKHS of a measurable kernel $k_{\gamma_j}$ over $A_j$ such that $\|k_{\gamma_j}\|_{L_2(\mathrm{P}_{X|A_j})}^2 < \infty$. Moreover, define the zero-extended RKHSs $(\widehat{\mathcal{H}}_j)_{j \in \mathcal{I}_H}$ by \eqref{entendedRKHS} and the joined RKHS $\mathcal{H}$ by \eqref{DirectSum} with the norm \eqref{DirectSumNorm}. 
Then, there exist constants $p \in (0,1)$ and $a_j'$ such that
\begin{align*}
\mathbb{E}_{D \sim \mathrm{P}^n}e_i(\lambda_{2,j}^{-1/2} B_{\widehat{\mathcal{H}}_j} \circ \eins_{A_j}, \|\cdot\|_{L_2(\mathrm{P}_X)}) \leq a_j' i^{-\frac{1}{2p}}, 
\qquad 
i \geq 1,
\end{align*}
where $a_j'$ satisfies
\begin{align}
\biggl( \sum_{j \in \mathcal{I}_H} \max \{ a_j', B \} \biggr)^{2p}
& \leq 2^{2p} 3 \ln(4) 4^{2p} c_p^{2p} (\sqrt{d} \cdot \overline{h}_0)^d |\mathcal{I}_H|^{1-p}
\biggl( \sum_{j \in \mathcal{I}_H} \lambda_{2,j}^{-1} \mathrm{P}_X(A_j) \gamma_j^{-\frac{d+2p}{p}} \biggr)^p
\nonumber\\
& \phantom{=}
+ 2^{2p} |\mathcal{I}_H|^{2p} \frac{2^{d+6}}{2pe} +2^{2p} |\mathcal{I}_H|^{2p} \biggl( \frac{B}{2} \biggr)^{2p}.
\label{equ::EntropyNum}
\end{align}
\end{proposition}

\subsubsection{Oracle Inequality for Single KHT} \label{sec::OracleKernelCk}

Now we are able to establish an oracle inequality to bound the excess risk for the single KHT $f_{\mathrm{D}, \gamma,H_n}$ based on the least squares loss and determining rule \eqref{equ::SVM}.

\begin{proposition}\label{prop::oracleCk}
For all $j =1,\ldots,m$, let $L : \mathcal{X}\times \mathcal{Y} \times \mathbb{R} \to [0,\infty)$ be a locally Lipschitz continuous loss that can be clipped at $M > 0$ and satisfies the supremum bound for a $B > 0$. Moreover, let $\mathcal{H} = \oplus_{j=1}^m \widehat{\mathcal{H}}_{\gamma_j}$ be the direct sum of separable RKHSs of related measurable kernels $k_{\gamma_j}$ over $A_j$ and $\mathrm{P}$ be a distribution on $\mathcal{X}\times\mathcal{Y}$ such that the variance bound is satisfied for constants $\vartheta \in [0,1]$, $V \geq B^{2-\vartheta}$, and all $f \in \mathcal{H}$. Assume that for fixed $n\geq 1$ there exist constants $p \in (0,1)$ and $a_j' \geq B$ such that
\begin{align*}
\mathbb{E}_{\mathrm{D}_j \sim \mathbb{P}^{|D_j|}} e_i(\mathrm{id}:\widehat{\mathcal{H}}_{\gamma_j} \to L_2(\mathrm{D}_j)) \leq a_j' i^{-\frac{1}{2p}} \qquad i\geq 1.
\end{align*}
Finally, fix an $f_0 \in \mathcal{H}$ and a constant $B_0\geq B$ such that $\|L\circ f_0\|_{\infty} \leq B_0$. Then, for all fixed $\tau >0$, the SVM derived by \eqref{equ::SVM} satisfies
\begin{align*}
&\lambda_1 (\underline{h}_0^*)^{q}+ \lambda_2\|f_{D,\gamma}\|_{\mathcal{H}}^2+\mathcal{R}_{L,\mathrm{P}}(f_{D,\gamma})-\mathcal{R}_{L,\mathrm{P}}^*
\\
&\leq 9(\lambda_1 \underline{h}_0^{q}+ \lambda_2\|f_0\|_{\mathcal{H}}^2+\mathcal{R}_{L,\mathrm{P}}(f_0)-\mathcal{R}_{L,\mathrm{P}}^*)
\\
&\phantom{=}+K\bigg(\frac{\big(\sum_{j=1}^ma_j'\big)^{2p}}{\lambda_2^p m^{p-1} n} \bigg)^{\frac{1}{2-p-\vartheta-\vartheta p}} +3\bigg(\frac{72V\tau}{n}\bigg)^{\frac{1}{2-\vartheta}}+\frac{15 B_0 \tau}{n}
\end{align*}
with probability $\mathrm{P}^n$ not less than $1-3e^{-\tau}$, where $K\geq 1$ is a constant only depending on $p$, $M$, $B$, $\vartheta$ and $V$.
\end{proposition}

\section{Numerical Experiments} \label{sec::numerical_experiments}

In this section, we present the computational experiments that we have carried out. In Section \ref{sec::subsec::experisetup}, we firstly give a brief account for the generation process of our histogram transforms, following by the other two regression methods and two effective measures of estimation accuracy named Mean Squared Error (\emph{MSE})
and efficiency named Average Running Time (\emph{ART}). We proceed by studying the behavior of our histogram transform ensembles depending on the values of tunable parameters in Section \ref{sec::subsec::parameter}. Then in Section \ref{sec::subsec::syntheticdata} we perform a simulation for synthetic data generated from a regression model to validate the exact difference of convergence rate between ensembles and single estimators. Finally, we compare our approach with other regression estimation methods for real data in terms of \emph{MSE} in Section \ref{sec::subsec::realdata}.

\subsection{Experimental Setup} \label{sec::subsec::experisetup}

\subsubsection{Generation Process for Histogram Transforms}

Firstly, note that the random rotation matrix $R$ is generated in the manner coinciding with Section \ref{subsec::HT}. For the elements of the scaling matrix $S$, applying the well known Jeffreys prior for scale parameters referred to \cite{Jeffreys1946An}, we draw $\log (s_i)$ from the uniform distribution over certain real-valued interval $[\log(\underline{s}_0),\log(\overline{s}_0)]$ with
\begin{align*}
\log (\underline{s}_0) & := s_{\min} +\log (\widehat{s}),
\\
\log (\overline{s}_0) & := s_{\max} +\log (\widehat{s}),
\end{align*}
where $s_{\min}, s_{\max} \in \mathbb{R}$ are tunable parameters with $s_{\min} < s_{\max}$ and the scale parameter $\widehat{s}$ is the inverse of the bin width $\widehat{h}$ measured on the input space, which is defined by
\begin{align*}
\widehat{s} := (\widehat{h})^{-1} = (3.5 \sigma)^{-1} n^{\frac{1}{2+d}}.
\end{align*}
Here, the standard deviation $\sigma := \sqrt{\mathrm{trace}(V)/d}$ with $V := \frac{1}{n-1} \sum_{i=1}^n (x_i - \bar{x})(x_i-\bar{x})^{\top}$ and $\bar{x} := \frac{1}{n} \sum_{i=1}^n x_i$ combines the information from all the dimensions of the input space.

\subsubsection{Performance Evaluation Criterion}

When it comes to the empirical performances for various different regression estimators $\widehat{f}$, two of our biggest concerns are accuracy and efficiency, where appropriate measurements are in demand.

On the one hand, we adopt the ubiquitous Mean Squared Error (\emph{MSE}) conducted over $m$ test samples $\{x_j\}_{j=1}^m$:
\begin{align}\label{equ::MSE}
\text{\emph{MSE}}(\widehat{f}) = \frac{1}{m} \sum_{j=1}^m (y_j - \widehat{f}(x_j))^2
\end{align}
Obviously, the lower \emph{MSE} implies the better performance of a regression function $\widehat{f}$.

On the other hand, we take the Average Running Time (\emph{ART}) of $m$ repeated experiments as the measure of efficiency, that is,
\begin{align}\label{equ::ATT}
\text{\emph{ART}}(\widehat{f}) = \frac{1}{m} \sum_{j=1}^m t_j(\widehat{f}),
\end{align}
where $t_j(\widehat{f})$ denotes the training time of the $j$-th experiment.

Either \emph{MSE}, the measure of accuracy, or \emph{ART}, the representative of efficiency, is not sufficient to be a comprehensive evaluation criterion of an algorithm. For relatively small-scale data sets or synthetic data, the training speed of an algorithm is often fast enough. Therefore, we mainly focus on the precision of the following simulations in Section \ref{sec::subsec::parameter} and \ref{sec::subsec::syntheticdata}. However, for moderate sized or large-scale real data sets, the discrepancy of training time among algorithms is no longer negligible. That is, not only should a good algorithm have desirable predicting accuracy, but it is also expected to be comparable in training time with other state-of-the-art regression methods. Therefore, 
in Section \ref{sec::subsec::realdata},
we consider the trade-off between \emph{MSE} and \emph{ART} in the real data analysis.

\subsection{Study of the Parameters}\label{sec::subsec::parameter}

In this subsection, taking NHTE as an instance, we perform an experiment dealing with the parameters of our HTE algorithm, namely the number of histogram transform estimators $T$ and the lower and upper scale parameters $s_{\min}, s_{\max} \in \mathbb{R}$. In what follows, we consider a synthetic data set following the regression model 
\begin{align}\label{model::syn1}
Y  = \sin(16 X) + \varepsilon,
\end{align} 
where $X \in \mathrm{Unif}[0,1]$ and
$\varepsilon \sim N(0, 0.1^2)$.

We firstly explore the influence of parameter $T$ on the experimental results of our algorithm. For each experiment, the empirical performance will be compared by average \emph{MSE} introduced in \eqref{equ::MSE}.
We have carried out experiments with $n = 2000, 3000, 4000, 5000$, and the number of test samples in each case is $m = 2000$. For every $n$ and $T$ we have made 300 runs of experiments, with fixed $(s_{\min}, s_{\max}) = (0,1)$. The results are shown in Figure \ref{fig::MSET}.

\begin{figure}[H]
\begin{minipage}[t]{0.62\textwidth}  
\centering  
\includegraphics[width=\textwidth]{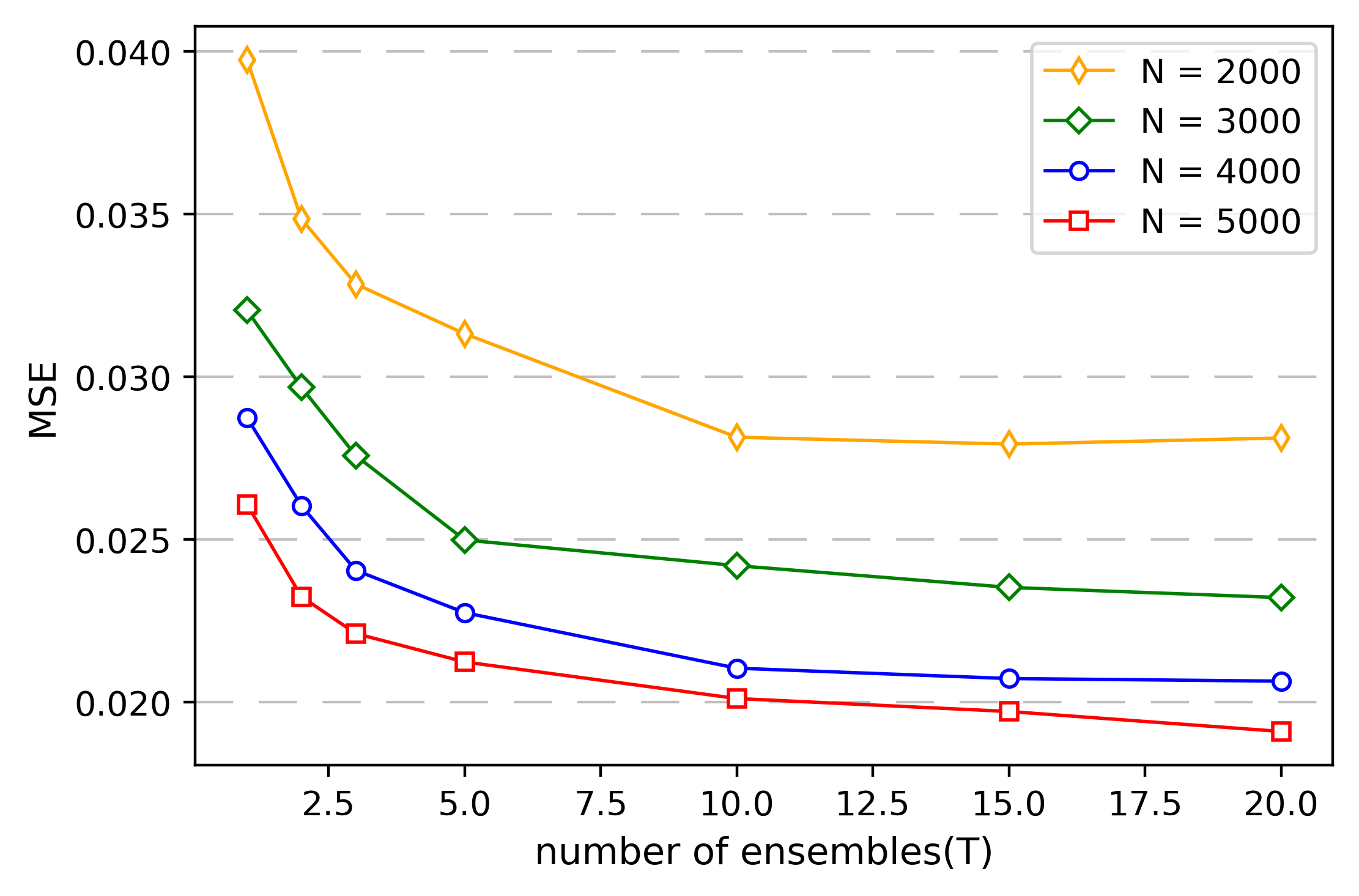}  
\end{minipage}  
\centering  
\caption{Average MSE for different values of $T$ applied for the synthetic dataset.}
\label{fig::MSET}
\end{figure}

As we can see, the performance of our histogram transform estimator enhances as $n$ grows which can be seen from the downward average \emph{MSE} of each line. On the other hand, the results improve dramatically when we go from $T = 1$ to $T = 20$, but then a steady state is reached, no matter how many larger ensembles we consider. This behavior is extremely convenient, since it means that increasing the number of components in an ensemble by raising $T$ does not have any significant effect beyond certain limit. Consequently, we have decided to use $T = 10$ in the subsequent experiment.

We now examine the dependency of our method with respect to the choice of the lower and upper scale parameters $s_{\min}$, $s_{\max}$. We recall that the scale parameters $s_{\min}$, $s_{\max}$ in the distribution of stretching matrix $S$ actually control the size of histograms. If the local structure of the input data set is very detailed, we need high values of both of them to have smaller histogram bins, and vice versa. On the other hand, if the local structure is finer in some regions of the data set and coarser in other regions, we need that both parameters have very different values to cope with the varying scales, while an homogeneous structure can be accommodated with a narrower range of histogram bin sizes. In order to illustrate this, we have obtained our ensemble NHT with $n = 500$ training data and then conducted the experiment with 1000 test observations, for the following values the scale parameters: $s_{\min} = 0$, $s_{\max} = 2$; $s_{\min} = 1$, $s_{\max} = 3$; and $s_{\min} = 2$, $s_{\max} = 4$. The results are shown in Figure \ref{fig::sminsmax}.

\begin{figure}[H]
\centering
{
\begin{minipage}{0.45\textwidth}
	\centering
	\includegraphics[width=\textwidth]{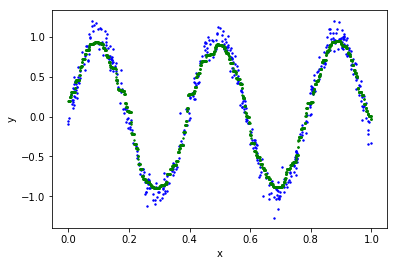}
\end{minipage}
}
{
\begin{minipage}{0.45\textwidth}
	\centering
	\includegraphics[width=\textwidth]{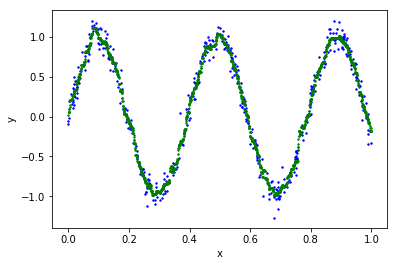}
\end{minipage}
}
{
\begin{minipage}{0.45\textwidth}
	\centering
	\includegraphics[width=\textwidth]{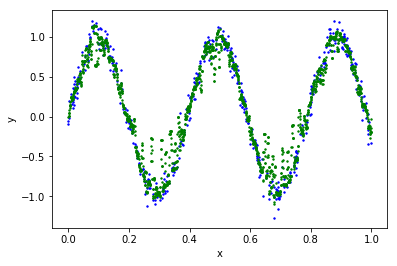}
\end{minipage}
}
\caption{Blue points represent the true sample and green ones are predictive values. 
Upper Left: $s_{\min} = 0$, $s_{\max} = 2$. 
Upper Right: $s_{\min} = 1$, $s_{\max} = 3$. 
Lower: $s_{\min} = 2$, $s_{\max} = 4$.}
\label{fig::sminsmax}
\end{figure}

As seen, lower values of these parameters yield a coarser approximation of the input distribution leading to the loss of precision (see the top left subfigure). Conversely, if the parameters are too high, then there are zones where no training samples exist. On this occasion, chances are high that more predictive points tend to be close to zero (see the lower subfigure). Therefore an optimization procedure is needed to obtain good values for $s_{\min}$ and $s_{\max}$, given an input data set.

To further illustrate the effect of bin width with regard to accuracy, we extend parameter grid of $(s_{\min}, s_{\max})$ to 5 pairs, that is $(-1,1)$, $(0,2)$, $(1,3)$, $(2,4)$, and $(3,5)$. In addition, to ensure the stability of this experimental result, we generate $10$ sets of synthetic data with the generating model \eqref{model::syn1}, and carry out $10$ runs with each set. In other words, we carry out $100$ runs of experiments in total, and 
utilize the average of \emph{MSE} for each experiment to represent the testing error.

\begin{figure}[H]
\begin{minipage}[t]{0.60\textwidth}  
\centering  
\includegraphics[width=\textwidth]{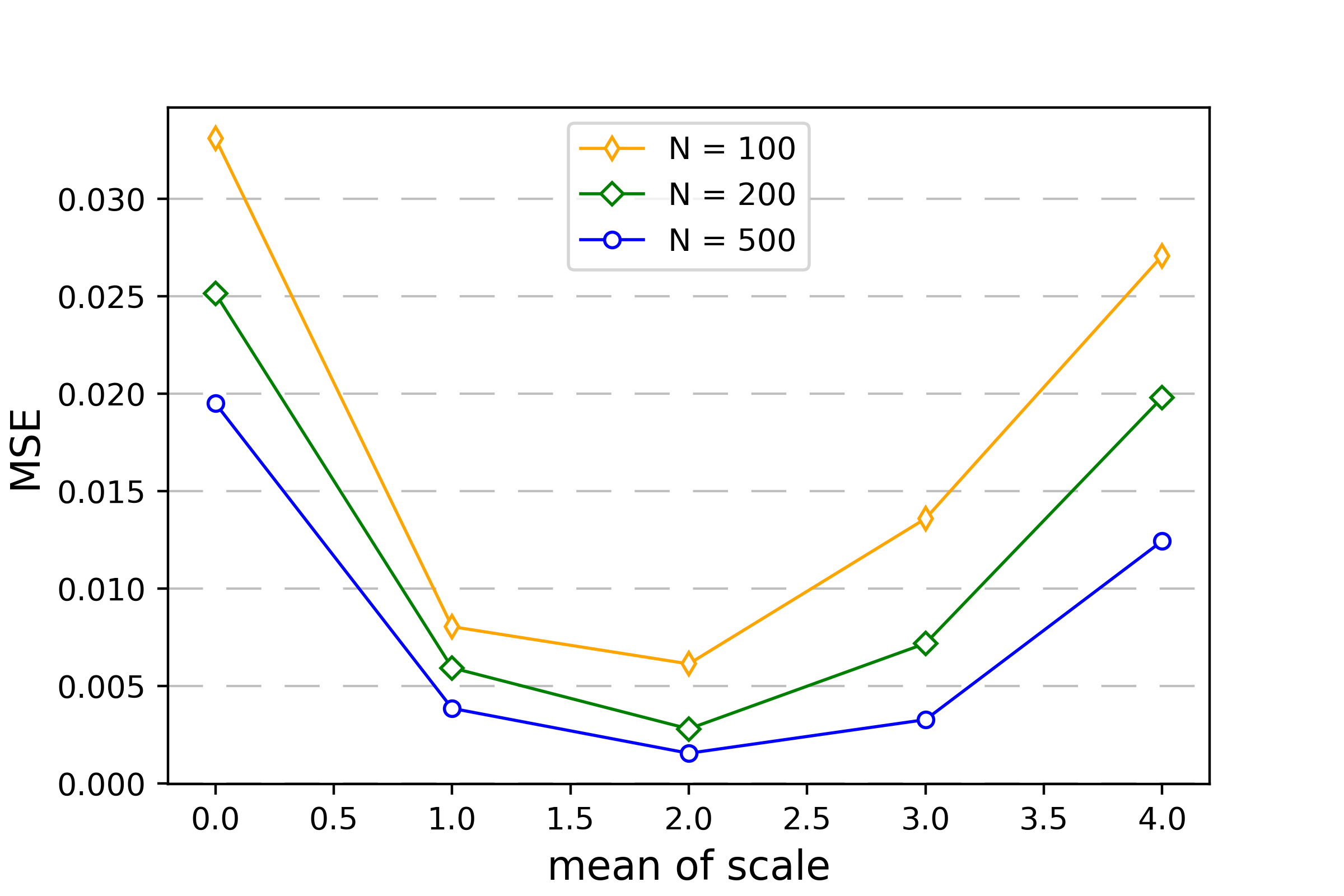}  
\end{minipage}  
\centering  
\caption{Average MSE for different values of $(s_{\min}, s_{\max})$ applied for the synthetic dataset. Note that the x-axis represents for the mean value of $s_{\min}$ and $s_{\max}$.}
\label{fig::MSEs}
\end{figure}

A clear trend can be seen from Figure \ref{fig::MSEs}. When the bin width is relatively large, the average \emph{MSE} for NHTE decreases with $s_{\min}$ and $s_{\max}$ increasing, that is, the empirical performance gets better with bin width decreasing. However, \emph{MSE} then attains the minimum at $(s_{\min}, s_{\max}) = (1,3)$. Subsequently, when the bin width is relatively small, further increasing of $s_{\min}$ and $s_{\max}$ leads to the deterioration of testing error. 
This exactly verifies the theoretical result in Section \ref{sec::upper} that there exists an optimal bin width with regard to the convergence rate.

\subsection{Synthetic Data Analysis} \label{sec::subsec::syntheticdata}

In order to give a more comprehensive understanding of this section, the reader will be reminded of the significance to illustrate the benefits of our histogram transform ensembles over a single estimator. Therefore, we start this simulation by constructing the above mentioned counterexample as the synthetic data. To be specific, we base the simulations on one particular distribution construction approach generating a toy example with dimension $d = 3$. Assume that the regression model for random vector $X = (X_1, X_2, X_3)^{\top} \in \mathbb{R}^3$, 
\begin{align*}
Y = \sum_{i=1}^3 10 X_i \cdot \sin(2 X_i - 3)+ \varepsilon,
\end{align*}
where $X_i \in \mathrm{Unif}[0,1]$, $i = 1, 2, 3$, and
$\varepsilon \sim N(0, 0.1^2)$.

It can be apparently seen that this example is based on all the three dimensions. We perform the synthetic data experiment with $m=1000$ and parameter pair $(s_{\min}, s_{\max})=(0,1)$. For every $T$ and $n$ we repeated the experiment 30 times and the resulting average \emph{MSE} versus $T$ are shown as follows.

\begin{figure}[H]
\begin{minipage}[t]{0.60\textwidth}  
\centering  
\includegraphics[width=\textwidth]{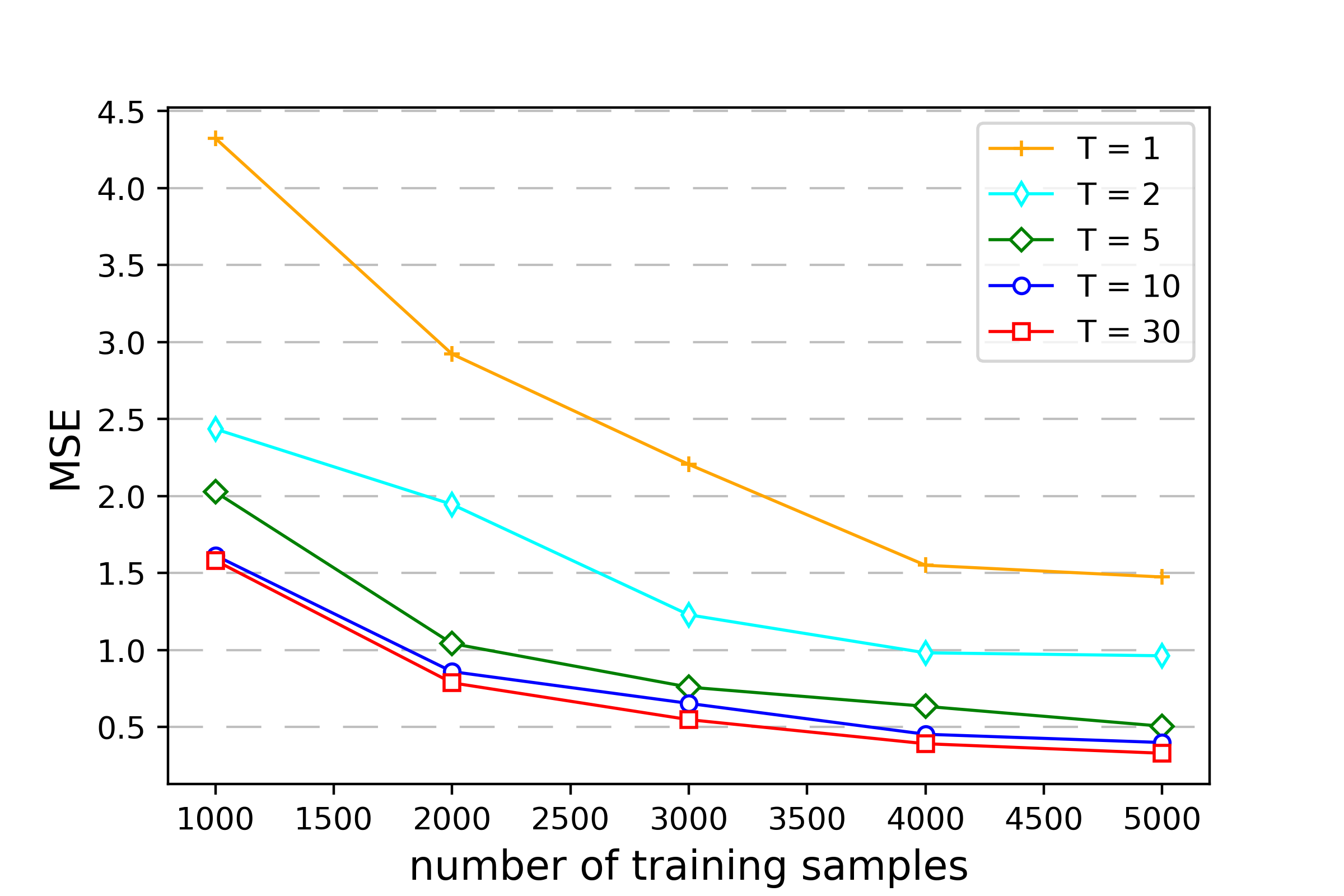}  
\end{minipage}  
\centering  
\caption{Average \textit{MSE} for different values of $T$ applied for the artificial counterexample dataset.}
\label{fig::counterMSET}
\end{figure}

Figure \ref{fig::counterMSET} captures the \emph{MSE} performance of our model for $T = 1, 2, 5, 10, 30$ respectively. The result is twofold: First of all, the lower \emph{MSE} of the steady state for $T > 1$ states that ensembles behave better than single estimator in terms of accuracy. Moreover, the difference of slope before the curves reach flat illustrates the lower bound of the convergence rate of single estimator to some extent.

\subsection{Real Data Analysis} \label{sec::subsec::realdata}

We have designed two sets of experiments with real data and comparisons with other state-of-the-art regression algorithms demonstrate the accuracy and efficiency of our algorithm.

\subsubsection{Adaptive KHTE Algorithm}

Recall that from the view point of algorithm architecture, the essence of our HTE lies in the following facts: firstly,  
the large diversity of random histogram transform and the inherent nature of ensembles help the algorithm overcome the long-standing boundary discontinuity;
on the other hand, taking full advantage of the data-independent partition process, this vertical method successfully achieves high efficiency via parallel computing. Till now, the partition processes considered have only performed in an equal-size histogram manner, however, in order to bring more resistance and taking the local adaptivity into account, all histogram transforms in the following experiments adopt the adaptive random stretching criterion to significantly improve the balancing property of splits and hence to increase the accuracy.

The \textit{adaptive splitting} technique helps formulate a data dependent partition. Instead of selecting the bin indices as the round points, where each cell shares the same size, this adaptive method creates more splits on fractions where samples points are densely resided, while  it splits less on sample-sparse areas. Therefore, every cell in the partition contains roughly the same number of sample points. 
A concrete description of the construction process of \textit{adaptive splitting} is shown in the following Algorithm \ref{alg::adaptiveSplitting}.

\begin{algorithm}[H]
\caption{Adaptive Splitting}
\label{alg::adaptiveSplitting}
\KwIn{
Transformed sample space $D^{\top}$ 
;
\\
\quad\quad\quad\quad Minimal number of samples required to split  $m$; \\
\quad\quad\quad\quad Number of splits $p$ initiated as 1.
\\
}
\Repeat{$\max(\textit{number of samples in all cells}) \leq m$}{
$k_t^p$ is the number of cells before the $p$-th split for the $t$-th partition; 
\\ 
\For{$j =1 \to k_t^p$}{
	\If{$\textit{number of samples in the j-th cell} > m$}{
		Select out the dimension with the largest variance; \\
		Select the split point as the median of samples in this dimension;
	}
}
$p++$. \\
}

\KwOut{Adaptive partition of the transformed sample space $D^{\top}$.}

\end{algorithm}

To avoid a cell to have too less samples or even no sample at all, we impose a stopping criterion when a cell contains less than $m$ samples. Then we focus on every \textit{qualified} cell with enough sample points, and select the to-be-split dimension as the one with the largest variance, and moreover, we choose the split point as the median of samples in the $d$-th dimension. By this means, we're able to make full use of the potential information containing in samples. On one hand, we reckon that the most varied dimension contains the most information. On the other hand, by splitting on the median, we are able to obtain two newly generated cells with even number of samples. Then we repeat this splitting method until all cells meet the stopping criterion.

With the help of \textit{adaptive splittings} and the improved stopping criterion, we are now ready to present our \textit{adaptive} KHTE algorithm.

\begin{algorithm}[H]
\caption{\textit{Adaptive} Kernel Histogram Transform Ensembles (\textit{Adaptive} KHTE)}
\label{alg::adaptiveKHTE}
\KwIn{
Training data $D:=((X_1, Y_1), \ldots, (X_n,Y_n))$;
;
\\
\quad\quad\quad\quad Number of histogram transforms $T$; 
\\
\quad\quad\quad\quad Regularization parameter $\lambda$ and bandwidth parameter of Gaussian kernel $\gamma$.
} 
\For{$t =1 \to T$}{
Generate random affine transform matrix $H_t=R_t$; 
\\
Apply \textit{adaptive splitting} to the transformed sample space; \\
Apply SVM to each cell \&  
compute 
global regression mapping $f_{\mathrm{D},\lambda,\gamma,H_t}(x)$
.
}
\KwOut{The kernel histogram transform ensemble for regression is
\begin{align*}
f_{\mathrm{D},\lambda,\gamma,\mathrm{E}}(x)
= \frac{1}{T} \sum_{t=1}^T f_{\mathrm{D},\lambda,\gamma,H_t}(x).
\end{align*}
}
\end{algorithm}

\subsubsection{Study of Parameters}\label{sec::subsec::realdataParameters}

This subsection delves into the study of parameters $T$ and $m$ in Algorithm \ref{alg::adaptiveSplitting}, that is, the number of partitions in an ensemble and the minimum number of samples required to split an internal node. 
We carry out experiments based on a real data set {\tt PTS}, the \textit{Physicochemical Properties of Protein Tertiary Structure Data Set}, available on UCI. It contains totally $45, 730$ samples of dimension $9$, with $70\%$ samples randomly selected as the training set, and the remaining $30\%$ as the testing set. 
The parameter grids of  $T$ and $m$ are $[1, 2, 5, 10]$ and $[200, 400, 1000, 1500, 2000]$. In addition, all experiments are repeated for $50$ times.

\begin{figure}[H]
\centering
{
\begin{minipage}{0.45\textwidth}
	\centering
	\includegraphics[width=\textwidth]{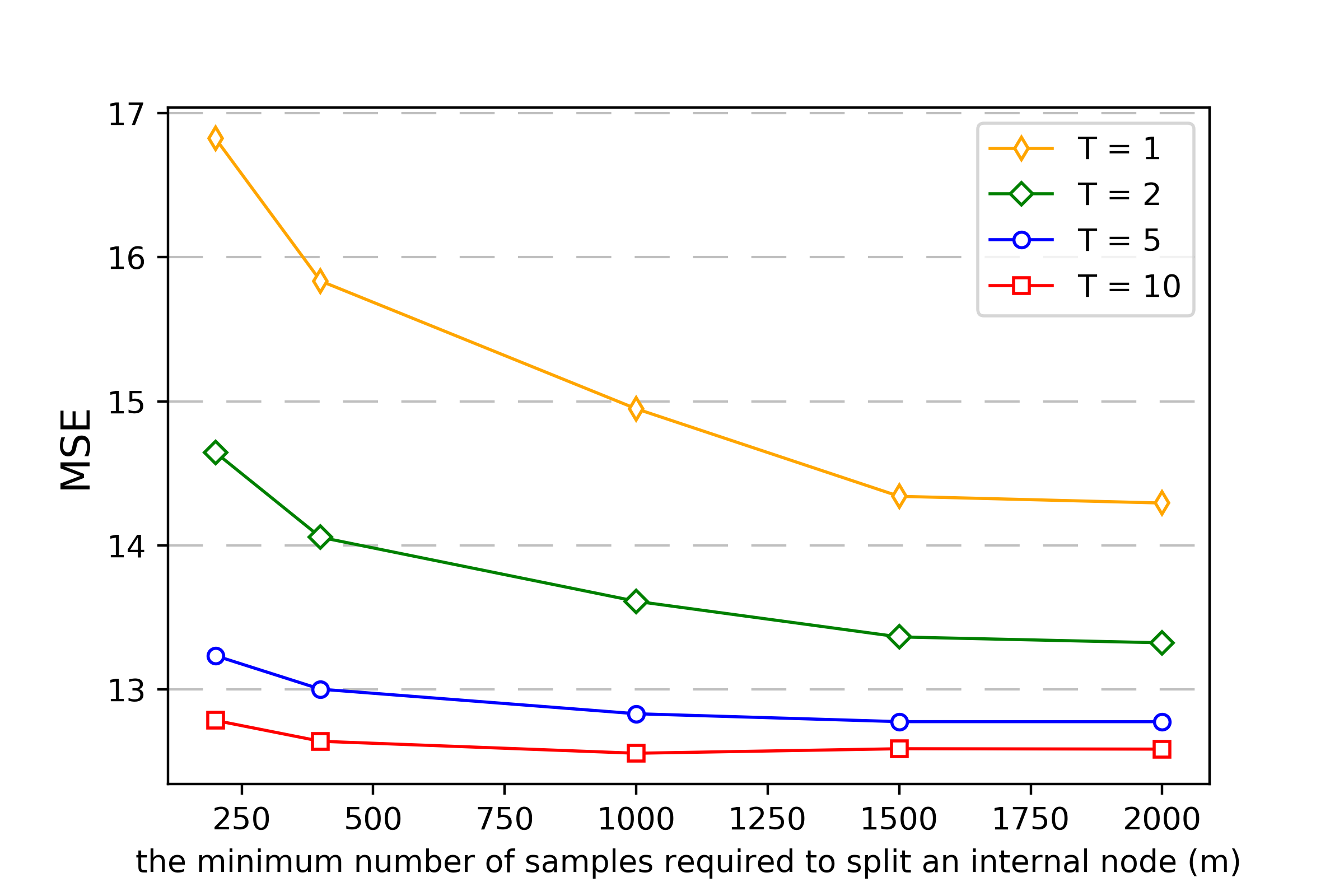}
\end{minipage}
}
{
\begin{minipage}{0.45\textwidth}
	\centering
	\includegraphics[width=\textwidth]{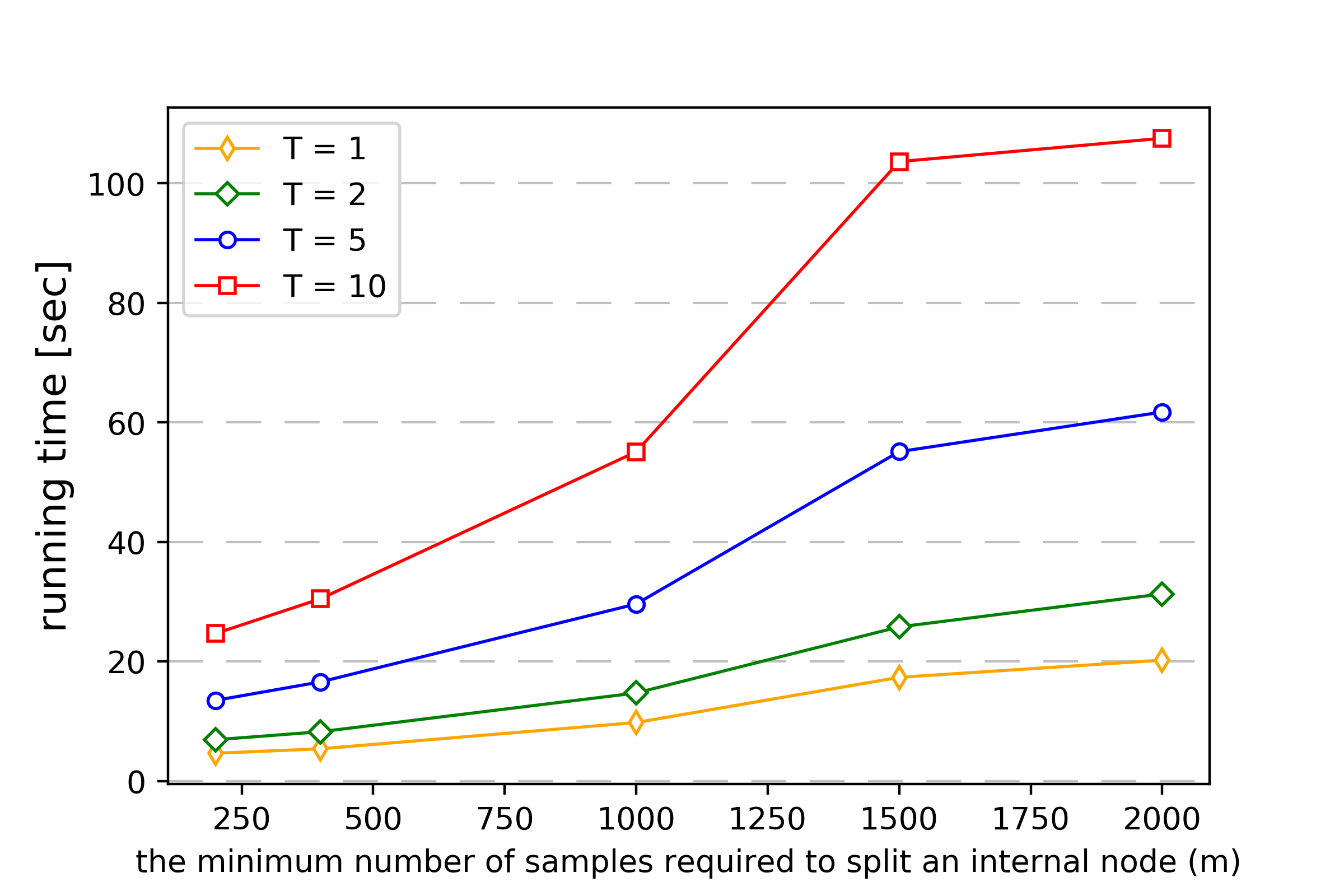}
\end{minipage}
}
\caption{Average \emph{MSE}/\emph{ART} for different values of $T$ and $m$.}
\label{fig::MSEvsm}
\end{figure}

As can be seen from the above figure, on the one hand, for a fixed $m$, when the number of partitions $T$ increases, training error decreases while the corresponding running time increases. On the other hand, when $T$ fixed, we can see \emph{MSE} decreases as $m$, the minimum number of samples required to split, increasing, with sacrifice of training time.

\subsubsection{Introduction to Other Large-scale Regressors}

In our experiments, comparisons are conducted among our adaptive KHTE, Patchwork Kriging (PK), and Voronoi partition SVM (VP-SVM).
\begin{itemize}[leftmargin=*]
\item 
\textbf{PK}: 
Patchwork kriging (PK) proposed by \cite{park2018patchwork} is an approach for Gaussian process (GP) regression for large datasets. 
This method involves partitioning the regression input domain into multiple local regions via spacial tree and apply a different local GP model fitted in each region.  
Different from previous Gaussian process vertical methods put forward in \cite{park11a} and \cite{park16a}, which tried to join up the boundaries of the adjacent local GP models by imposing various equal boundary constraints, PK presents a simple and natural way to enforce continuity by creating additional pseudo-observations around the boundaries. However, there stand some challenges. Firstly, although the employed spatial tree generates data partitioning of uniform sizes when data is unevenly distributed, artificially determined decomposition process brings a great impact on the final predictor. Secondly, this approach loses its competitive edge possessing the desirable global property of GPs as well as suffers from curse of dimensionality. Last but not least, when encountering data with high dimension and large volume, in order to achieve better prediction accuracy, more pseudo-observations need to be added to the boundaries, which leads to a significant growth in computational complexity. 
\item 
\textbf{VP-SVM}: 
Support vector machines for regression being a global algorithm is impeded by super-linear computational requirements in terms of the number of training samples in large-scale applications. To address this, \cite{meister16a} employs a spatially oriented method to generate the chunks in feature space, and fit LS-SVMs for each local region using training data belonging to the region. This is called the Voronoi partition support vector machine (VP-SVM). However, the boundaries are artificially selected and the boundary discontinuities do exist.
\end{itemize}

\subsubsection{Real world Data Set Analysis}

We have designed three sets of experiments on our adaptive KHTE, PK and VP-SVM. All experiments are conducted on the {\tt PTS} data set introduced in Section \ref{sec::subsec::realdataParameters}  and other data sets presented as follows.

\begin{itemize}
\item 
{\tt AEP}: The Appliances energy prediction ({\tt AEP}) data set available on UCI contains $19, 735$ samples of dimension $27$ with attribute “date” removed from the original data set. The data is used to predict the appliances energy use in a low energy building.
\item 
{\tt HPP}: This data set House-Price-8H prototask ({\tt HPP}) is originally from {\tt DELVE} dataset. 
It consists of $22,784$ observations of dimension $8$. Note that for the sake of clarity, all house prices in the original data set has been modified to be counted in thousands.
\item 
{\tt CAD}: This spacial data can be traced back to \cite{Pace1997Sparse}. It consists $20,640$ observations on housing prices with $9$ economic covariates. Similar as the data preprocessing for {\tt HPP}, all house prices in the original data set has been modified to be counted in thousands.
\item 
{\tt MSD}: The \textit{Year Prediction MSD Data Set} ({\tt MSD}) is available on UCI. It contains $463,715$ training samples and $51,630$ testing samples with $90$ attributes, depicting the timbre average and timbre covariance of songs released between 1922 and and 2011. The main task is to learn the audio features of a song and to predict its release year.
\end{itemize}

Samples in data sets {\tt AEP}, {\tt HPP}, {\tt PTS} and {\tt CAD} are scaled to zero mean and unit variance, and experiments carried on such data sets are repeated for 50 times. In addition, we randomly split each data set into training, with $70\%$ of the observations, and testing, containing the remaining $30\%$. Whereas for the {\tt MSD} data set, we respect the following train/test split that the first $463,715$ examples are treated as training set and the last $51,630$ are treated as testing set. In addition, because VP-SVM cannot run {\tt MSD} data set with the above standardization for some reason, data are rescaled such that all feature values are in the range $[0, 1]$. Moreover, experiments for {\tt MSD} data set are repeated for 10 times to obtain a relatively stable result, without consuming too much training time on such a large-scale data set.

In experiment, we set $(T,m)$ pair to be  $(5,1200)$ and $(20,1200)$ except for {\tt MSD} data set, where we select $(5,2000)$ and $(20,3000)$, for the trade off between accuracy and running time. We adopt grid search method for other hyperparameter selections. To be specific, for data sets {\tt HPP}, {\tt CAD}, {\tt PTS} and {\tt AEP}, the regularization parameter $\lambda$ and the kernel bin width $\gamma$ are selected from $7$ and $8$values, respectively, from $10^{-3}$ to $10^3$ and from $0.05$ to $10$, spaced evenly on a log scale with a geometric progression. For {\tt MSD} data set, we choose $\lambda$ from $\{0.01,1,100\}$, and $\gamma$ from $\{0.001, 0.1, 10\}$. We randomly split $30\%$ samples from training sets for validation in hyper-parameter selection.

Now we summarize the comparison results of KHTE, VP-SVM, PK in Table \ref{tab::Realdata}.

\begin{table}[H] 
\setlength{\tabcolsep}{11pt}
\centering
\captionsetup{justification=centering}
\caption{\footnotesize{Average \textit{MSE} and \textit{ART} over real data sets}}
\label{tab::Realdata} 
\resizebox{0.9\textwidth}{30mm}{
\begin{tabular}{ccccccccccc}
	\toprule
	\multirow{2}*{\text{Datasets}}
	& \multirow{2}*{$(n,d)$} 
	& \multicolumn{2}{c}{\text{KHTE (T=5)}}
	& \multicolumn{2}{c}{\text{KHTE (T=20)}}
	& \multicolumn{2}{c}{\text{PK}} 
	& \multicolumn{2}{c}{\text{VP-SVM}}\\
	\cline{3-10}
	\multicolumn{2}{c}{}&\textit{MSE}&\textit{ART}
	&\textit{MSE}&\textit{ART} &\textit{MSE}&\textit{ART} &\textit{MSE}&\textit{ART} \\
	\hline 
	\hline
	\multirow{2}*{{\tt CAD}}
	& \multirow{2}*{$(20640, 9)$} 
	& $2993.64$ & $\textbf{15.38}$ & $\textbf{2951.61}$ & $50.78$ & $3008.84$ & $99.17$    & $3010.75$ & $19.71$ \\ 
	& & $(66.58)$ & $(0.18)$ & $(70.34)$ & $(0.37)$ & $(82.70)$ & $(35.63)$ & $(76.06)$ & $(0.95)$ \\ 
	\hline
	\multirow{2}*{{\tt PTS}}
	& \multirow{2}*{$(45730, 9)$} 
	& $12.78$ & $55.12$ & $\textbf{12.52}$ & $200.43$ & $17.08$ & $176.56$    & $13.74$ & $\textbf{52.33}$ \\ 
	& & $(0.21)$ & $(1.30)$ & $(0.21)$ & $(1.57)$ & $(0.83)$ & $(40.88)$ & $(0.19)$ & $(1.60)$ \\ 
	\hline
	\multirow{2}*{{\tt AEP}}
	& \multirow{2}*{$(19735, 27)$} 
	& $6535.21$ & $21.40$ & $\textbf{6402.25}$ & $71.09$ & $7418.10$ & $132.21$    & $6827.94$ & $\textbf{11.48}$ \\ 
	& & $(369.08)$ & $(0.16)$ & $(358.35)$ & $(0.36)$ & $(461.60)$ & $(48.84)$ & $(341.90)$ & $(0.48)$ \\ 
	\hline
	\multirow{2}*{{\tt HPP}}
	& \multirow{2}*{$(22784, 8)$} 
	& $1260.52$ & $23.16$ & $\textbf{1242.53}$ & $77.34$ & $1349.17$ & $124.27$    & $1272.97$ & $\textbf{14.50}$ \\ 
	& & $(75.98)$ & $(1.06)$ & $(75.43)$ & $(0.97)$ & $(74.56)$ & $(39.72)$ & $(68.79)$ & $(0.88)$ \\ 
	\hline
	\multirow{2}*{{\tt MSD}}
	& \multirow{2}*{$(515345, 90)$} 
	& $82.88$ & $448.30$ & $\textbf{81.05}$ & $1674.63$ & \multirow{2}*{$--$} & \multirow{2}*{$\geq 36$h}    & $85.10$ & $\textbf{386.03}$ \\ 
	& & $(0.15)$ & $(1.81)$ & $(0.15)$ & $(25.43)$ &  & & $(0.00)$ & $(2.48)$ \\ 
	\bottomrule 
\end{tabular}}
\begin{minipage}{0.9\textwidth}
\begin{tablenotes}
	\footnotesize
	\item {*} The best results are marked in \textbf{bold}, and the standard deviation is reported in the parenthesis under each value. Note that, since PK doesn't fit in the parallel computing framework, its training time exceeds a $36$ hour-limit, and thus no average \textit{MSE} is reported.
\end{tablenotes}
\end{minipage}
\end{table}

As it can be seen from Table \ref{tab::Realdata}, our \textit{adaptive} KHTE method with $T=20$ outperforms the other two state-of-the-art algorithms VP-SVM and PK in terms of predicting accuracy, due to high level of smoothness brought about by a relatively large $T$, which, however, leads to more training time sacrificed. Therefore, we turn to the less time consuming case $T=5$. Maintaining desirable accuracy, our KHTE shows comparable or even smaller training time compared with the extremely efficient VP-SVM.

Experimental results presented so far are those we have temporarily tuned. More accurate results can be obtained if we sacrifice more training time, which is different from other methods, for their  accuracy are hard to be increased. Readers  interested in these experiments are encouraged to try various hyperparameters to further investigate even lower testing errors.

\section{Proofs}\label{sec::proofs}

\subsection{Proofs of Results for NHT in the space $C^{0,\alpha}$}

\subsubsection{Proofs Related to Section \ref{sec::AError1}}

\begin{proof}[of Proposition \ref{prop::ApproximatiON-ERROR}]
For a fixed $\underline{h}_0$, we write
\begin{align*}
f_{\mathrm{P},H} := \argmin_{f \in \mathcal{F}_H} \mathcal{R}_{L,\mathrm{P}}(f)-\mathcal{R}_{L,\mathrm{P}}^*.
\end{align*}
In other words, $f_{\mathrm{P},H}$ is the function that minimizes the excess risk $\mathcal{R}_{L,\mathrm{P}}(f)-\mathcal{R}_{L,\mathrm{P}}^*$ over the function set $\mathcal{F}_H$ with bin width $h \in [\underline{h}_0, \overline{h}_0]$. Then, elementary calculation yields
\begin{align*}
f_{\mathrm{P},H}  = \sum_{j \in \mathcal{I}_H} \frac{\int_{A_j} \mathbb{E}(Y|X)\, d\mathrm{P}_X}{\mathrm{P}_X(A_j)}\eins_{A_j}= \sum_{j \in \mathcal{I}_H} \frac{\int_{A_j} f_{L, \mathrm{P}}^*\, d\mathrm{P}_X}{\mathrm{P}_X(A_j)}\eins_{A_j}.
\end{align*}
The assumption $f_{L, \mathrm{P}}^* \in C^{0,\alpha}$ implies
\begin{align*}
\mathcal{R}_{L,\mathrm{P}}(f_{\mathrm{P},H}) - \mathcal{R}_{L,\mathrm{P}}^*
& = \|f_{\mathrm{P},H} - f_{L, \mathrm{P}}^*\|_{L_2(\mathrm{P}_X)}^2
\\
& = \bigg\| \sum_{j \in \mathcal{I}_H}
\frac{\int_{A_j} f_{L, \mathrm{P}}^* (x') \ d \mathrm{P}_X(x')}{\mathrm{P}_X(A_j)}
\eins_{A_j} (x) - \sum_{j\in \mathcal{I}_H} f_{L, \mathrm{P}}^* (x) \eins_{A_j} (x) \bigg\|_{L_2(\mathrm{P}_X)}^2 
\\
& = \bigg\| \sum_{j \in \mathcal{I}_H}
\frac{\eins_{A_j}(x)}{\mathrm{P}_X(A_j)} \int_{A_j} f_{L, \mathrm{P}}^*(x') - f_{L, \mathrm{P}}^*(x)  \ d \mathrm{P}_X(x')
\bigg\|_{L_2(\mathrm{P}_X)}^2 
\\
& \leq \bigg\| \sum_{j \in \mathcal{I}_H}
\frac{\eins_{A_j}(x)}{\mathrm{P}_X(A_j)} \int_{A_j} \big|f_{L, \mathrm{P}}^*(x') - f_{L, \mathrm{P}}^*(x) \big| \ d \mathrm{P}_X(x')
\bigg\|_{L_2(\mathrm{P}_X)}^2 
\\
& \leq \bigg\| \sum_{j \in \mathcal{I}_H}
\frac{\eins_{A_j}(x)}{\mathrm{P}_X(A_j)} \int_{A_j} \|x'-x\|^{\alpha}  \ d \mathrm{P}_X(x')
\bigg\|_{L_2(\mathrm{P}_X)}^2 
\\
& \leq \bigg\| \sum_{j \in \mathcal{I}_H}
\frac{\eins_{A_j}(x)}{\mathrm{P}_X(A_j)} (\sqrt{d}\cdot\overline{h}_0)^{\alpha}\mathrm{P}_X(A_j)
\bigg\|_{L_2(\mathrm{P}_X)}^2 
\\
& \leq (\sqrt{d}\cdot\overline{h}_0)^{2\alpha} 
\\
& \leq d^{\alpha}c_0^{-2\alpha}\underline{h}_0^{2\alpha},
\end{align*}
where the last inequality follows from Assumption \ref{assumption::h}.
Consequently we obtain
\begin{align*}
\lambda \underline{h}_0^{-2d} 
+ \mathcal{R}_{L,\mathrm{P}}(f_{\mathrm{P},\underline{h}}) - \mathcal{R}_{L,\mathrm{P}}^* 
& \leq \lambda \underline{h}_0^{-2d} + d^{\alpha}c_0^{-2\alpha}\underline{h}_0^{2\alpha}
\\
& \leq \bigl( (\underline{h}_0^*)^{-2d} + d^{\alpha} c_0^{-2\alpha} (\underline{h}_0^*)^{2\alpha} \bigr) \lambda^{\frac{\alpha}{\alpha+d}}
\\
&:= c\lambda^{\frac{\alpha}{\alpha+d}}
\end{align*}
with $\underline{h}_0^* := (d^{1-\alpha}c_0^{2\alpha} \alpha)^{\frac{1}{2\alpha + 2d}}$, where $c = (\underline{h}_0^*)^{-2d} + d^{\alpha} c_0^{-2\alpha} (\underline{h}_0^*)^{2\alpha}$ is a constant depending on $c_0$, $d$, and $\alpha$. 
This proves the desired assertion.
\end{proof}

\subsubsection{Proofs Related to Section \ref{sec::SError1}}

To prove Lemma \ref{VCindex}, we need the following fundamental lemma concerning with the VC dimension of purely random partitions which follows the idea put forward by \cite{bremain2000some} of the construction of purely random forest. To this end, let $p \in \mathbb{N}$ be fixed and $\pi_p$ be a partition of $\mathcal{X}$ with number of splits $p$ and $\pi_{(p)}$ denote the collection of all partitions $\pi_p$.

\begin{lemma} \label{VCindexPre}
Let $\mathcal{B}_p$ be defined by
\begin{align} \label{Bp}
\mathcal{B}_p := \biggl\{ B : B = \bigcup_{j \in J} A_j, J \subset \{ 0, 1, \ldots, p \}, A_j \in \pi_p \subset \pi_{(p)} \biggr\}.
\end{align}
Then the VC dimension of $\mathcal{B}_p$ can be upper bounded by $d p + 2$.
\end{lemma}

\begin{proof}[of Lemma \ref{VCindexPre}] 
The proof will be conducted by dint of geometric constructions, and we proceed by induction.

\begin{figure*}[htbp]
\centering
\begin{minipage}[b]{0.18\textwidth}
	\centering
	\includegraphics[width=\textwidth]{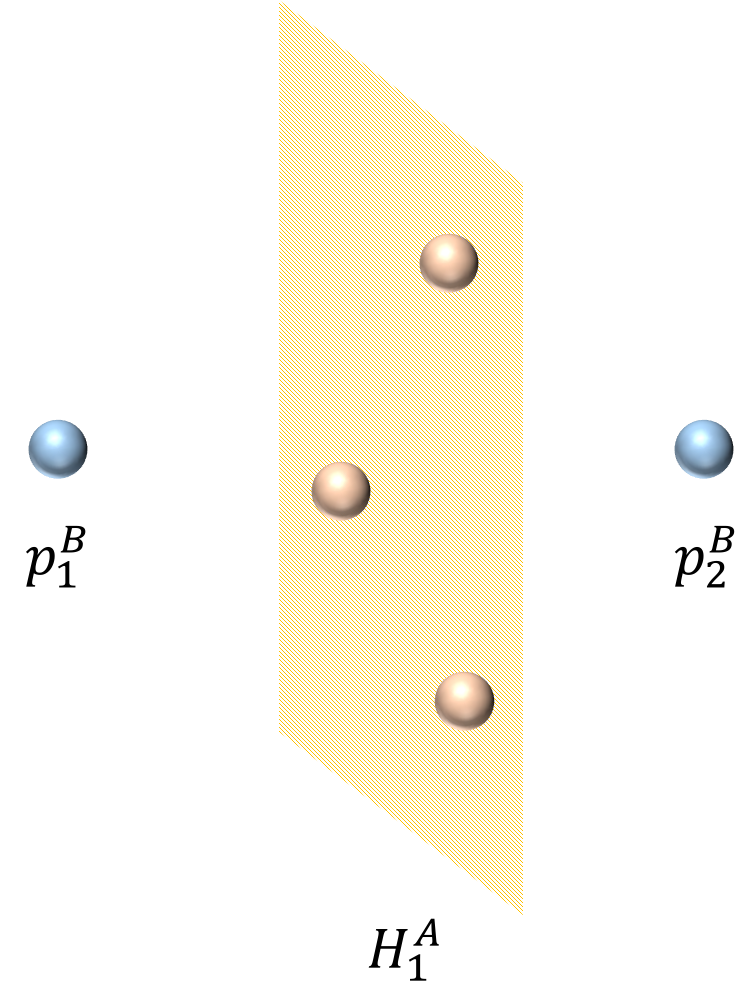}
	$p=1$
	\centering
	\label{fig::p=1}
\end{minipage}
\qquad
\begin{minipage}[b]{0.25\textwidth}
	\centering
	\includegraphics[width=\textwidth]{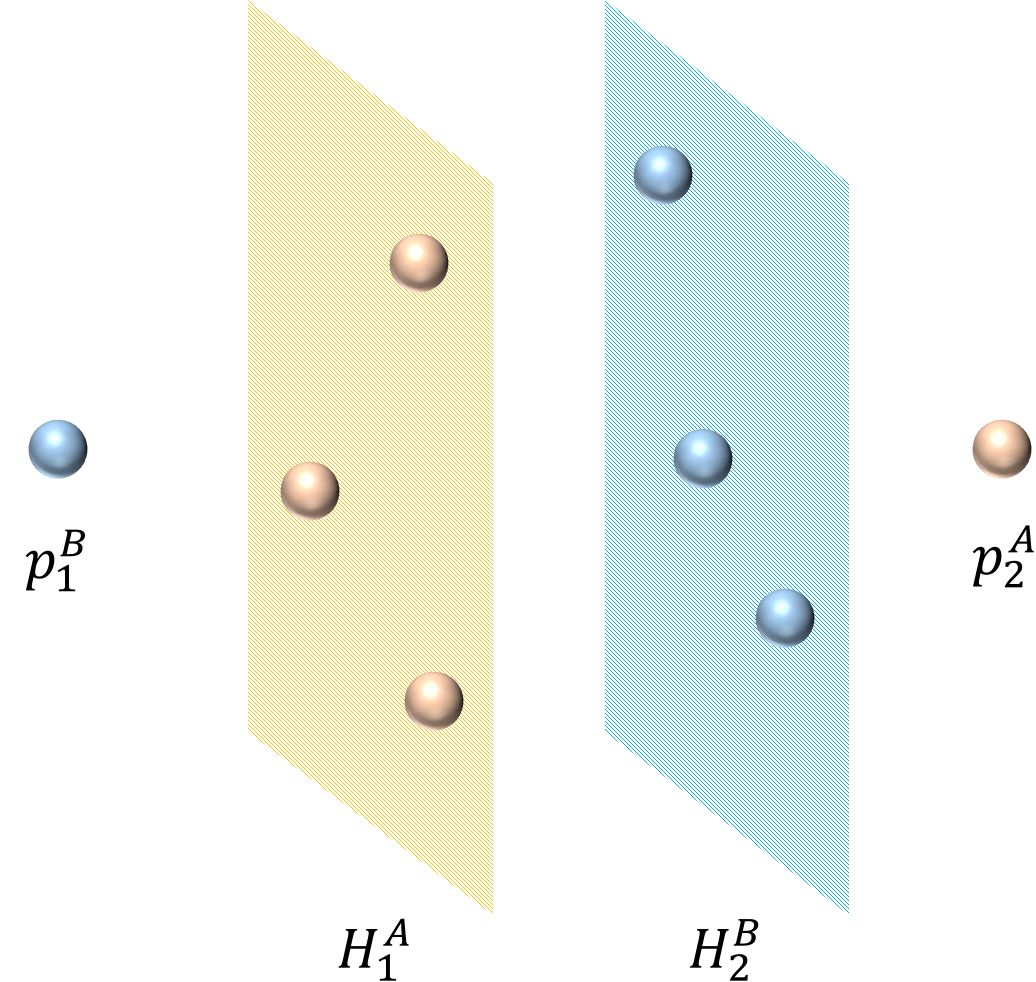}
	$p=2$
	\label{fig::p=2}
\end{minipage}
\qquad
\begin{minipage}[b]{0.43\textwidth}
	\centering
	\includegraphics[width=\textwidth]{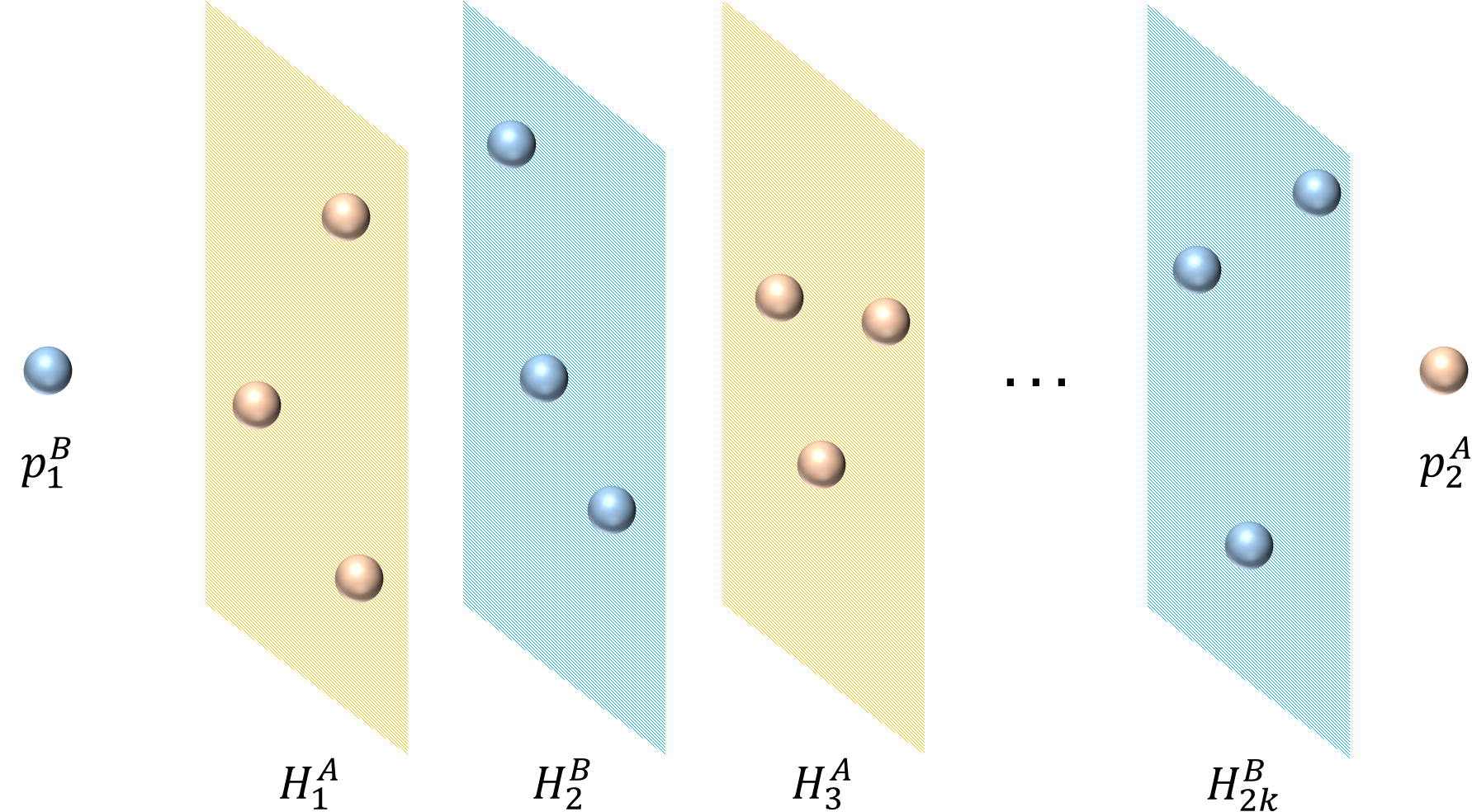}
	$p=2k$
	\label{fig::p=2k}
\end{minipage}
\caption{We take one case with $d=3$ as an example to illustrate the geometric interpretation of the VC dimension. The yellow balls represent samples from class $A$, blue ones are from class $B$ and slices denote the hyperplanes formed by samples. }
\label{fig::VC}
\end{figure*}

We begin by observing a partition with number of splits $p = 1$. On account that the dimension of the feature space is $d$, the smallest number of points that cannot be divided by $p = 1$ split is $d + 2$. Specifically, considering the fact that $d$ points can be used to form $d - 1$ independent vectors and therefore a hyperplane of a $d$-dimensional space, we now focus on the case where there is a hyperplane consisting of $d$ points all from the same class labeled as $A$, and there are two points from the other class $B$ on either side of the hyperplane. We denote the hyperplane by $H_1^A$ for brevity. In this case, points from two classes cannot be separated by one split, i.e. one hyperplane, which means that $\mathrm{VC}(\mathcal{B}(\pi_1)) \leq d + 2$.

We next turn to consider the partition with number of splits $p = 2$ which is an extension of the above case. Once we pick one point out of the two located on either side of the above hyperplane $H_1^A$, a new hyperplane $H_2^B$ parallel to $H_1^A$ can be constructed by combining the selected point with $d - 1$ newly-added points from class $B$. Subsequently, a new point from class $A$ is added to the side of the newly constructed hyperplane $H_2^B$. Notice that the newly added point should be located on the opposite side to $H_1^A$. Under this situation, $p = 2$ splits can never separate those $2 d + 2$ points from two different classes. As a result, we prove that $\mathrm{VC}(\mathcal{B}(\pi_2)) \leq 2 d + 2$.

If we apply induction to the above cases, the analysis of VC index can be extended to the general case where $p \in \mathbb{N}$. What we need to do is to add new points continuously to form $p$ mutually parallel hyperplanes with any two adjacent hyperplanes being built from different classes. Without loss of generality, we assume that $p = 2k+1$, $k \in \mathbb{N}$, and there are two points denoted by $p_1^B, p_2^B$ from class $B$ separated by $2 k + 1$ alternately appearing hyperplanes. Their locations can be represented by $p_1^B, H_1^A, H_2^B, H_3^A, H_4^B, \ldots, H_{(2k+1)}^A, p_2^B$. According to this construction, we demonstrate that the smallest number of points that cannot be divided by $p$ splits is $d p + 2$, which leads to $\mathrm{VC}(\mathcal{B}(\pi_p)) \leq d p + 2$.

It should be noted that our hyperplanes can be generated both vertically and obliquely, which is in line with our splitting criteria for the random partitions. This completes the proof. 
\end{proof}

\begin{proof}[of Lemma \ref{VCindex}]  
Again, the proof will be conducted by dint of geometric constructions.

Let us choose a data set $A \subset \mathbb{R}^d$ with $\#(A) = 2^d + 2$ and consider firstly the general case that there exists $x \in A$ such that $x \in \mathcal{C}(A \setminus \{x\})$, that is, $x$ lies in the convex hull of the set $A \setminus \{x\}$. Then there exists a set $A_1 \subset (A \setminus \{x\})$ such that
\begin{align*}
\#(A_1) = \#(A) - 2
\quad
\text{ and } 
\quad
x \in \mathcal{C}(A_1).
\end{align*}
Then for a fixed $B \in \pi_h$ with $A_1 \subset A \cap B$, there always holds
\begin{align*}
A_1 \cup \{x\} \subset A \cap B.
\end{align*}
Clearly, there exists no $B \in \pi_h$ such that $A \cap B = A_1$ and therefore $\pi_h$ cannot shatter $A$.

It remains to consider the case when $x \not\in \mathcal{C}(A \setminus \{x\})$ holds for all $x \in A$. Obviously, the convex hull of $A$ forms a hyperpolyhedron whose vertices are the points of $A$. Note that the hyperpolyhedron can be regarded as an undirected graph, therefore as usual, we define the distance $d(x_1 , x_2)$ between a pair of samples $x_1$ and $x_2$ on the graph by the shortest path between them. Clearly, there exists a starting point $x_0 \in A$ such that $\deg(x) = 2^{d - 1}$. Then we construct another data set $A_2 \neq A_1$ by
\begin{align*}
A_2 = \{ y : d(x_0, y) \mod 2 = 1, y \in A \}. 
\end{align*}
Again, for a fixed $B \in \pi_h$ such that $A_2 \subset A \cap B$, we deduce that there exists no $B \in \pi_h$ such that $A \cap B = A_2$ and therefore $\pi_h$ cannot shatter $A$ as well. By Definition  \ref{def::VCdimension}, we immediately obtain
\begin{align*} 
\mathrm{VC}(\pi_h) \leq 2^d+2.
\end{align*}

Next, we turn to prove the second assertion. The choice $k := \lfloor \frac{2R\sqrt{d}}{\underline{h}_0} \rfloor+1$ leads to the partition of $B_R$ of the form $\pi_k := \{ A_{i_1, \ldots, i_d} \}_{i_j = 1,\ldots,k}$ with
\begin{align} \label{def::cells}
A_{i_1, \ldots, i_d} 
:= \prod_{j=1}^d A_{i_j}
:= \prod_{j=1}^d \biggl[ - R + \frac{2R(i_j-1)}{k}, -R+\frac{2Ri_j}{k} \biggr).
\end{align}
Obviously, we have $|A_{i_j}| \leq \frac{\underline{h}_0}{\sqrt{d}}$. Let $D$ be a data set with 
\begin{align*}
\#(D) = (d (2^d - 1) + 2) \biggl( \biggl\lfloor \frac{2R\sqrt{d}}{\underline{h}_0} \biggr\rfloor + 1 \biggr)^d. 
\end{align*}
Then there exists at least one cell $A$ with 
\begin{align} \label{DcapANo}
\#(D \cap A) \geq d(2^d-1)+2.
\end{align}
Moreover, for any $x, x' \in A$, the construction of the partition \eqref{def::cells} implies $\|x - x'\| \leq \underline{h}_0$.
Consequently, at most one vertex of $A_j$ induced by histogram transform $H$ lies in $A$, since the bin width of $A_j$ is larger than $\underline{h}_0$. Therefore, 
\begin{align*} 
{\Pi_h}_{|A} := \{B \cap A : B \in \Pi_h\}
\end{align*}
forms a partition of $A$ with $\#({\Pi_h}_{|A}) \leq 2^d$. It is easily seen that this partition can be generated by $2^d-1$ splitting hyperplanes. In this way, Lemma \ref{VCindexPre} implies that ${\Pi_h}_{|A}$ can only shatter a dataset with at most $d(2^d-1)+1$ elements. Thus \eqref{DcapANo} indicates that ${\Pi_h}_{|A}$ fails to shatter $D \cap A$ and therefore
$\Pi_h$ cannot shatter the data set $D$ as well. By Definition \ref{def::VCdimension}, we immediately get 
\begin{align*} 
\mathrm{VC}(\Pi_h) \leq 
(d(2^d-1)+2) \biggl( \biggl\lfloor \frac{2R\sqrt{d}}{\underline{h}_0} \biggr\rfloor+1 \biggr)^d
\end{align*}
and the assertion is thus proved.
\end{proof}

\begin{proof}[of Lemma \ref{ScriptBhCoveringNumber}]
The first assertion concerning covering numbers of $\pi_h$ follows directly from Theorem 9.2 in \cite{Kosorok2008introcuction}. For the second estimate, we find the upper bound \eqref{VCMathcalBh}  of $\mathrm{VC}(\Pi_h)$ satisfies
\begin{align*} 
\bigl( d(2^d-1) + 2 \bigr) (2 R \sqrt{d} / \underline{h}_0 + 1)^d 
& \leq \bigl( (d+1)2^d \bigr) (3 R \sqrt{d}/\underline{h}_0)^d
\\
& \leq 2 d \cdot 2^d (3 R \sqrt{d}/\underline{h}_0)^d
\\
& =: (c_d R/\underline{h}_0)^d,
\end{align*}
where the constant $c_d := 3 \cdot 2^{1+\frac{1}{d}} \cdot d^{\frac{1}{d}+\frac{1}{2}}$. Again, Theorem 9.2 in \cite{Kosorok2008introcuction} yields the second assertion and thus completes the proof.
\end{proof}

\begin{proof}[of Lemma \ref{lem::CoveringNumberBH}]
Denote the covering number of $\eins_{\Pi_h}$ with respect to $L_2(\mathrm{P}_X)$ as $\mathcal{N}(\varepsilon) := \mathcal{N} (\eins_{\Pi_h}, \|\cdot\|_{L_2(\mathrm{P}_X)}, \varepsilon)$. Then, there exist $B_1, \ldots, B_{\mathcal{N}(\varepsilon)} \in \Pi_h$ such that the function set $\{\eins_{B_1}, \ldots, \eins_{B_{\mathcal{N}(\varepsilon)}}\}$ is an $\varepsilon$-net of $\eins_{\Pi_h}$ in the sense of $L_2(\mathrm{P}_X)$. That is, for any $\eins_{B} \in \eins_{\Pi_h}$, there exists a $j \in \{1, \ldots, \mathcal{N}(\varepsilon)\}$ such that $\|\eins_{B} - \eins_{B_j}\|_{L_2(\mathrm{P}_X)} \leq \varepsilon$. Now, for all $g \in \mathcal{F}_H^b$, the equivalent definition \eqref{equ::equalBH} implies that $g$ can be written as $g = \eins_B - \eins_{B^c} = 2\eins_B -1$ for some $B \in \Pi_H \in \Pi_h$. The above discussion yields that there exists a $j \in \{1, \ldots, \mathcal{N}(\varepsilon)\}$ such that for $g_j := 2\eins_{B_j}-1$, there holds
\begin{align*}
\| g - g_j\|_{L_2(\mathrm{P}_X)} 
& = \|(2 \eins_B - 1) - (2 \eins_{B_j} - 1)\|_{L_2(\mathrm{P}_X)} 
\\
& = \|2 \eins_B - 2 \eins_{B_j}\|_{L_2(\mathrm{P}_X)} 
\\
& = 2 \|\eins_B - \eins_{B_j}\|_{L_2(\mathrm{P}_X)} 
\\
& \leq 2 \varepsilon.
\end{align*}
This implies that $\{g_1, \ldots, g_{\mathcal{N}(\varepsilon)}\}$ is a $2\varepsilon$-net of $\mathcal{F}_H^b$ with respect to $\|\cdot\|_{L_2(\mathrm{P}_X)}$. Consequently, we obtain
\begin{align*}
\mathcal{N} (\mathcal{F}_H^b, \|\cdot\|_{L_2(\mathrm{P}_X)}, \varepsilon) 
& \leq \mathcal{N} (\eins_{\Pi_h}, \|\cdot\|_{L_2(\mathrm{P}_X)}, \varepsilon/2)
\\
& \leq K (c_d R/\underline{h}_0 + 1)^d 
(4e)^{(c_d R/\underline{h}_0 + 1)^d}
(2/\varepsilon)^{2(c_d R/\underline{h}_0 + 1)^d-2}.
\end{align*}
This proves the assertion.
\end{proof}

\begin{proof}[of Lemma \ref{lem::entropynumber}]
For any $h_i \in \mathcal{H}_r^b$ with $h_i = L \circ g_i - L \circ f_{L,\mathrm{P}}^*$, $i = 1,2$, there holds
\begin{align*}
\|h_1 - h_2\|_{L_2(\mathrm{D})} 
& = \biggl( \frac{1}{n} \sum_{i=1}^n ( h_1(x_i, y_i) - h_2(x_i, y_i))^2 \biggr)^{1/2} 
\\
& = 2 \biggl( \frac{1}{n} \sum_{i=1}^n (g_1(x_i) - g_2(x_i))^2 \biggr)^{1/2}
\\
& = 2 \|g_1 - g_2\|_{L_2(\mathrm{D})}.
\end{align*}
This together with Lemma \ref{lem::CoveringNumberBH} yields
\begin{align*}
\mathcal{N}(\mathcal{H}_r^b, \|\cdot\|_{L_2(\mathrm{D})}, \varepsilon) 
& \leq \mathcal{N}(\mathcal{F}_r^b, \|\cdot\|_{L_2(\mathrm{D})}, \varepsilon/2)
\\
& \leq \mathcal{N}(\mathcal{F}_H^b, \|\cdot\|_{L_2(\mathrm{D})}, \varepsilon/2)
\\
& \leq K (c_d R/\underline{h}_0+1)^d 
(4 e)^{(c_d R/\underline{h}_0 + 1)^d}
(4/\varepsilon)^{2 (c_d R/\underline{h}_0 + 1)^d - 2}.
\end{align*}
Elementary calculations show that for any $\varepsilon \in (0, 1/\max \{e, K\})$, there holds
\begin{align*}
& \log \mathcal{N}(\mathcal{H}_r, \|\cdot\|_{L_2(\mathrm{D})}, \varepsilon) 
\\
& \leq \log \Bigl( K (c_d R/\underline{h}_0 + 1)^d 
(4 e)^{(c_d R/\underline{h}_0 + 1)^d}
(4/\varepsilon)^{2 ( c_d R / \underline{h}_0 + 1)^d - 2} \Bigr)
\\
& = \log K 
+ d \log (c_d R/\underline{h}_0 + 1) 
+ (c_d R/\underline{h}_0 + 1)^d \log (4 e) 
+ 2 (c_d R/\underline{h}_0 + 1)^d \log (4/\varepsilon)
\\
& \leq 11 (2 c_d R/\underline{h}_0)^d \log (1/\varepsilon),
\end{align*}
where the last inequality is based on the following basic inequalities:
\begin{align*}
\log K 
& \leq \log (1/\varepsilon) 
\leq (c_d R/\underline{h}_0 + 1)^d \log (1/\varepsilon) 
\leq (2 c_d R/\underline{h}_0)^d \log (1/\varepsilon),
\\
d \log (c_d R/\underline{h}_0 + 1) 
& \leq (c_d R/\underline{h}_0 + 1)^d
\leq (c_d R/\underline{h}_0 + 1)^d \log (1/\varepsilon) 
\leq (2 c_d R/\underline{h}_0)^d \log (1/\varepsilon),
\\
(c_d R/\underline{h}_0 + 1)^d \log (4 e) 
& \leq (c_d R/\underline{h}_0 + 1)^d \log(e^3) 
\leq 3 (c_d R/\underline{h}_0 + 1)^d 
\leq 3 (2 c_d R/\underline{h}_0)^d \log (1/\varepsilon),
\\
2 (c_d R/\underline{h}_0 + 1)^d \log (4/\varepsilon)
& = 2 (c_d R/\underline{h}_0 + 1)^d (\log 4 + \log (1/\varepsilon))
\leq 2 (2 c_d R/\underline{h}_0)^d (\log e^2 + \log (1/\varepsilon))
\\
& = 2 (2 c_d R/\underline{h}_0)^d (2 + \log (1/\varepsilon)) 
\leq 6 (2 c_d R/\underline{h}_0)^d \log (1/\varepsilon).
\end{align*}
Consequently, for all $\delta \in (0,1)$, we have
\begin{align}\label{equ::CoverNumHr}
\sup_{\varepsilon \in (0, 1/\max\{e,K\})} \varepsilon^{2\delta} \log \mathcal{N}(\mathcal{H}_r, \|\cdot\|_{L_2(\mathrm{D})}, \varepsilon) 
\leq 11 (2 c_d R/\underline{h}_0)^d
\sup_{\varepsilon \in (0, 1)} \varepsilon^{2\delta} \log (1/\varepsilon).
\end{align}
Simple analysis shows that the right hand side of \eqref{equ::CoverNumHr} is maximized at $\varepsilon^* = e^{-1/(2\delta)}$ and we obtain
\begin{align*}
\log \mathcal{N}(\mathcal{H}_r, \|\cdot\|_{L_2(\mathrm{D})}, \varepsilon) 
\leq 11/(2e\delta) (2 c_d R/\underline{h}_0)^d \varepsilon^{-2\delta}.
\end{align*}

Next, we shall use $r$ to bound $\underline{h}_0$ in the space $\mathcal{F}_r^b$. 
For all $g \in \mathcal{F}_r^b$, there holds
\begin{align*}
\lambda \underline{h}_0^{-2d} 
\leq \lambda \underline{h}_0^{-2d} + \mathcal{R}_{L, \mathrm{P}}(g) - \mathcal{R}_{L, \mathrm{P}}^* 
\leq r
\end{align*}
and consequently we have
\begin{align*}
\underline{h}_0^{-1} \leq (r/\lambda)^{1/(2d)}. 
\end{align*}
Then Exercise 6.8 in \cite{StCh08} implies that the entropy number of $\mathcal{H}_r^b$ with respect to $L_2(\mathrm{D})$ satisfies
\begin{align*}
e_i(\mathcal{H}_r^b, \|\cdot\|_{L_2(\mathrm{D})}) 
\leq \bigl( 33/(2e \delta) (2 c_d R/\underline{h}_0)^d \bigr)^{\frac{1}{2\delta}} i^{-\frac{1}{2\delta}} 
\leq \bigl( 33/(2e \delta) (2 c_d R (r/\lambda)^{\frac{1}{2d}})^d \bigr)^{\frac{1}{2\delta}} i^{-\frac{1}{2\delta}}.
\end{align*}
Taking expectation on both sides of the above inequality, we get
\begin{align*}
\mathbb{E}_{\mathrm{D}\sim \mathrm{P}} e_i (\mathcal{H}_r^b, \|\cdot\|_{L_2(\mathrm{D})})
\leq \bigl( 33/(2e \delta) (2 c_d R (r/\lambda)^{\frac{1}{2d}})^d \bigr)^{\frac{1}{2\delta}} i^{-\frac{1}{2\delta}}.
\end{align*}
Thus, we finished the proof.
\end{proof}

\begin{proof}[of Lemma \ref{lem::EERad}]
First of all, we notice that for all $h \in \mathcal{H}_r^b$, there holds
\begin{align*}
\|h\|_{\infty} \leq 4 =: B_1,
\qquad
\mathbb{E}_{\mathrm{P}} h^2 \leq 16 r =: \sigma^2.
\end{align*}
Then $a := ( \frac{33}{2e\delta} ( 2 c_d R (\frac{r}{\lambda})^{1/2})^d )^{\frac{1}{2\delta}} \geq B_1$ in Lemma \ref{lem::entropynumber} together with Theorem 7.16 in \cite{StCh08} yields that there exist constants $c_1(\delta) > 0$ and $c_2(\delta) > 0$ depending only on $\delta$ such that
\begin{align*}
\mathbb{E}_{D \sim \mathrm{P}^n} \mathrm{Rad}_{\mathrm{D}} (\mathcal{H}_r^b, n) 
& \leq \max \Bigl\{ 
c_1(\delta) \bigl( 33/(2e\delta) (2 c_d R (r/\lambda)^{\frac{1}{2d}})^d \bigr)^{\frac{1}{2}} 
(16 r)^{\frac{1-\delta}{2}} n^{-\frac{1}{2}}, 
\\
& \phantom{=} \qquad \qquad
c_2(\delta) \bigl( 33/(2e\delta) (2 c_d R (r/\lambda)^{\frac{1}{2d}})^d \bigr)^{\frac{1}{1+\delta}} 4^{\frac{1-\delta}{1+\delta}} n^{-\frac{1}{1+\delta}} \Bigr\}
\\
&= \max \Bigl\{ c'_1(\delta)
\lambda^{-\frac{1}{4}} r^{\frac{3-2\delta}{4}} n^{-\frac{1}{2}}, 
c'_2(\delta)  \lambda^{-\frac{1}{2(1+\delta)}} r^{\frac{1}{2(1+\delta)}} n^{-\frac{1}{1+\delta}} \Bigr\},
\end{align*}
where the constants are
\begin{align*}
c'_1(\delta) & := c_1(\delta) (33/(2e\delta)^{\frac{1}{2}} 16^{\frac{1-\delta}{2}}
(2 c_d R)^{\frac{d}{2}},
\\
c'_2(\delta) & := c_2(\delta) (33/(2e\delta)^{\frac{1}{1+\delta}} 4^{\frac{1-\delta}{1+\delta}} 
(2 c_d R)^{\frac{d}{1+\delta}}.
\end{align*}
Consequently we obtain
\begin{align*}
\mathbb{E}_{D \sim \mathrm{P}^n} \mathrm{Rad}_{\mathrm{D}}(\mathcal{H}_r,n) 
& \leq M \mathbb{E}_{D \sim \mathrm{P}^n} \mathrm{Rad}_{\mathrm{D}} (\mathcal{H}_r^b, n) 
\\
& \leq \max \bigl\{ c''_1(\delta) \lambda^{-\frac{1}{4}} r^{\frac{3-2\delta}{4}} n^{-\frac{1}{2}},    
c''_2(\delta) \lambda^{-\frac{1}{2(1+\delta)}} r^{\frac{1}{2(1+\delta)}} n^{-\frac{1}{1+\delta}} \bigr\},
\end{align*}
where $c''_1(\delta) := M c'_1(\delta)$ and 
$c''_2(\delta) := M c'_2(\delta)$.
This proves the assertion.
\end{proof}

\subsubsection{Proofs Related to Section \ref{sec::OracleTreeC0}}

\begin{proof}[of Theorem \ref{thm::OracleTree}]
For the least square loss $L$, the supremum bound
\begin{align*}
L(x,y,t) \leq 4M^2 =: B, \quad \forall (x,y) \in \mathcal{X} \times \mathcal{Y}, t\in [-M,M]
\end{align*}
and the variance bound
\begin{align*}
\mathbb{E}(L \circ g - L \circ f_{L, \mathrm{P}}^*)^2 
\leq V(\mathbb{E}(L \circ g - L \circ f_{L, \mathrm{P}}^*))^{\vartheta}
\end{align*}
holds for $V = 16 M^2$ and $\vartheta = 1$. 
Moreover, Lemma \ref{lem::EERad} implies that the expected empirical Rademacher average of $\mathcal{H}_r$ can be bounded by the function $\varphi_n(r)$ as
\begin{align*}
\varphi_n(r) := 
\max \bigl\{ c''_1(\delta) \lambda^{-\frac{1}{4}} r^{\frac{3-2\delta}{4}} n^{-\frac{1}{2}}, 
c''_2(\delta) \lambda^{-\frac{1}{2(1+\delta)}} r^{\frac{1}{2(1+\delta)}} n^{-\frac{1}{1+\delta}} \bigr\},
\end{align*}
where $c''_1(\delta)$ and $c''_2(\delta)$ are 
some constants depending on $\delta$.
Simple algebra shows that the condition $\varphi_n(4r) \leq 2\sqrt{2} \varphi_n(r)$ is satisfied. 
Since $2\sqrt{2} < 4$, similar arguments show that the statements of the Peeling Theorem 7.7 in \cite{StCh08} still hold. Therefore, Theorem 7.20 in \cite{StCh08} can also be applied, if 
the assumptions on $\varphi_n$ and $r$ are modified to $\varphi_n(4r) \leq 2\sqrt{2} \varphi_n(r)$ and
$r \geq \max\{75\varphi_n(r), 1152M^2\tau/n, r^*\}$,
respectively. Some elementary calculations show that the condition $r > 75\varphi_n(r)$ is satisfied if 
\begin{align*}
r & \geq \max \Bigl\{ 
(75 c''_1(\delta) \lambda^{-\frac{1}{4}} n^{-\frac{1}{2}})^{\frac{4}{1+2\delta}}, 
(75 c''_2(\delta) \lambda^{-\frac{1}{2(1+\delta)}} n^{-\frac{1}{1+\delta}})^{\frac{2(1+\delta)}{1+2\delta}} 
\Bigr\}
\\
& = \max \Bigl\{
(75 c''_1(\delta))^{\frac{4}{1+2\delta}},
(75 c''_2(\delta))^{\frac{2(1+\delta)}{1+2\delta}} \Bigr\} 
\cdot
\lambda^{-\frac{1}{1+2\delta}} n^{-\frac{2}{1+2\delta}},
\end{align*}
which yields the assertion.
\end{proof}

\subsubsection{Proofs Related to Section \ref{sec::C0}}

\begin{proof}[of Theorem \ref{thm::tree}]
Theorem \ref{thm::OracleTree} and Proposition \ref{prop::ApproximatiON-ERROR} imply that with probability $\nu_n$ at least $1 - 3 e^{-\tau}$, there holds 
\begin{align}\label{equ::Excessbound}
\lambda \underline{h}_{0,n}^{-2d} 
+ \mathcal{R}_{L,\mathrm{P}}(f_{\mathrm{D},H_n}) 
- \mathcal{R}_{L,\mathrm{P}}^* 
\leq 9 c \lambda ^{\frac{\alpha}{\alpha+d}} 
+ 3 c_{\delta} \lambda^{-\frac{1}{1+2\delta}} n^{-\frac{2}{1+2\delta}}
+ 3456 M^2 \tau / n,
\end{align}
where $c$ and $c_{\delta}$ are the constants defined as in Proposition \ref{prop::ApproximatiON-ERROR} and Theorem \ref{thm::OracleTree}, respectively. 
Minimizing the right hand side of \eqref{equ::Excessbound} with respect to $\lambda$, 
by choosing 
\begin{align*}
\lambda := n^{-\frac{2(\alpha+d)}{d+2\alpha(1+\delta)}}, 
\end{align*}
we get
\begin{align*}
\lambda\underline{h}_{0,n}^{-2d} + \mathcal{R}_{L,\mathrm{P}}(f_{\mathrm{D},H_n}) -\mathcal{R}_{L,\mathrm{P}}^* \leq c n^{-\frac{2\alpha}{d+2\alpha (1+\delta)}},
\end{align*}
where $c$ is some constant depending on $c_0$, $\delta$, $d$, $M$, and $R$.
Moreover, there holds
\begin{align*}
n^{-\frac{2\alpha}{d+2\alpha (1+\delta)}}=n^{-\frac{2\alpha}{d+2\alpha}\cdot \frac{d+2\alpha}{d+2\alpha(1+\delta)}}=n^{-\frac{2\alpha}{d+2\alpha}\cdot (1-\frac{2\alpha\delta}{d+2\alpha(1+\delta)})}=n^{-\frac{2\alpha}{d+2\alpha}+\xi}
\end{align*}
where $\xi := \frac{4\alpha^2\delta}{(d+2\alpha)(d+2\alpha(1+\delta))} > 0$ can be arbitrarily small.
Thus, the assertion is proved.
\end{proof}

\begin{proof}[of Theorem \ref{thm::forest}]
According to Jensen's inequality, there holds
\begin{align*}
\biggl( \sum_{t=1}^T f_{\mathrm{D},H_t} - f_{L,\mathrm{P}}^* \biggr)^2
\leq T \sum_{t=1}^T (f_{\mathrm{D},H_t} - f_{L,\mathrm{P}}^*)^2
\end{align*}
and consequently we have
\begin{align*}
\mathcal{R}_{L, \mathrm{P}}(f_{\mathrm{D},T}) - \mathcal{R}_{L,\mathrm{P}}^* 
& = \int_{\mathcal{X}} \biggl( \frac{1}{T} \sum_{t=1}^T f_{\mathrm{D},H_t} - f_{L, \mathrm{P}}^* \biggr)^2 \, d\mathrm{P}_X
\\
& \leq \frac{1}{T} \sum_{t=1}^T \int_{\mathcal{X}} (f_{\mathrm{D},H_t} - f_{L, \mathrm{P}}^*)^2 \, d\mathrm{P}_X
\\
& = \frac{1}{T} \sum_{t=1}^T (\mathcal{R}_{L, \mathrm{P}}(f_{\mathrm{D},H_t}) - \mathcal{R}_{L,\mathrm{P}}^*).
\end{align*}
Then the union bound together with Theorem \ref{thm::tree} implies
\begin{align*}
& \nu_n 
\Bigl( \mathcal{R}_{L, \mathrm{P}}(f_{\mathrm{D},T}) - \mathcal{R}_{L,\mathrm{P}}^*  
\leq cn^{-\frac{2\alpha}{2 \alpha + d}+\xi} \Bigr) 
\\
& \geq 1 - \sum_{t=1}^T \mathrm{P} \otimes \mathrm{P}_H 
\Bigl( \mathcal{R}_{L, \mathrm{P}}(f_{\mathrm{D},H_t}) - \mathcal{R}_{L,\mathrm{P}}^*  
>c n^{-\frac{2\alpha}{2\alpha+d}+\xi} \Bigr)
\\
& \geq 1 - 3 T e^{-\tau}.
\end{align*}
As a result, we obtain
\begin{align*}
\mathcal{R}_{L, \mathrm{P}}(f_{\mathrm{D},T}) - \mathcal{R}_{L,\mathrm{P}}^* 
\leq c n^{-\frac{2\alpha}{2\alpha+d}+\xi}
\end{align*}
with probability $\nu_n$ at least $1 - 3 e^{-\tau}$, where $c$ is some constant depending on $c_0$, $\delta$, $d$, $M$, $R$, and $T$.
\end{proof}

\subsection{Proofs of Results  for NHT  in the space $C^{1,\alpha}$}

The following Lemma presents the explicit representation of $A_H(x)$ which will  play a key role later in the proofs of subsequent sections.

\begin{lemma}\label{binset}
Let the histogram transform $H$ be defined as in \eqref{HistogramTransform} 
and $A'_H$, $A_H$ be as in \eqref{TransBin} and \eqref{equ::InputBin} respectively. 
Then for any $x \in \mathbb{R}^d$, the set $A_H(x)$ can be represented as 
\begin{align*}
A_H(x) = \bigl\{ x + (R \cdot S)^{-1} z :  z \in [-b', 1 - b'] \bigr\},
\end{align*}
where $b' \sim \mathrm{Unif}(0, 1)^d$.
\end{lemma}

\begin{proof}[of lemma \ref{binset}]
For any $x \in \mathbb{R}^d$, we define $b' := H(x) - \lfloor H(x) \rfloor \in \mathbb{R}^d$.
Then we have $b' \sim \mathrm{Unif}(0,1)^d$ according to the definition of $H$. 
For any $x' \in A'_H(x)$, we define
\begin{align*}
z := H(x') - H(x) = (R \cdot S) (x' - x).
\end{align*}
Then we have
\begin{align*}
x' = x + (R \cdot S)^{-1} z.
\end{align*}
Moreover, 
since $\lfloor H(x') \rfloor = \lfloor H(x) \rfloor$, we have $z \in [-b', 1 - b']$. 
\end{proof}

\subsubsection{Proofs Related to Section \ref{sec::AError3}}

\begin{proof}[of Proposition \ref{prop::biasterm}]
According to the generation process,
the histogram transforms $\{H_t\}_{t=1}^T$ are independent and identically distributed.
Therefore, for any $x \in B_R$, the expected approximation error term can be decomposed as follows:
\begin{align}
&\mathbb{E}_{\mathrm{P}_H}  \bigl(f_{\mathrm{P},\mathrm{E}}^*(x) - f_{L, \mathrm{P}}^*(x) \bigr)^2 
\nonumber
\\
&=\mathbb{E}_{\mathrm{P}_H}\bigl((f_{\mathrm{P},\mathrm{E}}^*(x)-\mathbb{E}_{\mathrm{P}_H}(f_{\mathrm{P},\mathrm{E}}^*(x)))+(\mathbb{E}_{\mathrm{P}_H}(f_{\mathrm{P},\mathrm{E}}^*(x))-f_{L,\mathrm{P}}^*(x)) \bigr)^2
\nonumber
\\
&=\mathrm{Var}(f_{\mathrm{P},\mathrm{E}}^*(x))+(\mathbb{E}_{\mathrm{P}_H}(f_{\mathrm{P},\mathrm{E}}^*(x))-f_{L,\mathrm{P}}^*(x))^2
\nonumber
\\
&= \frac{1}{T} \cdot \mathrm{Var}_{\mathrm{P}_H}(f_{\mathrm{P}, H_1}^*(x))+\bigl( \mathbb{E}_{\mathrm{P}_H} ( f_{\mathrm{P},H_1}^*(x) ) - f_{L, \mathrm{P}}^*(x) \bigr)^2.
\label{equ::biasvarianceDecom}
\end{align}
In the following, 
for the simplicity of notations, we drop the subscript of $H_1$ and write $H$ instead of $H_1$ when there is no confusion.

For the first term in \eqref{equ::biasvarianceDecom}, the assumption $f_{L,\mathrm{P}}^* \in C^{1,\alpha}$ implies
\begin{align}\label{equ::first}
\mathrm{Var}_{\mathrm{P}_H} \bigl( f_{\mathrm{P},H}^*(x) \bigr)
&=\mathbb{E}_{\mathrm{P}_H}(f_{\mathrm{P},H}^*(x)-\mathbb{E}_{\mathrm{P}_H}(f_{\mathrm{P},H}^*(x)))^2
\nonumber
\\
&\leq \mathbb{E}_{\mathrm{P}_H} \bigl( f_{\mathrm{P},H}^*(x) - f_{L, \mathrm{P}}^*(x) \bigr)^2
\nonumber
\\
&=\mathbb{E}_{\mathrm{P}_H} \biggl( \frac{\int_{A_H(x)} f_{L, \mathrm{P}}^*(x') \, dx'}{\mu(A_H(x))} - f_{L, \mathrm{P}}^*(x) \biggr)^2
\nonumber
\\
&=\mathbb{E}_{\mathrm{P}_H}\biggl(\frac{\int_{A_H(x)} f_{L, \mathrm{P}}^*(x') - f_{L, \mathrm{P}}^*(x) \, dx'}{\mu(A_H(x))} \biggr)^2
\nonumber
\\
&\leq \mathbb{E}_{\mathrm{P}_H} \bigl( c_L \mathrm{diam} \bigl( A_H(x) \bigr) \bigr)^2
\nonumber
\\
&\leq c_L^2 d \overline{h}_0^2.
\end{align}

We now consider the second term in \eqref{equ::biasvarianceDecom}. Lemma \ref{binset} implies that for any $x' \in A_H(x)$, there exist a random vector $u \sim \mathrm{Unif}[0,1]^d$ and a vector $v \in [0,1]^d$ such that 
\begin{align} \label{xPrimex}
x' = x + S^{-1} R^{\top} (- u + v).
\end{align}
Therefore, we have
\begin{align}
dx' = \det \biggl( \frac{dx'}{dv} \biggr) dv
& = \det \biggl( \frac{d(x + S^{-1} R^{\top}(- u + v))}{dv} \biggr) dv
\nonumber\\
& = \det (R S^{-1}) dv
= \biggl( \prod_{i=1}^d h_i \biggr) dv.
\label{JacobiTrans}
\end{align}
Taking the first-order Taylor expansion of $f_{L,\mathrm{P}}^*(x')$ at $x$, we get
\begin{align} \label{TaylorExpansion}
f_{L,\mathrm{P}}^*(x') - f_{L,\mathrm{P}}^*(x) = \int_0^1 \bigl( \nabla f_{L,\mathrm{P}}^*(x + t(x' - x)) \bigr)^{\top} (x' - x) \, dt. 
\end{align}
Moreover, we obviously have
\begin{align} \label{Trivial}
\nabla f_{L,\mathrm{P}}^*(x)^{\top} (x' - x)
= \int_0^1 \nabla f_{L,\mathrm{P}}^*(x)^{\top} (x' - x) \, dt. 
\end{align}
Thus, \eqref{TaylorExpansion} and \eqref{Trivial} imply that for any $f_{L,\mathrm{P}}^* \in C^{1, \alpha}$, there holds
\begin{align*}
&\bigl| f_{L,\mathrm{P}}^*(x') - f_{L,\mathrm{P}}^*(x) - \nabla f_{L,\mathrm{P}}^*(x)^{\top} (x' - x) \bigr|
\\
&= \biggl| \int_0^1 \bigl( \nabla f_{L,\mathrm{P}}^*(x + t(x' - x)) - \nabla f_{L,\mathrm{P}}^*(x) \bigr)^{\top} (x' - x) \, dt \biggr|
\\
& \leq \int^1_0 c_L (t \|x' - x\|_2)^{\alpha} \|x' - x\|_2 \, dt
\\
& \leq c_L \|x' - x\|^{1+\alpha}. 
\end{align*}
This together with \eqref{xPrimex} yields
\begin{align*}
\bigl| f_{L,\mathrm{P}}^*(x') - f_{L,\mathrm{P}}^*(x) - \nabla f_{L,\mathrm{P}}^*(x)^{\top} S^{-1} R^{\top} (- u + v) \bigr|
\leq c_L \overline{h}_0^{1+\alpha}
\end{align*}
and consequently there exists a constant $c_{\alpha} \in [-c_L, c_L]$ such that
\begin{align} \label{TaylorEntwicklung}
f_{L,\mathrm{P}}^*(x') - f_{L,\mathrm{P}}^*(x) = \nabla f_{L,\mathrm{P}}^*(x)^{\top} S^{-1}R^{\top} (- u + v) + c_{\alpha} \overline{h}_0^{1+\alpha}.
\end{align}
Therefore, there holds
\begin{align*}
f_{\mathrm{P},H}^*(x) = \frac{1}{\mathrm{P}_X(A_H(x))}\int_{A_H(x)} f_{L,\mathrm{P}}^*(x')\ dx' = \frac{1}{\mu(A_H(x))}\int_{A_H(x)} f_{L,\mathrm{P}}^*(x')\ dx'.
\end{align*}
This together with \eqref{TaylorEntwicklung} and \eqref{JacobiTrans} yields
\begin{align}
f_{\mathrm{P},H}^*(x) - f_{L,\mathrm{P}}^*(x)
& = \frac{1}{\mu(A_H(x))} \int_{A_H(x)} f_{L,\mathrm{P}}^*(x') \, dx' - f_{L,\mathrm{P}}^*(x)
\nonumber\\
& = \frac{1}{\mu(A_H(x))} \int_{A_H(x)} \bigl( f_{L,\mathrm{P}}^*(x') - f_{L,\mathrm{P}}^*(x) \bigr) \, dx'
\nonumber\\
& = \frac{\prod_{i=1}^d h_i}{\mu(A_H(x))}  
\int_{[0,1]^d} \Bigl( \nabla f_{L,\mathrm{P}}^*(x)^{\top} S^{-1} R^{\top} (- u + v) + c_{\alpha} \overline{h}_0^{1+\alpha} \Bigr) \, dv
\nonumber\\
& = \bigg( \int_{[0,1]^d} (- u + v)^{\top} \, dv \biggr) R S^{-1} \nabla f_{L,\mathrm{P}}^*(x) + c_{\alpha} \overline{h}_0^{1+\alpha}
\nonumber\\
& = \biggl( \frac{1}{2} - u \biggr)^{\top} R S^{-1} \nabla f_{L,\mathrm{P}}^*(x) + c_{\alpha} \overline{h}_0^{1+\alpha}.
\label{StepOne}
\end{align}
Since the random variables $(u_i)_{i=1}^d$ are independent and identically distributed as $\mathrm{Unif}[0, 1]$, we have
\begin{align} \label{CrossTermPropertyy1}
\mathbb{E}_{\mathrm{P}_H} \bigg(\frac{1}{2}-u_i \bigg) = 0,
\qquad \qquad
i = 1, \ldots, d.
\end{align}
Combining \eqref{StepOne} with \eqref{CrossTermPropertyy1}, we obtain
\begin{align}
\mathbb{E}_{\mathrm{P}_H} ( f_{\mathrm{P},H}^*(x) - f_{L,\mathrm{P}}^*(x) )
= 0 + c_{\alpha} \overline{h}_0^{1+\alpha}
= c_{\alpha} \overline{h}_0^{1+\alpha}
\end{align}
and consequently
\begin{align}\label{equ::biasbound}
\bigl( \mathbb{E}_{\mathrm{P}_H} ( f_{\mathrm{P},H_1}^*(x) ) - f_{L,\mathrm{P}}^*(x) \bigr)^2
\leq c_L^2 \overline{h}_0^{2(1+\alpha)}.
\end{align}
Combining \eqref{equ::biasvarianceDecom} with \eqref{equ::biasbound} and \eqref{equ::first}, we obtain
\begin{align*}
\mathbb{E}_{\mathrm{P}_H} \bigl( f_{\mathrm{P},\mathrm{E}}^*(x) - f_{L,\mathrm{P}}^*(x) \bigr)^2
\leq c_L^2 \overline{h}_0^{2(1+\alpha)} + \frac{1}{T} \cdot d c_L^2 \overline{h}_0^2,
\end{align*}
which completes the proof.
\end{proof}

\subsubsection{Proofs Related to Section \ref{sec::SError3}}

\begin{proof}[of Lemma \ref{lem::VCreg}]
The choice $k := \lfloor \frac{2R\sqrt{d}}{\underline{h}_0} \rfloor+1$ leads to the partition of $B_R$ of the form $\pi_k := \{ A_{i_1, \ldots, i_d} \}_{i_j = 1,\ldots,k}$ with
\begin{align} \label{def::cells2}
A_{i_1, \ldots, i_d} 
:= \prod_{j=1}^d A_{i_j}
:= \prod_{j=1}^d \biggl[ - R + \frac{2R(i_j-1)}{k}, -R+\frac{2Ri_j}{k} \biggr).
\end{align}
Obviously, we have $|A_{i_j}| \leq \frac{\underline{h}_0}{\sqrt{d}}$. Let $D$ be a data set of the form 
\begin{align*}
D := \{ (x_i, t_i) : x_i \in B_R, t_i \in [-M, M], i = 1, \cdots, \#(D) \}
\end{align*} 
and
\begin{align*}
\#(D) = \big(2(d+1)(2^{d} - 1) + 2) \biggl( \biggl\lfloor \frac{2R\sqrt{d}}{\underline{h}_0} \biggr\rfloor + 1 \biggr)^d. 
\end{align*}
Then there exists at least one cell $A$ with 
\begin{align} \label{DcapANo2}
\#(D \cap (A\times [-M,M])) \geq 2(d+1)(2^{d}-1)+2.
\end{align}
Moreover, for any $x, x' \in A$, the construction of the partition \eqref{def::cells2} implies $\|x - x'\| \leq \underline{h}_0$.
Consequently, at most one vertex of $A_j$ induced by histogram transform $H$ lies in $A$, since the bin width of $A_j$ is larger than $\underline{h}_0$. The VC dimension of $\mathcal{F}_H$ represents the largest number of points can be shattered by 
\begin{align*}
\big\{\{(x, t): t \leq f(x)\}, f\in \mathcal{F}_H\big\},
\end{align*}
which is the subset of the collection 
\begin{align*}
\Pi'_h := \biggl\{
\bigcup_{j \in \mathcal{I}_H} \{(x,t): x\in A_j, a_j(c_j - t) \leq 0 \} : (a_j)_{j\in\mathcal{I}_H}  \in \{-1,1\}^{\mathcal{I}_H}, \pi_H \in \Pi_h \biggr\}.
\end{align*}
Obviously, the restriction of $\Pi'_h$ on the set $A \times [-M, M]$, that is,
\begin{align*}
{\Pi'_h}_{|A\times [-M,M]} := \{B \cap (A\times [-M,M]) : B \in \Pi'_h\}
\end{align*}
forms a partition of $A \times [-M,M]$ with cardinality 
$\#( {\Pi'_h}_{|A\times [-M,M]}) \leq 2^{d+1}$,
which can be generated by $2(2^d-1)$ splitting hyperplanes. In this way, Lemma \ref{VCindexPre} implies that ${\Pi_h}_{|A\times [-M,M]}$ can only shatter a dataset with at most $2(d+1)(2^d-1)+1$ elements.
However, \eqref{DcapANo} indicates that
$D \cap (A\times [-M,M])$ has at least $2(d+1)(2^d-1)+2$ elements and consequently ${\Pi'_h}_{|A\times [-M,M]}$ fails to shatter $D \cap (A\times [-M,M])$. Therefore,
the data set $D$ cannot be shattered by $\Pi'_h$. By Definition \ref{def::VCdimension}, we then have 
\begin{align*} 
\mathrm{VC}(\Pi'_h) \leq 
(2(d+1)(2^d-1)+2) \biggl( \biggl\lfloor \frac{2R\sqrt{d}}{\underline{h}_0} \biggr\rfloor+1 \biggr)^d
\end{align*}
and thus the first assertion is proved.

For the second assertion, we find 
\begin{align*} 
(2(d+1)(2^d-1)+2) \biggl( \biggl\lfloor \frac{2R\sqrt{d}}{\underline{h}_0} \biggr\rfloor+1 \biggr)^d
& \leq \bigl( 2(d+1)(2^d-1) + 2 \bigr) (2 R \sqrt{d} / \underline{h}_0 + 1)^d 
\\
& \leq \bigl( (d+1)2^{d+1} \bigr) (3 R \sqrt{d}/\underline{h}_0)^d
\\
& \leq 2 d \cdot 2^{d+1} (3 R \sqrt{d}/\underline{h}_0)^d
\\
& =: 2(c_d R/\underline{h}_0)^d,
\end{align*}
where the constant $c_d := 3 \cdot 2^{1+\frac{1}{d}} \cdot d^{\frac{1}{d}+\frac{1}{2}}$. 
Then Theorem 2.6.7 in \cite{vandervart1996weak} yields 
\begin{align*}
\mathcal{N}(\mathcal{F}_H, L_2(\mathrm{Q}), M \varepsilon) 
\leq 2K(c_d R/\overline{h}_0)^d(16e)^{2(c_d R/\overline{h}_0)^d}(1/\varepsilon)^{4(c_d R/\overline{h}_0)^d},
\end{align*}
which proves the second assertion and thus completes the proof.
\end{proof}

The following lemma follows directly from Theorem 2.6.9 in \cite{vandervart1996weak}. 
For the sake of completeness, we present the proof.

\begin{lemma}\label{thm::vart}
Let $Q$ be a probability measure on $X$ and 
\begin{align*}
\mathcal{F} := \bigl\{ f : X \to \mathbb{R} : f\in [-M, M] \text{ and } \|f\|_{L_2(\mathrm{Q})} < \infty \bigr\}.
\end{align*}
Assume that for some fixed $\varepsilon > 0$ and $v > 0$, the covering number of $\mathcal{F}$ satisfies
\begin{align} \label{CoverAssumption}
\mathcal{N}(\mathcal{F}, L_2(\mathrm{Q}), M \varepsilon)
\leq c (1/\varepsilon)^v.
\end{align}
Then there exists a universal constant $c$ such that
\begin{align*}
\log \mathcal{N}(\mathrm{Co}(\mathcal{F}), L_2(\mathrm{Q}), M \varepsilon)
\leq c' \, c^{-2/(v+2)} \varepsilon^{-2v/(v+2)}.
\end{align*}
\end{lemma}

\begin{proof}[of Lemma \ref{thm::vart}]
Let $\mathcal{F}_{\varepsilon}$ be an $\varepsilon$-net over $\mathcal{F}$. 
Then, for any $f \in \mathrm{Co}(\mathcal{F})$, there exists an $f_{\varepsilon} \in \mathrm{Co}(\mathcal{F}_{\varepsilon})$
such that $\|f - f_{\varepsilon}\|_{L_2(\mathrm{Q})} \leq \varepsilon$.
Therefore, we can assume without loss of generality that $\mathcal{F}$ is finite.

Obviously, \eqref{CoverAssumption} holds for $1 \leq \varepsilon \leq c^{1/v}$.
Let $v' := 1/2 + 1/v$ and $M' := c^{1/v}M$. Then \eqref{CoverAssumption} implies that
for any $n \in \mathbb{N}$, there exists $f_1, \ldots, f_n \in \mathcal{F}$ such that
for any $f \in \mathcal{F}$, there exists an $f_i$ such that
\begin{align*}
\|f - f_i\|_{L_2(\mathrm{Q})} \leq M' n^{-1/v}.
\end{align*}
Therefore, for each $n \in \mathbb{N}$, we can find
sets $\mathcal{F}_1 \subset \mathcal{F}_2 \subset \cdots \subset \mathcal{F}$ such that 
the set $\mathcal{F}_n$ is a $M' n^{-1/v}$-net over $\mathcal{F}$ and $\#(\mathcal{F}_n) \leq n$.

In the following, we show by induction that for $q \geq 3 + v$, there holds
\begin{align}\label{eq::main}
\log \mathcal{N} \bigl( \mathrm{Co}(\mathcal{F}_{nk^q}), L_2(\mathrm{Q}), c_k M' n^{-v'} \bigr)
\leq c'_k n,
\qquad 
n, k \geq 1,
\end{align}
where $c_k$ and $c'_k$ are constants depending only on $c$ and $v$ such that $\sup_k \max \{ c_k, c'_k \} < \infty$.
The proof of \eqref{eq::main} will be conducted by a nested induction argument.

Let us first consider the case $k = 1$. For a fixed $n_0$, let $n \leq n_0$. Then for $c_1$ satisfying $c_1 M' n_0^{-v'} \geq M$, there holds 
\begin{align*}
\log \mathcal{N} \bigl( \mathrm{Co}(\mathcal{F}_{nk^q}), L_2(\mathrm{Q}), c_k M' n^{-v'} \bigr) = 0,
\end{align*}
which immediately implies \eqref{eq::main}. For a general $n \in \mathbb{N}$, let $m := n/\ell$ for large enough $\ell$ to be chosen later. Then for any $f \in \mathcal{F}_n \setminus \mathcal{F}_m$, there exists an $f^{(m)} \in \mathcal{F}_m$ such that
\begin{align*}
\|f - f^{(m)}\|_{L_2(\mathrm{Q})} \leq M' m^{-1/v}.
\end{align*}
Let $\pi_m : \mathcal{F}_n \setminus \mathcal{F}_m \to \mathcal{F}_m$ be the projection operator. Then for any $f \in \mathcal{F}_n \setminus \mathcal{F}_m$, there holds
\begin{align*}
\|f - \pi_m f\|_{L_2(\mathrm{Q})} \leq M' m^{-1/v}
\end{align*}
and consequently for $\lambda_i, \mu_j \geq 0$ and $\sum_{i=1}^n \lambda_i = \sum_{j=1}^m \mu_j = 1$, we have
\begin{align*}
\sum_{i=1}^n \lambda_i f^{(n)}_i 
= \sum_{j=1}^m \mu_j f^{(m)}_j 
+ \sum_{k=m+1}^n \lambda_k \bigl( f^{(n)}_k - \pi_m f^{(n)}_k \bigr).
\end{align*}
Let $\mathcal{G}_n$ be the set 
\begin{align*}
\mathcal{G}_n := \{ 0 \} \cup \{ f - \pi_m f : f \in \mathcal{F}_n \setminus \mathcal{F}_m \}.
\end{align*}
Then we have $\#(\mathcal{G}_n) \leq n$ and for any $g \in \mathcal{G}_n$, there holds
\begin{align*}
\|g\|_{L_2(\mathrm{Q})} 
\leq M'm^{-1/v}.
\end{align*}
Moreover, we have
\begin{align} \label{SpaceSplit}
\mathrm{Co}(\mathcal{F}_n) \subset \mathrm{Co}(\mathcal{F}_m) + \mathrm{Co}(\mathcal{G}_n).
\end{align}

Applying Lemma 2.6.11 in \cite{vandervart1996weak} with $\varepsilon := \frac{1}{2} c_1 m^{1/v} n^{-v'}$ 
to $\mathcal{G}_n$, we can find a $\frac{1}{2}c_1 M'n^{-v'}$-net over $\mathrm{Co}(\mathcal{G}_n)$ consisting of at most 
\begin{align} \label{CapacityCoGn}
(e + e n \varepsilon^2)^{2/\varepsilon^2} 
\leq \biggl( e + \frac{e c_1^2}{\ell^{2/v}} \biggr)^{8 \ell^{2/v} c_1^{-2} n}
\end{align}
elements.

Suppose that \eqref{eq::main} holds for $k = 1$ and $n = m$.
In other words, there exists a $c_1 M' m^{-v'}$-net over $\mathrm{Co}(\mathcal{F}_m)$ consisting of at most $e^m$ elements, which partitions $\mathrm{Co}(\mathcal{F}_m)$ into $m$-dimensional cells of diameter at most $2 c_1 M' m^{-v'}$. 
Each of these cells can be isometrically identified with a subset of a ball of radius $c_1 M' m^{-v'}$ in $\mathbb{R}^m$
and can be therefore further partitioned into
\begin{align*}
\bigg(\frac{3c_1 M' m^{-v'}}{\frac{1}{2} c_1 M' n^{-v'}} \bigg)^m = (6\ell^{v'})^{n/\ell}
\end{align*}
cells of diameter $\frac{1}{2}c_1 M' n^{-v'}$. As a result, we get a $\frac{1}{2}c_1 M' n^{-v'}$-net of $\mathrm{Co}(\mathcal{F}_m)$ containing at most 
\begin{align} \label{CapacityCoFm}
e^m \cdot (6\ell^{v'})^{n/\ell}
\end{align}
elements.

Now, \eqref{SpaceSplit} together with \eqref{CapacityCoGn} and \eqref{CapacityCoFm}
yields that there exists a $c_1 M' n^{-v'}$-net of $\mathrm{Co}(\mathcal{F}_n)$ 
whose cardinality can be bounded by 
\begin{align*}
e^{n/\ell} \bigl( 6 \ell^{v'} \bigr)^{n/\ell} 
\biggl( e + \frac{e c_1^2}{\ell^{2/v}} \biggr)^{8 \ell^{2/v} c_1^{-2} n}
\leq e^n,
\end{align*}
for suitable choices of $c_1$ and $\ell$ depending only on $v$. 
This concludes the proof of \eqref{eq::main} for $k=1$ and every $n \in \mathbb{N}$.

Let us consider a general $k \in \mathbb{N}$. 
Similarly as above, there holds
\begin{align} \label{SpaceSplitGeneral}
\mathrm{Co}(\mathcal{F}_{nk^q}) \subset \mathrm{Co}(\mathcal{F}_{n(k-1)^q}) + \mathrm{Co}(\mathcal{G}_{n,k}),
\end{align}
where the set $\mathcal{G}_{n,k}$ contains at most $n k^q$ elements with norm smaller than $M'(n(k-1)^q)^{-1/v}$.
Applying Lemma 2.6.11 in \cite{vandervart1996weak} to $\mathcal{G}_{n,k}$, we can find an $M'k^{-2}n^{-v'}$-net over $\mathrm{Co}(\mathcal{G}_{n,k})$ consisting of at most 
\begin{align} \label{CapacityCoGnk}
\bigl( e + e k^{2q/v-4+q} \bigr)^{2^{2q/v+1}k^{4-2q/v}n}
\end{align}
elements. Moreover, by the induction hypothesis, we have a $c_{k-1}M'n^{-v'}$-net over $\mathrm{Co}(\mathcal{F}_{n(k-1)^q})$ consisting of at most 
\begin{align} \label{CapacityCoFnk-1q}
e^{c'_{k-1}n}
\end{align}
elements. Using \eqref{SpaceSplitGeneral}, \eqref{CapacityCoGnk}, and \eqref{CapacityCoFnk-1q},
we obtain a $c_k M' n^{-v'}$-net over $\mathrm{Co}(\mathcal{F}_{nk^q})$ consisting of at most $e^{c'_k n}$ elements, where
\begin{align*}
c_k &= c_{k-1} + \frac{1}{k^2},
\\
c'_k &= c'_{k-1} + 2^{2q/v+1}\frac{1+\log(1+k^{2q/v-4+q})}{k^{2q/v-4}}.
\end{align*}
Form the elementary analysis we know that if $2q/v - 5 = 2$, 
then there exist constants $c''_1$, $c''_2$, and $c''_3$ such that
\begin{align*}
\lim_{k \to \infty} c_k 
& = c^{-1/v} n_0^{(v+2)/2v} + \sum_{i=2}^{\infty} 1/i^2 
\leq c''_1 c^{-1/v} + c''_2,
\\
\lim_{k \to \infty} c'_k 
& = 1 + c \sum_{i=1}^{\infty} 2 (2/i)^{2q/v}i^5 
\leq c''_3.
\end{align*}
Thus \eqref{eq::main} is proved. 
Taking $\varepsilon := c_k M' n^{-v'} / M$ in \eqref{eq::main}, we get 
\begin{align*}
\log \mathcal{N} ( \mathrm{Co}(\mathcal{F}_{nk^q}), L_2(\mathrm{Q}), M \varepsilon)
\leq c'_k c_k^{1/v'} (M')^{1/v'}M^{-1/v'}\varepsilon^{-1/v'}.
\end{align*}
This together with $(M')^{1/v'} = c^{2v/(v+2)} M \leq c^2M$ yields
\begin{align*}
\log \mathcal{N} ( \mathrm{Co}(\mathcal{F}), L_2(\mathrm{Q}), M \varepsilon)
\leq c' c^{-2/(v+2)}\varepsilon^{-2v/(v+2)},
\end{align*}
where the constant $c'$ depends on the constants $c''_1$, $c''_2$ and $c''_3$.
This completes the proof.
\end{proof}

\begin{proof}[of Lemma \ref{conv}]
Lemma \ref{lem::VCreg} tells us that for any probability measure $\mathrm{Q}$, there holds
\begin{align*}
\mathcal{N}(\mathcal{F}_H, L_2(\mathrm{Q}), M \varepsilon) \leq 2K(c_d R/\overline{h}_0)^d(16e)^{2(c_d R/\overline{h}_0)^d}(1/\varepsilon)^{4(c_d R/\overline{h}_0)^d}.
\end{align*}
Consequently, for any $\varepsilon \in (0, 1/\max \{e, 2K\})$, we have
\begin{align*}
& \log \mathcal{N}(\mathcal{F}_H, \|\cdot\|_{L_2(\mathrm{D})}, M \varepsilon) 
\\
& \leq \log \Bigl( 2K(c_d R/\overline{h}_0)^d(16e)^{2(c_d R/\overline{h}_0)^d}(1/\varepsilon)^{4(c_d R/\overline{h}_0)^d}\Bigr)
\\
& = \log 2K 
+ d \log (c_d R/\underline{h}_0) 
+ 2 (c_d R/\underline{h}_0)^d \log (16e)
+ 4(c_d R/\underline{h}_0)^d \log(1/\varepsilon)
\\
& \leq 16 (c_d R/\underline{h}_0)^d \log (1/\varepsilon),
\end{align*}
where the last inequality is based on the following basic inequalities:
\begin{align*}
\log 2K 
& \leq \log (1/\varepsilon) 
\leq (c_d R/\underline{h}_0)^d \log (1/\varepsilon),
\\
d \log (c_d R/\underline{h}_0) 
& \leq (c_d R/\underline{h}_0)^d
\leq (c_d R/\underline{h}_0)^d \log (1/\varepsilon),
\\
(c_d R/\underline{h}_0)^d \log (16 e) 
& \leq (c_d R/\underline{h}_0)^d \log(e^5) 
\leq 5 (c_d R/\underline{h}_0)^d 
\leq 5 (c_d R/\underline{h}_0)^d \log (1/\varepsilon).
\end{align*}
Consequently, for all $\delta \in (0,1)$, we have
\begin{align}
\mathcal{N}(\mathcal{F}_H, \|\cdot\|_{L_2(\mathrm{D})}, \varepsilon) \leq (1/\varepsilon)^{16 (c_d R/\underline{h}_0)^d}.
\end{align} 
Applying Lemma \ref{thm::vart} with $v = \mathrm{VC}(\mathrm{Co}(\mathcal{F}_H))$, we then have 
\begin{align}
\log \mathcal{N} \bigl( \mathrm{Co}(\mathcal{F}_H), L_2(\mathrm{Q}), M \varepsilon \bigr) 
& \leq K (1/\varepsilon)^{2v/(v+2)} 
\nonumber\\
& \leq K (1/\varepsilon)^{2 - 4/(16(c_d R/\underline{h}_0)^d+2)}
\nonumber\\
& = K (1/\varepsilon)^{2 - 2/(8(c_d R/\underline{h}_0)^d+1)},
\label{equ::convexhulllogcovering}
\end{align}
which proves the assertion.
\end{proof}

\begin{proof}[of Proposition \ref{thm::OracleForest}]
Denote
\begin{align*}
r_c^* := \inf_{f \in \mathrm{Co}(\mathcal{F}_H)} \lambda\underline{h}_0^{-2d}+\mathcal{R}_{L,\mathrm{P}}(f)-\mathcal{R}_{L,\mathrm{P}}^*,
\end{align*}
and for $r > r_c^*$, we write
\begin{align*}
\mathcal{F}_r^c &:= \{f \in \mathrm{Co}(\mathcal{F}_H) : \lambda\underline{h}_0^{-2d}+\mathcal{R}_{L,\mathrm{P}}(f)-\mathcal{R}_{L,\mathrm{P}}^* \leq r \},
\\
\mathcal{H}_r^c &:= \{L\circ f - L\circ f_{L,\mathrm{P}}^* : f \in \mathcal{F}_r^c \}.
\end{align*}
Let $\delta := 1/(8(c_d R/\underline{h}_0)^d+1)$,
$\delta' := 1 - \delta$, and $a := K^{1/(2\delta')}M$. Then \eqref{equ::convexhulllogcovering} implies 
\begin{align*}
\log\mathcal{N}(\mathcal{H}_r^c, L_2(\mathrm{Q}), \varepsilon) 
&\leq \log\mathcal{N}(\mathrm{Co}(\mathcal{F}_H), L_2(\mathrm{Q}), \varepsilon) 
\\
&\leq K\big(M/\varepsilon)^{2 - 2/(8(c_d R/\underline{h}_0)^d+1)}
= (a/\varepsilon)^{2\delta'}.
\end{align*}
This together with \eqref{equ::r*} yields 
\begin{align*}
e_i(\mathcal{H}_r^c, \|\cdot\|_{L_2(\mathrm{Q})}) \leq 3^{1/(2\delta')}ai^{-1/(2\delta')}=(3K)^{1/(2\delta')}Mi^{-1/(2\delta')}.
\end{align*}
Taking expectation with respect to $\mathrm{P}^n$, we get
\begin{align} \label{EntropyEstimateHrc}
\mathbb{E}_{D \sim \mathrm{P}^n} e_i(\mathcal{H}_r^c, \|\cdot\|_{L_2(\mathrm{Q})})
\leq (3K)^{1/(2\delta')}Mi^{-1/(2\delta')}.
\end{align}
From the definition of $\mathcal{F}_r^c$ we easily find
\begin{align*}
\lambda\underline{h}_0^{-2d} 
\leq \lambda\underline{h}_0^{-2d}+\mathcal{R}_{L,\mathrm{P}}(g)-\mathcal{R}_{L,\mathrm{P}}^*
\leq r,
\end{align*}
which yields 
\begin{align*}
\underline{h}_0^{-1} \leq (r/\lambda)^{1/(2d)}.
\end{align*}
Therefore, if $\underline{h}_0 \leq 1$, then we have $r/\lambda \geq 1$ and 
\eqref{EntropyEstimateHrc} can be further estimated by
\begin{align*}
\mathbb{E}_{D \sim \mathrm{P}^n}e_i(\mathcal{H}_r^c, \|\cdot\|_{L_2(\mathrm{Q})})
& \leq (3K)^{1/(2\delta')}Mi^{-1/(2\delta')} 
\\
& \leq (3K)^{1/(2\delta')}M (r/\lambda)^{1/(4\delta')}i^{-1/(2\delta')}.
\end{align*}
From the definition of $\mathcal{H}_r^c$ we easily see that
for all $h \in \mathcal{H}_r^c$, there holds
\begin{align*}
\|h\|_{\infty} \leq 4 =: B_1,
\qquad
\mathbb{E}_{\mathrm{P}} h^2 \leq 16 r =: \sigma^2.
\end{align*}
Then Theorem 7.16 in \cite{StCh08} with 
$a := (3K)^{1/(2\delta')} M (r/\lambda)^{1/(4\delta')} \geq B_1$
yields that there exist constants $c_1(\delta) > 0$ and $c_2(\delta) >0$ depending only on $\delta$ such that
\begin{align*}
\mathbb{E}_{D\sim \mathrm{P}^n} \mathrm{Rad}_D(\mathcal{H}_r^c, n)&\leq \max \Bigl\{ c_1(\delta)(3K)^{1/2}M^{\delta'}r^{1/4}\lambda^{-1/4}(16r)^{\frac{1-\delta'}{2}}n^{-\frac{1}{2}}, 
\\ 
&\phantom{=} \qquad \quad
c_2(\delta)(3K)^{\frac{1}{1+\delta'}}M^{\frac{2\delta'}{1+\delta'}}r^{\frac{1}{2(1+\delta')}}\lambda^{-\frac{1}{2(1+\delta')}}4^{\frac{1-\delta'}{1+\delta'}}n^{-\frac{1}{1+\delta'}} \Bigr\}
\\
&=\max \Bigl\{  
c'_1(\delta)
\lambda^{-\frac{1}{4}}n^{-\frac{1}{2}} \cdot r^{\frac{3-2\delta'}{4}},
c'_2(\delta)
\lambda^{-\frac{1}{2(1+\delta')}}n^{-\frac{1}{1+\delta'}} \cdot
r^{\frac{1}{2(1+\delta')}} \Bigr\} := \varphi_n(r)
\end{align*}
with the constants 
$c'_1(\delta) :=
c_1(\delta)(3K)^{1/2}M^{\delta'}16^{\frac{1-\delta'}{2}}$ and
$c'_2(\delta) :=
c_2(\delta)(3K)^{\frac{1}{1+\delta'}}M^{\frac{2\delta'}{1+\delta'}}4^{\frac{1-\delta'}{1+\delta'}}$.
Simple algebra shows that
the condition $\varphi_n(4r) \leq 2\sqrt{2}\varphi_n(r)$ is satisfied. 
Since $2\sqrt{2} < 4$, similar arguments 
show that the statements of the peeling Theorem 7.7 in \cite{StCh08} still hold. Therefore, Theorem 7.20 in \cite{StCh08} can be applied, 
if the assumptions on $\varphi_n$ and $r$ are modified to $\varphi_n(4r) \leq 2\sqrt{2}\varphi_n(r)$ and 
$r\geq \max\{75\varphi_n(r),1152M^2\tau/n, r^* \}$,
respectively.
Some elementary calculations show that the condition $r\geq 75\varphi_n(r)$ is satisfied if
\begin{align*}
r &\geq \max\bigg\{(75 c'_1(\delta) \lambda^{-1/4}n^{-\frac{1}{2}})^{\frac{4}{1+2\delta'}}, (75 c'_2(\delta) \lambda^{-\frac{1}{2(1+\delta')}}n^{-\frac{1}{1+\delta'}})^{\frac{2(1+\delta')}{1+2\delta'}}\bigg \}
\\
&= \max\Big\{(75 c'_1(\delta))^{\frac{4}{1+2\delta'}}, (75 c'_2(\delta))^{\frac{2(1+\delta')}{1+2\delta'}}\Big \}\lambda^{-\frac{1}{1+2\delta'}}n^{-\frac{2}{1+2\delta'}},
\end{align*}
which yields the assertion.
\end{proof}

\subsubsection{Proofs Related to Section \ref{sec::AError4}}

\begin{proof}[of Proposition \ref{counterapprox}]
Recall that the regression model is defined as $Y = f(X) +\varepsilon$. Considering the case when $X$ follows the uniform distribution, for any $x=(x_1,\ldots,x_d) \in \mathcal{X}$, we have
\begin{align*}
f_{\mathrm{P},H}^*(x) = \frac{1}{\mathrm{P}_X(A_H(x))}\int_{A_H(x)} f(x')\ dx' = \frac{1}{\mu(A_H(x))}\int_{A_H(x)} f(x')\ dx'.
\end{align*}
Then we get
\begin{align*}
(f_{\mathrm{P},H}^*(x)-f(x))^2 
&= \biggl(f(x)-\frac{1}{\mu(A_H(x))}\int_{A_H(x)}f(x')\ dx'\biggr)^2
\\
&=\frac{1}{\mu(A_H(x))^2}\biggl(\int_{A_H(x)} f(x') - f(x)\ dx' \biggr)^2.
\end{align*}
Lemma \ref{binset} implies that for any $x' \in A_H(x)$, there exist a random vector $u \sim \mathrm{Unif}[0,1]^d$ and a vector $v \in [0,1]^d$ such that 
\begin{align} 
x' = x + S^{-1} R^{\top} (- u + v).
\end{align}
Therefore, we have
\begin{align}
dx' = \det \biggl( \frac{dx'}{dv} \biggr) dv
& = \det \biggl( \frac{d(x + S^{-1} R^{\top}(- u + v))}{dv} \biggr) dv
\nonumber\\
& = \det (R S^{-1}) dv
= \biggl( \prod_{i=1}^d h_i \biggr) dv.
\end{align}
Moreover, \eqref{TaylorEntwicklung} yields that
there exists a constant $c_{\alpha} \in [-c_L, c_L]$ such that
\begin{align}
f(x') - f(x) = \nabla f(x)^{\top} S^{-1}R^{\top} (- u + v) + c_{\alpha} \overline{h}_0^{1+\alpha}.
\end{align}
Taking expectation with regard to $\mathrm{P}_H$ and $\mathrm{P}_X$, we get
\begin{align}\label{equ::approxc1}
&\mathbb{E}_{\mathrm{P}_X}(f_{\mathrm{P},H}^*(X)-f(X))^2
\nonumber
\\
&\geq \mathbb{E}_{\mathrm{P}_X}(f_{\mathrm{P},H}^*(X)-f_{L,\mathrm{P}}^*(X))^2\eins_{B_{R,\sqrt{d}\cdot \overline{h}_0}^+}(X)
\nonumber
\\
&= \int_{B_{R,\sqrt{d}\cdot \overline{h}_0}^+}(f_{\mathrm{P},H}^*(x)-f_{L,\mathrm{P}}^*(x))^2\ d\mathrm{P}_X
\nonumber
\\
&= \int_{B_{R,\sqrt{d}\cdot \overline{h}_0}^+} \frac{1}{\mu(A_H(x))^2}\biggl(\int_{A_H(x)} \nabla f(x)^{\top} S^{-1} R^{\top} (- u + v)+c_{\alpha} \overline{h}_0^{1+\alpha} \ dy \biggr)^2 \ d\mathrm{P}_X
\nonumber
\\
&= \int_{B_{R,\sqrt{d}\cdot \overline{h}_0}^+} \frac{(\prod_{i=1}^d h_i)^2}{\mu(A_H(x))^2}  \biggl(\int_{[0,1]^d} (-u+v)^T \ dv R S^{-1}\nabla f(x) +c_{\alpha} \overline{h}_0^{1+\alpha} \biggr)^2 \ d\mathrm{P}_X
\nonumber
\\
&=\int_{B_{R,\sqrt{d}\cdot \overline{h}_0}^+}\biggl(\biggl(\frac{1}{2}-u\biggr)^T R S^{-1}\nabla f(x)+c_{\alpha} \overline{h}_0^{1+\alpha} \biggr)^2 \ d\mathrm{P}_X
\nonumber
\\
&=\int_{B_{R,\sqrt{d}\cdot \overline{h}_0}^+}\biggl( \sum_{i=1}^d\biggl(\frac{1}{2}-u_i \biggr)\sum_{j=1}^d R_{ij}h_j\frac{\partial f}{\partial x_j}+c_{\alpha} \overline{h}_0^{1+\alpha} \biggr)^2 \ d\mathrm{P}_X.
\end{align}
Since the random variables $(u_i)_{i=1}^d$ are independent and identically distributed as $\mathrm{Unif}[0, 1]$, we have
\begin{align} \label{CrossTermProperty3}
\mathbb{E}_{\mathrm{P}_H} \bigg(\frac{1}{2}-u_i \bigg)=0,
\qquad \qquad
i = 1, \ldots, d,
\end{align}
and
\begin{align} \label{CrossTermProperty4}
\mathbb{E}_{\mathrm{P}_H} \bigg(\frac{1}{2}-u_i \bigg)^2=\frac{1}{12},
\qquad \qquad
i = 1, \ldots, d.
\end{align}
Therefore, we have
\begin{align*}
&\mathbb{E}_{\mathrm{P}_H}\int_{B_{R,\sqrt{d}\cdot \overline{h}_0}^+}\biggl( \sum_{i=1}^d\biggl(\frac{1}{2}-u_i \biggr)\sum_{j=1}^d R_{ij}h_j\frac{\partial f}{\partial x_j}+c_{\alpha} \overline{h}_0^{1+\alpha} \biggr)^2 \ d\mathrm{P}_X
\nonumber
\\
&=\int_{B_{R,\sqrt{d}\cdot \overline{h}_0}^+} \mathbb{E}_{\mathrm{P}_H}\sum_{i=1}^d \biggl(\frac{1}{2}-u_i \biggr)^2\biggl(\sum_{j=1}^d R_{ij}h_j\frac{\partial f}{\partial x_j} \biggr)^2 \ d\mathrm{P}_X.
\end{align*}
Moreover, the orthogonality \eqref{RotationMatrix} of the rotation matrix $R$ tells us that
\begin{align} \label{CrossTermProperty1}
\sum_{i=1}^d R_{ij} R_{ik} = 
\begin{cases}
1, & \text{ if } j = k, \\
0,& \text{ if } j \neq k
\end{cases}
\end{align} 
and consequently we have
\begin{align} \label{CrossTermProperty2}
\sum_{i=1}^d \sum_{j \neq k} R_{ij} R_{ik} h_j h_k \cdot \frac{\partial f(x)}{\partial x_j} \cdot \frac{\partial f(x)}{\partial x_k}
= \sum_{j \neq k} h_j h_k \cdot \frac{\partial f(x)}{\partial x_j} \cdot \frac{\partial f(x)}{\partial x_k} \sum_{i=1}^d R_{ij} R_{ik}
= 0.
\end{align}
For any $n > N'$, we have
\begin{align*}
(R-2\sqrt{d}\cdot \overline{h}_0)^d \geq (R/2)^d.
\end{align*}
Consequently, \eqref{CrossTermProperty1} and \eqref{CrossTermProperty2} imply that 
\begin{align}
&\int_{B_{R,\sqrt{d}\cdot \overline{h}_0}^+} \mathbb{E}_{\mathrm{P}_H}\sum_{i=1}^d \biggl(\frac{1}{2}-u_i \biggr)^2\biggl(\sum_{j=1}^d R_{ij}h_j\frac{\partial f}{\partial x_j} \biggr)^2 \ d\mathrm{P}_X
\nonumber
\\
&=\int_{B_{R,\sqrt{d}\cdot \overline{h}_0}^+}\sum_{i=1}^d\frac{1}{12}\mathbb{E}_{\mathrm{P}_R} \sum_{j=1}^d R_{ij}^2h_j^2\bigg(\frac{\partial f}{\partial x_j}\bigg)^2  \ d\mathrm{P}_X
\nonumber
\\
&\geq \int_{B_{R,\sqrt{d}\cdot \overline{h}_0}^+ \cap \mathcal{A}_f}\frac{1}{12}\underline{h}_0^2\underline{c}_f^2  \ d\mathrm{P}_X
\geq \frac{1}{12} \biggl( \frac{R}{2} \biggr)^d c_0^2
\mathrm{P}_X(\mathcal{A}_f) \underline{c}_f^2 \cdot
\overline{h}_0^2.
\label{equ::approx}
\end{align}
Thus, the assertion is proved.
\end{proof}

\subsubsection{Proofs Related to Section \ref{sec::SError4}}

\begin{proof}[of Proposition \ref{counterapprox2}]
For any fixed $j \in \mathcal{I}_H$, 
we define the random variable $Z_j$ by
\begin{align*}
Z_j := \sum_{i=1}^n \eins_{A_j}(X_i).
\end{align*} 
Since the random variables $\{ \eins_{A_j}(X_i) \}_{i=1}^n$ are i.i.d.~Bernoulli distributed with parameter 
$\mathrm{P}(X\in A_j)$, elementary probability theory implies that the random variable $Z_j$ is Binomial distributed with parameters $n$ and $\mathrm{P}(X\in A_j)$. Therefore, for any $j \in \mathcal{I}_H$, we have
\begin{align*}
\mathbb{E}(Z_j) = n \cdot \mathrm{P}(X \in A_j). 
\end{align*}
Moreover, the single NHT regressor $f_{\mathrm{D},H}$ can be defined by
\begin{align*}
f_{\mathrm{D},H}(x) =
\begin{cases}
\displaystyle
\frac{\sum_{i=1}^nY_i\eins_{A_j}(X_i)}{\sum_{i=1}^n\eins_{A_j}(X_i)}\eins_{A_j}(x) & \text{ if } Z_j > 0,
\\
0 & \text{ if } Z_j = 0.
\end{cases}
\end{align*}
By the law of total probability, we get 
\begin{align}
& \mathbb{E}_{\mathrm{P}_X} \bigl( f_{\mathrm{D},H}(X) - f_{\mathrm{P},H}^*(X) \bigr)^2
\nonumber\\
& = \sum_{j \in \mathcal{I}_H} \mathbb{E}_{\mathrm{P}_X}
\bigl( \bigl( f_{\mathrm{D},H}(X) - f_{\mathrm{P},H}^*(X) \bigr)^2 \big| X \in A_j \bigr)
\cdot \mathrm{P}(X\in A_j)
\nonumber\\
& = \sum_{j \in \mathcal{I}_H} \mathbb{E}_{\mathrm{P}_X}
\bigl( \bigl( f_{\mathrm{D},H}(X) - f_{\mathrm{P},H}^*(X) \bigr)^2 \big| X \in A_j, Z_j > 0 \bigr)
\cdot \mathrm{P}(Z_j > 0) \cdot \mathrm{P}(X \in A_j)
\label{equ::term1}
\\
& \phantom{=}
+ \sum_{j \in \mathcal{I}_H} \mathbb{E}_{\mathrm{P}_X}
\bigl( \bigl( f_{\mathrm{D},H}(X) - f_{\mathrm{P},H}^*(X) \bigr)^2 \big| X \in A_j, Z_j = 0 \bigr)
\cdot \mathrm{P}(Z_j = 0) \cdot \mathrm{P}(X \in A_j)
\label{equ::term2}.
\end{align}
For the term \eqref{equ::term1}, we have
\begin{align*}
& \sum_{j \in \mathcal{I}_H} \mathbb{E}_{\mathrm{P}_X}((f_{\mathrm{D},H}(X)-f_{\mathrm{P},H}^*(X))^2|X\in A_j, Z_j>0)\mathrm{P}(Z_j > 0)\mathrm{P}(X\in A_j)
\\
&=\sum_{j\in\mathcal{I}_H} \biggl(
\frac{\sum_{i=1}^nY_i\eins_{A_j}(X_i)}{\sum_{i=1}^n\eins_{A_j}(X_i)}
- \mathbb{E}(f_{L,\mathrm{P}}^*(X)|X\in A_j)
\biggr)^2\mathrm{P}(Z_j > 0)\mathrm{P}(X\in A_j)
\\
&=\sum_{j\in\mathcal{I}_H}\frac{\mathrm{P}(X\in A_j)}{(\sum_{i=1}^n\eins_{A_j}(X_i))^2}\biggl(\sum_{i=1}^n\eins_{A_j}(X_i)(Y_i-\mathbb{E}(f_{L,\mathrm{P}}^*(X)|X\in A_j))\biggr)^2\mathrm{P}(Z_j > 0),
\end{align*}
which yields that for a fixed $j \in \mathcal{I}_H$, there holds
\begin{align}
& \mathbb{E} \biggl( \sum_{j \in \mathcal{I}_H} \frac{\mathrm{P}(X \in A_j)}{(\sum_{i=1}^n \eins_{A_j}(X_i))^2}
\biggl( \sum_{i=1}^n \eins_{A_j}(X_i) \bigl( Y_i - \mathbb{E}(f_{L,\mathrm{P}}^*(X) | X \in A_j) \bigr) \biggr)^2
\bigg| X_i \in A_j \biggr)
\nonumber\\
& = \sum_{j \in \mathcal{I}_H}
\frac{\mathrm{P}(X \in A_j)}{(\sum_{i=1}^n \eins_{A_j}(X_i))^2} 
\sum_{i=1}^n \eins_{A_j}^2(X_i) \mathbb{E} \bigl( \bigl( Y - f_{\mathrm{P},H}^*(X) \bigr)^2 \big| X \in A_j \bigr)
\nonumber\\
& = \sum_{j \in \mathcal{I}_H}
\frac{\mathrm{P}(X \in A_j)}{\sum_{i=1}^n \eins_{A_j}(X_i)} \mathbb{E} \bigl( \bigl( Y - f_{\mathrm{P},H}^*(X))^2 \big| X \in A_j \bigr).
\label{ConditionalExpectationOnAj}
\end{align} 
Obviously, for any fixed $j \in \mathcal{I}_H$, there holds
\begin{align*}
\mathbb{E}(f_{\mathrm{P},H}^*(X)|X\in A_j) = \mathbb{E}(f_{L,\mathrm{P}}^*(X)|X\in A_j) 
\end{align*}
and consequently we obtain
\begin{align*}
& \mathbb{E}((Y-f_{\mathrm{P},H}^*(X))^2|X\in A_j)
\\
&=\mathbb{E}((Y-f_{L,\mathrm{P}}^*(X))^2|X\in A_j)+\mathbb{E}((f_{L,\mathrm{P}}^*(X)-f_{\mathrm{P},H}^*(X))^2|X\in A_j)
\\
&=\sigma^2+\mathbb{E}((f_{L,\mathrm{P}}^*(X)-f_{\mathrm{P},H}^*(X))^2|X\in A_j).
\end{align*}
Taking expectation over both sides of \eqref{ConditionalExpectationOnAj} with respect to $\mathrm{P}^n$, we get
\begin{align*}
& \mathbb{E}_{D\sim \mathrm{P}^n} \mathbb{E}_{\mathrm{P}_X} 
\bigl( f_{\mathrm{D},H}(X) - f_{\mathrm{P},H}^*(X))^2
\\
&=\mathbb{E}_{D \sim \mathrm{P}^n} \bigl( 
\mathbb{E} \bigl( \mathbb{E}_{\mathrm{P}_X} 
\bigl( f_{\mathrm{D},H}(X) - f_{\mathrm{P},H}^*(X))^2 \big| X_i \in A_j \bigr) \bigr)
\\
& = \bigl( \sigma^2 + \mathbb{E}(f_{L,\mathrm{P}}^*(X) - f_{\mathrm{P},H}^*(X))^2 \bigr)
\\
&\phantom{=}\cdot
\sum_{j \in \mathcal{I}_H}
\biggl(
\mathrm{P}(X \in A_j) \mathbb{E}_{D \sim \mathrm{P}^n} \biggl(\biggl( \sum_{i=1}^n\eins_{A_j}(X_i) \biggr)^{-1}\bigg|Z_j>0\biggr) \biggr)\mathrm{P}(Z_j > 0)
\\
& = \bigl( \sigma^2 + \mathbb{E}(f_{L,\mathrm{P}}^*(X) - f_{\mathrm{P},H}^*(X))^2 \bigr)
\\
&\phantom{=}\cdot\sum_{j \in \mathcal{I}_H} \bigl( n^{-1} \cdot n \cdot
\mathrm{P}(X \in A_j) \mathbb{E}_{D \sim \mathrm{P}^n} (Z_j^{-1} | Z_j > 0) \bigr)\mathrm{P}(Z_j > 0)
\\
& = n^{-1} \bigl( \sigma^2 + \mathbb{E}(f_{L,\mathrm{P}}^*(X) - f_{\mathrm{P},H}^*(X))^2 \bigr)
\\
&\phantom{=}\cdot
\sum_{j \in \mathcal{I}_H}  \bigl( \mathbb{E}(Z_j) \cdot \mathbb{E}(Z_j^{-1} | Z_j > 0) \bigr)\mathrm{P}(Z_j > 0).
\end{align*}
Clearly, $x^{-1}$ is convex for $x > 0$.
Therefore, by Jensen's inequality, we get
\begin{align*}
\mathbb{E}(Z_j) \cdot \mathbb{E}(Z_j^{-1}|Z>0)\mathrm{P}(Z_j > 0)
& \geq \mathbb{E}(Z_j) \cdot \mathbb{E}(Z_j|Z_j>0)^{-1}\mathrm{P}(Z_j > 0)
\\
&=\mathbb{E}(Z) \cdot \mathbb{E}(Z \eins_{\{Z > 0\}})^{-1} \mathrm{P}(Z > 0)\mathrm{P}(Z > 0)
\\
&=\mathrm{P}(Z>0)^2 = (1 - \mathrm{P}(Z=0))^2
\\
&= (1-(1-\mathrm{P}(X\in A_j))^n)^2
\\
&\geq 1-2e^{-n\mathrm{P}(X\in A_j)},
\end{align*}
where the last inequality follows from $(1-x)^n \leq e^{-nx}$, $x \in (0,1)$.

We now turn to estimate the term \eqref{equ::term2}. 
By the definition of $f_{\mathrm{D},H}$, there holds
\begin{align*}
& \sum_{j \in \mathcal{I}_H} \mathbb{E}_{\mathrm{P}_X}
\bigl( \bigl( f_{\mathrm{D},H}(X) - f_{\mathrm{P},H}^*(X))^2 \big| X \in A_j, Z_j = 0 \bigr)
\cdot \mathrm{P}(Z_j = 0) \cdot \mathrm{P}(X \in A_j)
\\
& = \sum_{j \in \mathcal{I}_H} \mathbb{E}_{\mathrm{P}_X} \bigl( \bigl( f_{\mathrm{P},H}^*(X) \bigr)^2 \big| X \in A_j \bigr) \cdot \mathrm{P}(Z_j = 0) \cdot \mathrm{P}(X \in A_j)
\geq 0.
\end{align*}
Let us denote
\begin{align*}
\mathcal{I}_{H}^{(1)} := \{j \in \mathcal{I}_H : A_j \cap B_R = A_j \}
\end{align*}
and 
\begin{align*}
\mathcal{I}_{H}^{(2)} := \mathcal{I}_H \setminus \mathcal{I}_{H}^{(1)}.
\end{align*}
Then we obviously have $\mathrm{P} (X \in A_j) = \mu(A_j) \geq \underline{h}_0^d$ for all $j \in \mathcal{I}_{H}^{(1)}$.
Combing the above results, we obtain
\begin{align*}
& \mathbb{E}_{D\sim \mathrm{P}^n} \mathbb{E}_{\mathrm{P}_X} 
\bigl( f_{\mathrm{D},H}(X) - f_{\mathrm{P},H}^*(X))^2
\\
&=\sum_{j\in\mathcal{I}_H}\mathbb{E}_{\mathrm{P}_X}((f_{\mathrm{D},H}(X)-f_{\mathrm{P},H}^*(X))^2|X\in A_j, Z_j>0) \cdot \mathrm{P}(Z_j > 0) \cdot \mathrm{P}(X\in A_j)
\\
& \phantom{=}
+ \sum_{j \in \mathcal{I}_H} \mathbb{E}_{\mathrm{P}_X}((f_{\mathrm{D},H}(X)-f_{\mathrm{P},H}^*(X))^2|X\in A_j, Z_j=0) \cdot \mathrm{P}(Z_j = 0) \cdot \mathrm{P}(X\in A_j)
\\
&\geq \sum_{j\in\mathcal{I}_H}\mathbb{E}_{\mathrm{P}_X}((f_{\mathrm{D},H}(X)-f_{\mathrm{P},H}^*(X))^2|X\in A_j, Z_j>0) \cdot \mathrm{P}(Z_j > 0) \cdot \mathrm{P}(X\in A_j)
\\
&=\sum_{j\in\mathcal{I}_{H}^{(1)}}\mathbb{E}_{\mathrm{P}_X}((f_{\mathrm{D},H}(X)-f_{\mathrm{P},H}^*(X))^2|X\in A_j, Z_j>0) \cdot \mathrm{P}(Z_j > 0) \cdot \mathrm{P}(X\in A_j)
\\
&\phantom{=}+\sum_{j\in\mathcal{I}_{H}^{(2)}}\mathbb{E}_{\mathrm{P}_X}((f_{\mathrm{D},H}(X)-f_{\mathrm{P},H}^*(X))^2|X\in A_j, Z_j>0) \cdot \mathrm{P}(Z_j > 0) \cdot \mathrm{P}(X\in A_j)
\\
&\geq \sum_{j\in\mathcal{I}_{H}^{(1)}}\mathbb{E}_{\mathrm{P}_X}((f_{\mathrm{D},H}(X)-f_{\mathrm{P},H}^*(X))^2|X\in A_j, Z_j>0) \cdot \mathrm{P}(Z_j > 0) \cdot \mathrm{P}(X\in A_j)
\\
&\geq \frac{1}{n} \sum_{j \in \mathcal{I}_{H}^{(1)}} \bigl( 1 - 2 e^{-n\mathrm{P}(X \in A_j)} \bigr) 
\bigl( \sigma^2+\mathbb{E}(f_{L,\mathrm{P}}^*(X)-f_{\mathrm{P},H}^*(X))^2 \bigr)
\\
&\geq \frac{\sigma^2}{n}\biggl(|\mathcal{I}_{H}^{(1)}|-\sum_{j \in \mathcal{I}_{H}^{(1)}}2e^{-n\mathrm{P}(X\in A_j)}\biggr).
\end{align*} 
Therefore, we have
\begin{align}
\mathbb{E}_{D\sim \mathrm{P}^n} \mathbb{E}_{\mathrm{P}_X} 
\bigl( f_{\mathrm{D},H}(X) - f_{\mathrm{P},H}^*(X))^2
& \geq \frac{\sigma^2}{n}\biggl(|\mathcal{I}_{H}^{(1)}|-\sum_{j \in \mathcal{I}_{H}^{(1)}}2e^{-n\mathrm{P}(X\in A_j)}\biggr)
\nonumber\\
& = \frac{\sigma^2}{n} \biggl( |\mathcal{I}_{H}^{(1)}|
- 2 |\mathcal{I}_{H}^{(1)}| \exp \bigl( - n \underline{h}_0^d \bigr)  \biggr)
\nonumber\\
& \geq \frac{\sigma^2}{n} \biggl(\frac{2R-\sqrt{d} \cdot \overline{h}_0}{\overline{h}_0}\biggr)^d \biggl( 1 - \frac{2}{e} \biggr)
\nonumber\\
&\geq 4R^d\sigma^2(1-2e^{-1})\overline{h}_0^{-d}n^{-1},
\label{VarianceTermEstimateOne}
\end{align}
where the last inequality follows from Assumption \ref{assumption::h}.
\end{proof}

\subsubsection{Proofs Related to Section \ref{subsec::C1}}

\begin{proof}[of Theorem \ref{thm::optimalForest}]
Proposition \ref{thm::OracleForest} together with Proposition \ref{prop::biasterm} implies 
\begin{align*}
\mathcal{R}_{L_{\overline{h}_0}, \mathrm{P}}(f_{\mathrm{D},\mathrm{E}}) - \mathcal{R}_{L_{\overline{h}_0}, \mathrm{P}}^*
\lesssim \lambda_n (\underline{h}_{0,n})^{-2d}
+ \overline{h}_{0,n}^{2(1+\alpha)}
+ T_n^{-1} \overline{h}_{0,n}^2
+ \lambda_n^{-\frac{1}{1+2\delta'}} n^{-\frac{2}{1+2\delta'}},
\end{align*}
where $\delta' := 1 - \delta$ and 
$\delta := 1/(8(c_d R/\underline{h}_0)^d+1)$.
Choosing 
\begin{align*}
\lambda_n := n^{-\frac{1}{2(1+\alpha)+2d}}, 
\qquad
\overline{h}_{0,n} := n^{-\frac{1}{2(1+\alpha)(2-\delta)+d}}, 
\qquad
T_n := n^{\frac{2\alpha}{2(1+\alpha)(2-\delta)+d}},
\end{align*}
we obtain
\begin{align*}
\mathcal{R}_{L_{\overline{h}_0},\mathrm{P}}(f_{\mathrm{D},\mathrm{E}}) - \mathcal{R}_{L_{\overline{h}_0},\mathrm{P}}^*
\lesssim n^{-\frac{2(1+\alpha)}{2(1+\alpha)(2-\delta)+d}}.
\end{align*}
This completes the proof.
\end{proof}

\begin{proof}[of Theorem \ref{prop::counter}]
Recall the error decomposition \eqref{equ::L2Decomposition}. 
Using
the estimates \eqref{equ::approx} and \eqref{VarianceTermEstimateOne}
and choosing $\overline{h}_{0,n} := n^{-\frac{1}{d+2}}$, we get 
\begin{align*} 
& \mathbb{E}_{\nu_n} (\mathcal{R}_{L,\mathrm{P}}(f_{\mathrm{D},H_n}) - \mathcal{R}_{L,\mathrm{P}}^*)
= \mathbb{E}_{\nu_n}
\mathbb{E}_{\mathrm{P}_X} 
(f_{\mathrm{D},H_n}(X) - f_{L,\mathrm{P}}^*(X))^2 
\\
& \geq  \frac{d}{12} \biggl( \frac{R}{2} \biggr)^d c_0^2
\mathrm{P}_X(\mathcal{A}_f) \underline{c}_f^2 \cdot
\overline{h}_{0,n}^2 
+ 4 R^2 \sigma^2(1 - 2 e^{-1}) \overline{h}_{0,n}^{-d} n^{-1}
\gtrsim n^{-\frac{2}{2+d}},
\end{align*}
which proves the assertion.
\end{proof}

\subsection{Proofs of Results  for KHT  in the space $C^{k,\alpha}$}

\subsubsection{Proofs Related to Section \ref{sec::AError2}}
To prove Proposition \ref{thm::convolution}, we need the following lemmas.

\begin{lemma}\label{lem::difference}
Let $f \in C^{k, \alpha}(\mathbb{R})$ and the $q$-th difference of $f$ be defined by \eqref{q-thDifference}. Moreover, 
for $r \in \mathbb{N}$ with $r \leq k$, let $D^r f = f^{(r)}$ denote the $r$-th differentiation of $f$ and $N_{r,h}$ be the $r-1$-times convolution of $\eins_{[0,1]}$ with itself and $N_{r,h}(u)=\frac{1}{h}N_r(\frac{u}{h})$. Then we have
\begin{align} \label{Trianglehr}
\triangle_h^r(f, x) = \int_{\mathbb{R}} h^r D^r f(u) N_{r,h}(u - x) \, du.
\end{align}
\end{lemma}

\begin{proof}[of Lemma \ref{lem::difference}]
The proof is by induction on $r$. 
For any $x \in \mathbb{R}^d$, there holds
\begin{align*}
\triangle_h^1(f,x) 
& = f(x + h) - f(x) 
\\
& = \int_{x}^{x+h} D f(u) \, du
\\
& = \int_{\mathbb{R}} D f(u) \eins_{[x, x+h]}(u) \, du
\\
& = \int_{\mathbb{R}} D f(u) \eins_{[0,1]} \biggl( \frac{u-x}{h} \biggr) \, du
\\ 
& = \int_{\mathbb{R}} h D f(u) N_{1,h}(u - x) \, du.
\end{align*}
Therefore, \eqref{Trianglehr} holds when $r = 1$.
Now let $r \geq 1$ be given and suppose \eqref{Trianglehr} is true for $r$. 
Then we have
\begin{align*}
\triangle_h^{r+1}(f,x)
& = \triangle_h^1(\triangle_h^r(f(x,\cdot),x) 
\\
& = \triangle_h^r(f,x+h) - \triangle_h^r(f,x)
\\
& = \int_x^{x+h} D(\triangle_h^r(f))(v) \, dv
\\
& = \int_x^{x+h} D \biggl( \int_{\mathbb{R}} h^r D^r f(u) N_{r,h}(u-v) \, du \biggr) \, dv
\\
& = \int_{\mathbb{R}} D \biggl( \int_{\mathbb{R}} h^r D^r f(u) N_{r,h}(u-v) \, du \biggr) 
\eins_{[0,1]} \biggl( \frac{v-x}{h} \biggr) \, dv
\\
& = - \int_{\mathbb{R}} h^r \biggl( \int_{\mathbb{R}} D^r f(u) N_{r,h}(u-v) \, du \biggr) 
\eins'_{[0,1]} \biggl( \frac{v-x}{h} \biggr) \frac{1}{h} \, dv
\\
& = - h^{r-1} \int_{\mathbb{R}} D^rf(u) 
\biggl( \int_{\mathbb{R}} N_{r,h}(t) \eins'_{[0,1]} \biggl( \frac{u-x-t}{h} \biggr) \, dt \biggr) \, du,
\\
& = - h^{r-1} \int_{\mathbb{R}} D^r f(u) \biggl( \int_{\mathbb{R}} N_{r,h}(t) 
\eins'_{[0,1]} \biggl( \frac{u-x-t}{h} \biggr) \, dt \biggr) \, du
\\
& = - h^{r-1} \int_{\mathbb{R}} D^r f(u) \biggl( - \eins_{[0,1]} \biggl( \frac{u-x-t}{h} \biggr) h N_{r,h}(t) \bigg|_{-\infty}^{\infty}
\\ 
& \qquad \qquad \qquad \qquad \qquad 
+ h \int_{-\infty}^{\infty} \eins_{[0,1]} \biggl( \frac{u-x-t}{h} \biggr) N'_{r,h}(t) \, dt \biggr) \, du
\\
& = - h^r \int_{\mathbb{R}} D^r f(u) \biggl( 
\int_{-\infty}^{\infty} \eins_{[0,1]} \biggl( \frac{u-x-t}{h} \biggr) N'_{r,h}(t) \, dt \biggr) \, du
\\
& = - h^r \int_{\mathbb{R}} D^r f(u) \biggl( 
\int_{-\infty}^{\infty} \eins_{[0,1]} \biggl( \frac{u-x-t}{h} \biggr) 
\frac{1}{h^2} N'_r \biggl( \frac{t}{h} \biggr) \, dt \biggr) \, du
\\
& = - h^{r-1} \int_{\mathbb{R}} D^r f(u) 
\biggl( \int_{-\infty}^{\infty} \eins_{[0,1]}(s) N'_r \biggl( \frac{u-x}{h} - s \biggr) \, ds \biggr) \, du,
\end{align*}
where $\eins'_{[0,1]}(u)$ denotes the derivative of $\eins_{[0,1]}$ with respect to $u$.
Since $f * (\partial g) = \partial (f * g)$, we have 
\begin{align*}
(\eins_{[0,1]} * N'_r)(u) 
= (\eins_{[0,1]} * N_r)'(u) 
= N'_{r+1}(u)
\end{align*}
and consequently 
\begin{align*}
\triangle_h^{r+1}(f,x) 
& = - h^{r-1} \int_{\mathbb{R}} D^r f(u) N'_{r+1} \biggl( \frac{u-x}{h} \biggr) \, du
\\
& = h^r \int_{\mathbb{R}} D^{r+1} f(u) N_{r+1} \biggl( \frac{u-x}{h} \biggr) \, du
\\
& = \int_{\mathbb{R}} h^{r+1} D^{r+1} f(u) N_{r+1,h}(u-x) \, du.
\end{align*}
Thus, \eqref{Trianglehr} holds for $r + 1$, and the proof of the induction step is complete. By the principle of induction, \eqref{Trianglehr} is thus true for all $r \geq 1$.
\end{proof}

\begin{lemma}\label{lem::binomial}
Let $f : \mathbb{R}^d \to \mathbb{R}$ be a function
and the $q$-th difference of $f$ be defined by \eqref{q-thDifference}.
Moreover, for any $i = 1, \ldots, d$, let 
$g_i : \mathbb{R} \to \mathbb{R}$ be defined by
\begin{align*}
g_i(y) := f(x_1+h_1, \ldots, x_{i-1}+h_{i-1}, y, x_{i+1}, \ldots, x_d). 
\end{align*}
Then we have 
\begin{align} \label{SingleDecomposition}
\triangle_h^r(f, x) 
= \sum_{k = 0}^r \binom{r}{k} (-1)^{r-k} \sum_{i=1}^d g_i(x_i + kh_i)
= \sum_{i=1}^d \triangle_{h_i}^r(g_i, x_i).
\end{align}
\end{lemma}

\begin{proof}[of Lemma \ref{lem::binomial}]
The proof is by induction on $r$.
For any $x \in \mathbb{R}^d$, there holds
\begin{align*}
\triangle_h^1(f, x)
& = f(x+h) - f(x)
\\
& = f(x_1+h_1, \ldots, x_d+h_d) - f(x_1+h_1, \ldots, x_{d-1}+h_{d-1}, x_d)
\\
& \phantom{=} 
+ \cdots + f(x_1+h_1, x_2, \ldots, x_d) - f(x_1, x_2, \ldots, x_d)
\\
& = \sum_{i=1}^d (g_i(x_i+h_i) - g_i(x_i)).
\end{align*}
Therefore, \eqref{SingleDecomposition} holds when $r = 1$.
Now let $r \geq 1$ be given and suppose \eqref{SingleDecomposition} is true for $r$. 
Then we have
\begin{align*}
& \triangle_h^{r+1}(f,x) 
\\
& = \triangle_h^1(\triangle_h^{r}(f,x))
= \triangle_h^1 \biggl( \sum_{k=0}^r {r \choose k} (-1)^{r-k} \sum_{i=1}^d g_i(x_i+kh_i) \biggr)
\\
& = \sum_{k=0}^r {r \choose k} (-1)^{r-k} \sum_{i=1}^d \bigl( g_i(x_i+(k+1)h_i) - g_i(x_i+kh_i) \bigr)
\\
& = \sum_{i=1}^d \sum_{\ell=1}^{r+1} {r \choose {\ell-1}} (-1)^{r-\ell+1} g_i(x_i+\ell h_i) 
+ \sum_{i=1}^d \sum_{k=0}^r {r \choose k} (-1)^{r-k+1} g_i(x_i+kh_i)
\\
& = \sum_{i=1}^d \biggl( (-1)^{r+1} g_i(x_i) + g_i(x_i+(r+1)h_i) 
+ \sum_{\ell=1}^r \biggl( {r \choose {\ell-1}} + {r \choose \ell} \biggr) (-1)^{r-\ell+1} g_i(x_i+\ell h_i) \biggr)
\\
& = \sum_{i=1}^d \biggl( (-1)^{r+1} g_i(x_i) + g_i(x_i+(r+1)h_i) 
+ \sum_{\ell=1}^r {{r+1} \choose \ell} (-1)^{r+1-\ell} g_i(x_i+\ell h_i) \biggr)
\\
& = \sum_{i=1}^d \sum_{\ell=0}^{r+1} (-1)^{r+1-\ell} {{r+1} \choose \ell} g_i(x_i+\ell h_i)
\\
& = \sum_{i=1}^d \triangle_{h_i}^{r+1}(g_i, x_i).
\end{align*}
Thus, \eqref{SingleDecomposition} holds for $r + 1$, and the proof of the induction step is complete. By the principle of induction, \eqref{SingleDecomposition} is thus true for all $r \geq 1$.
\end{proof}

\begin{lemma}\label{lem::modulusone}
Let $f \in C^{k, \alpha}(\mathbb{R}^d)$
and the modulus of smoothness of $f$ be defined by \eqref{q-thModulusSmoothness}.
Then for any $t > 0$, there holds
\begin{align*}
\omega_{k+1, L_{\infty}(\mathbb{R}^d)}(f, t) \leq c_L d\, t^{k+\alpha}, 
\end{align*}
where $c_L$ is the constant as in Definition \ref{def::Cp}.
\end{lemma}

\begin{proof}[of Lemma \ref{lem::modulusone}]
By \eqref{SingleDecomposition}, we have 
\begin{align*}
\triangle_h^{k+1}(f, x) = \sum_{i=1}^d \triangle_{h_i}^{k+1}(g_i, x_i).
\end{align*}
Using the triangle inequality, we get
\begin{align}\label{deco}
\|\triangle_h^{k+1}(f, x)\|_{\infty} 
\leq \sum_{i=1}^d \|\triangle_{h_i}^{k+1}(g_i, x_i)\|_{\infty}.
\end{align}
Since $f \in C^{k, \alpha}(\mathbb{R}^d)$, we have $g_i \in C^{k, \alpha}(\mathbb{R})$ for all $i = 1, \ldots, d$. 
Thus, for any $i = 1, \ldots, d$ and $r \leq k-1$, there holds 
\begin{align*}
g_i^{(r)}(x_i+h_i) - g_i^{(r)}(x_i) 
= \int_{x_i}^{x_i+h_i} g_i^{(r+1)}(u) \, du.
\end{align*}
Then \eqref{Trianglehr} implies that for any $i = 1, \ldots, d$, we have
\begin{align*}
\triangle_{h_i}^k(g_i, x_i) 
= \int_{\mathbb{R}} h_i^k g_i^{(k)}(u) N_{k,h_i}(u - x_i) \, du
\end{align*}
and consequently
\begin{align*}
\triangle_{h_i}^{k+1}(g_i, x_i) 
& = \triangle_{h_i}^1(\triangle_{h_i}^k(g_i, \cdot), x_i)
\\
& = \triangle_{h_i}^1 \biggl( \int_{\mathbb{R}} h_i^k g_i^{(k)}(u) N_{k,h_i}(u - x_i) \, du \biggr)
\\
& = \int_{\mathbb{R}} h_i^k g_i^{(k)}(u) N_{k,h_i}(u - x_i - h_i) \, du 
- \int_{\mathbb{R}} h_i^k g_i^{(k)}(u) N_{k,h_i}(u - x_i) \, du
\\
& = \int_{\mathbb{R}} h_i^k g_i^{(k)}(t+x_i+h_i) N_{k,h_i}(t) \, dt 
- \int_{\mathbb{R}} h_i^k g_i^{(k)}(t+x_i) N_{k,h_i}(t) \, dt
\\
& = \int_{\mathbb{R}} h_i^k (g_i^{(k)}(t+x_i+h_i) - g_i^{(k)}(t+x_i)) N_{k,h_i}(t) \, dt.
\end{align*}
Since $f \in C^{k,\alpha}$ and $\|N_{r,h_i}\|_1 = 1$, we have
\begin{align*}
|\triangle_{h_i}^{k+1}(g_i, x_i)|
& \leq \int_{\mathbb{R}} h_i^k |g_i^{(k)}(t+x_i+h_i) - g_i^{(k)}(t+x_i)| N_{k,h_i}(t) \, dt
\\
& \leq \int_{\mathbb{R}} h_i^k c_L h_i^{\alpha} N_{k,h_i}(t) \, dt
\\
& = c_L h_i^{k+\alpha} \int_{\mathbb{R}} N_{k,h_i}(t) \, dt
\\
& = c_L h_i^{k+\alpha}.
\end{align*} 
This together with \eqref{deco} yields
\begin{align*}
\|\triangle_h^{k+1}(f, x)\|_{\infty} 
\leq \sum_{i=1}^d \|\triangle_{h_i}^{k+1}(g_i, x_i)\|_{\infty} 
\leq \sum_{i=1}^d c_L h_i^{k+\alpha}.
\end{align*}
Taking the supremum over 
both sides of the above inequality with respect to $\|h\|_2\leq t$, we get
\begin{align*}
\omega_{k+1,L_{\infty}(\mathbb{R}^d)}(f,t) 
\leq c_L d \, t^{k+\alpha},
\end{align*}
which completes the proof.
\end{proof}

\begin{proof}[of Proposition \ref{thm::convolution}]
For any $x \in \mathbb{R}^d$, there holds
\begin{align*}
K_j * f(x) 
& = \int_{\mathbb{R}^d} \sum_{\ell=1}^{k+1} {{k+1} \choose \ell} (-1)^{1-\ell} \frac{1}{\ell^d} 
\biggl( \frac{2}{\gamma_j^2 \pi} \biggr)^{d/2} 
\exp \biggl( - \frac{2 \|x - t\|_2^2}{\ell^2 \gamma_j^2} \biggr) f(t) \, dt
\\
& = \int_{\mathbb{R}^d} \biggl( \frac{2}{\gamma_j^2 \pi} \biggr)^{d/2} 
\exp \biggl( - \frac{2\|h\|_2^2}{\gamma_j^2} \biggr) 
\biggl( \sum_{\ell=1}^{k+1} {k+1 \choose \ell} (-1)^{1-\ell} f(x+\ell h) \biggr) \, dh.
\end{align*}
Let $\mathcal{S}_{\nu} := \{ A \in \mathbb{R}^d : \nu(\mathbb{R}^d \setminus A) = 0 \}$, then we have
\begin{align*}
\biggl\| \sum_{j \in J} \eins_{A_j} \cdot (K_j * f) - f \biggr\|_{L_{\infty}(\nu)}
= \sup_{A \in S_{\nu}} \sup_{x \in A} \biggl| \sum_{j \in J} \eins_{A_j}(x) (K_j * f)(x) - f(x) \biggr|.
\end{align*}
Using the equality
\begin{align*}
\int_{\mathbb{R}^d} \exp \biggl( - \frac{2 \|h\|_2^2}{\gamma_j^2} \biggr) \, dh 
= \biggl( \frac{\gamma_j^2 \pi}{2} \biggr)^{d/2},
\end{align*}
we obtain
\begin{align*}
f(x) = \int_{\mathbb{R}^d} \biggl( \frac{2}{\gamma_j^2 \pi} \biggr)^{d/2}  \exp \biggl( - \frac{2 \|h\|_2^2}{\gamma_j^2} \biggr) f(x) \, dh
\end{align*}
and consequently
\begin{align*}
& \biggl| \sum_{j \in J} \eins_{A_j}(x) K _j * f(x) - f(x) \biggr| 
\\
& = \biggl| \sum_{j \in J} \eins_{A_j}(x) 
\int_{\mathbb{R}^d}  \biggl( \frac{2}{\gamma_j^2 \pi} \biggr)^{\frac{d}{2}} 
\exp \biggl( - \frac{2\|h\|_2^2}{\gamma_j^2} \biggr) 
\biggl( \sum_{\ell=0}^{k+1} {{k+1} \choose \ell} (-1)^{2(k+1)+1-\ell} f(x+\ell h) \biggr) \, dh \biggr|
\\
& = \biggl| \sum_{j \in J} \eins_{A_j}(x) (-1)^{k+1+1} \int_{\mathbb{R}^d} 
\biggl( \frac{2}{\gamma_j^2 \pi} \biggr)^{\frac{d}{2}} 
\exp \biggl( - \frac{2\|h\|_2^2}{\gamma_j^2} \biggr) 
\triangle_h^{k+1}(f, x) \, dh \biggr|
\\
& = \sum_{j \in J}  \eins_{A_j}(x) \biggl| \int_{\mathbb{R}^d}
\biggl( \frac{2}{\gamma_j^2 \pi} \biggr)^{\frac{d}{2}} 
\exp \biggl( - \frac{2\|h\|_2^2}{\gamma_j^2} \biggr) \triangle_h^{k+1}(f,x) \, dh \biggr|
\\
& \leq \sum_{j \in J} \eins_{A_j}(x) 
\int_{\mathbb{R}^d} \biggl( \frac{2}{\gamma_j^2 \pi} \biggr)^{\frac{d}{2}} 
\exp \biggl( - \frac{2\|h\|_2^2}{\gamma_j^2} \biggr) 
|\triangle_h^{k+1}(f,x)| \, dh
\\
& \leq \int_{\mathbb{R}^d} \biggl( \frac{2}{\underline{\gamma}^2 \pi} \biggr)^{\frac{d}{2}} 
\exp \biggl( - \frac{2\|h\|_2^2}{\overline{\gamma}^2} \biggr) 
\sum_{j \in J} \eins_{A_j}(x) |\triangle_h^{k+1}(f,x)| \, dh.
\end{align*}
Since $A \in \mathcal{S}_{\nu}$, we have
\begin{align*}
\biggl\| \sum_{j \in J} \eins_{A_j} \cdot (K_j * f) - f \biggr\|_{L_{\infty}(\nu)}
& = \int_{\mathbb{R}^d} \biggl( \frac{2}{\underline{\gamma}^2 \pi} \biggr)^{\frac{d}{2}} 
\exp \biggl( - \frac{2\|h\|_2^2}{\overline{\gamma}^2} \biggr) 
\|\triangle_h^{k+1}(f, \cdot) \|_{L_{\infty}(\nu)} \, dh
\\
& \leq \int_{\mathbb{R}^d} \biggl( \frac{2}{\underline{\gamma}^2 \pi} \biggr)^{\frac{d}{2}} 
\exp \biggl( - \frac{2\|h\|_2^2}{\overline{\gamma}^2} \biggr) 
\omega_{k+1,L_{\infty}(\nu)}(f, \|h\|_2) \, dh.
\end{align*}
Lemma \ref{lem::modulusone} implies that for $f \in C^{k, \alpha}$, there holds 
\begin{align*}
\omega_{k+1, L_{\infty}(\nu)}(f, \|h\|_2) \leq c_L d \|h\|_2^{k+\alpha}
\end{align*}
and thus we obtain
\begin{align*}
& \biggl\| \sum_{j \in J} \eins_{A_j} \cdot (K_j * f) - f \biggr\|_{L_{\infty}(\nu)} 
\\
& \leq \int_{\mathbb{R}^d}  \biggl(\frac{2}{\underline{\gamma}^2 \pi} \biggr)^{\frac{d}{2}} 
\exp \biggl( - \frac{2\|h\|_2^2}{\overline{\gamma}^2} \biggr) 
c_L d \|h\|_2^{k+\alpha} \, dh 
\\
& = c_L d \biggl( \frac{2}{\underline{\gamma}^2 \pi} \biggr)^{\frac{d}{2}} 
\int_{\mathbb{R}^d} \exp \biggl( - \frac{2\|h\|_2^2}{\overline{\gamma}^2}\biggr) \|h\|_2^{k+\alpha} \, dh
\\
& \leq c_L d \biggl( \frac{2}{\underline{\gamma}^2 \pi} \biggr)^{\frac{d}{2}}
\biggl( \int_{\mathbb{R}^d} \exp \biggl( - \frac{2\|h\|_2^2}{\overline{\gamma}^2}\biggr) dh \biggr)^{1/2}
\biggl( \int_{\mathbb{R}^d} \exp \biggl( - \frac{2\|h\|_2^2}{\overline{\gamma}^2}\biggr)\|h\|_2^{2(k+\alpha)} dh \biggr)^{1/2} 
\\
& = c_L d \biggl(\frac{2\overline{\gamma}^2}{ \pi\underline{\gamma}^4} \biggr)^{\frac{d}{4}} 
\biggl( \int_{\mathbb{R}^d} \exp \biggl( - \frac{2\|h\|_2^2}{\overline{\gamma}^2}\biggr)\|h\|_2^{2(k+\alpha)} dh \biggr)^{1/2}.
\end{align*}
For any $x \in \mathbb{R}^d$, there holds
\begin{align*}
\|x\|_2 \leq d^{\frac{k + \alpha - 1}{2(k + \alpha)}} \|x\|_{2(k + \alpha)}, 
\end{align*}
where $d^{\frac{k + \alpha - 1}{2(k + \alpha)}}$ is the embedding constant
of $\ell_{2(k + \alpha)}^d$ to $\ell_2^d$. This 
together with the equality 
$\int_{\mathbb{R}} \exp (- \frac{2 x^2}{\gamma^2}) \, dx = (\frac{\gamma^2 \pi}{2})^{1/2}$
implies
\begin{align*}
& \biggl\| \sum_{j \in J} \eins_{A_j} \cdot (K_j * f) - f \biggr\|_{L_{\infty}(\nu)} 
\\
& \leq c_L d \biggl( \frac{2 \overline{\gamma}^2}{\pi \underline{\gamma}^4} \biggr)^{\frac{d}{4}}
\biggl( \int_{\mathbb{R}^d} d^{k + \alpha-1} \sum_{i=1}^d h_i^{2(k + \alpha)} 
\exp \biggl( - \frac{2\|h\|_2^2}{\overline{\gamma}^2}\biggr) \, dh \biggr)^{1/2}
\\
& \leq c_L d \biggl( \frac{2 \overline{\gamma}^2}{\pi \underline{\gamma}^4} \biggr)^{\frac{d}{4}}
\biggl(d^{k + \alpha-1} \int_{\mathbb{R}^d} \sum_{i=1}^d h_i^{2(k+\alpha)} 
\exp \biggl( - \frac{2 \sum_{i=1}^d h_i^2}{\overline{\gamma}^2} \biggr) \, dh \biggr)^{1/2}
\\
& \leq c_L d^{\frac{k + \alpha+1}{2}} 
\biggl( \frac{2 \overline{\gamma}^2}{\pi \underline{\gamma}^4} \biggr)^{\frac{d}{4}}
\biggl( \int_{\mathbb{R}^d} \sum_{i=1}^d h_i^{2(k + \alpha)} 
\prod_{\ell=1}^d \exp \biggl( - \frac{2 h_{\ell}^2}{\overline{\gamma}^2}\biggr) \, d(h_1, \ldots, h_d) \biggr)^{1/2}
\\
& = c_L d^{\frac{k + \alpha+1}{2}} 
\biggl( \frac{2 \overline{\gamma}^2}{\pi \underline{\gamma}^4} \biggr)^{\frac{d}{4}}
\biggl( \sum_{i=1}^d \int_{\mathbb{R}^d} h_i^{2(k + \alpha)} 
\prod_{\ell=1}^d \exp \biggl( - \frac{2 h_{\ell}^2}{\overline{\gamma}^2}\biggr) \, d h_1 \cdots d h_d  \biggr)^{1/2}
\\
&\leq c_L d^{\frac{k + \alpha+1}{2}} 
\biggl( \frac{2 \overline{\gamma}^2}{\pi \underline{\gamma}^4} \biggr)^{\frac{d}{4}} 
\biggl(\sum_{i=1}^d \big(\frac{\overline{\gamma}^2\pi}{2}\big)^{\frac{d-1}{2}} \int_{\mathbb{R}} h_i^{2(k + \alpha)} \exp \biggl( - \frac{2 h_i^2}{\overline{\gamma}^2}\biggr) \, dh_i\biggr)^{1/2}
\\
& = c_L d^{\frac{k + \alpha+1}{2}} 
\biggl( \frac{2 \overline{\gamma}^2}{\pi \underline{\gamma}^4} \biggr)^{\frac{d}{4}} 
\biggl( \frac{\overline{\gamma}^2\pi}{2} \biggr)^{\frac{d-1}{4}} 
\biggl( \sum_{i=1}^d \int_{\mathbb{R}} h_i^{2(k + \alpha)} 
\exp \biggl( - \frac{2 h_i^2}{\overline{\gamma}^2} \biggr) \, dh_i \biggr)^{1/2}
\\
& = c_L d^{\frac{k + \alpha}{2}+1} 
\biggl( \frac{2}{\pi \overline{\gamma}^2} \biggr)^{\frac{1}{4}}  
\biggl( \frac{\overline{\gamma}}{\underline{\gamma}} \biggr)^{\frac{d}{2}}
\biggl( \int_{\mathbb{R}} x^{2(k + \alpha)} 
\exp \biggl( - \frac{2 x^2}{\overline{\gamma}^2}\biggr) \, dx \biggr)^{1/2}.
\end{align*}
With the substitution $x := (\frac{1}{2} \overline{\gamma}^2 u)^{\frac{1}{2}}$ we get $dx = \frac{\overline{\gamma}}{2\sqrt{2u}} \, du$ and therefore
\begin{align*}
\int_{\mathbb{R}} x^{2(k + \alpha)} \exp \biggl( - \frac{2 x^2}{\overline{\gamma}^2} \biggr) \, dx 
& = \int_{\mathbb{R}} \biggl( \frac{1}{2} \overline{\gamma}^2 u \biggr)^{k + \alpha} 
e^{-u} \frac{\overline{\gamma}}{2\sqrt{2u}} \, du
\\
& = 2^{-(k + \alpha) - \frac{3}{2}} \overline{\gamma}^{2(k + \alpha)+1} 
\int_{\mathbb{R}} u^{k + \alpha-\frac{1}{2}} e^{-u} \, du
\\
& = 2^{-(k + \alpha)-\frac{3}{2}} \overline{\gamma}^{2(k + \alpha)+1} 
\Gamma \biggl( k + \alpha+\frac{1}{2} \biggr).
\end{align*}
Consequently, we obtain 
\begin{align*}
& \biggl\| \sum_{j \in J} \eins_{A_j} \cdot (K_j * f) - f \biggr\|_{L_{\infty}(\nu)} 
\\
& \leq c_L d^{\frac{k + \alpha}{2}+1} 
\biggl( \frac{2}{\pi \overline{\gamma}^2}\biggr)^{\frac{1}{4}} 
\biggl( \frac{\overline{\gamma}}{\underline{\gamma}} \biggr)^{\frac{d}{2}} 
2^{-\frac{k + \alpha}{2} - \frac{3}{4}}
\overline{\gamma}^{k + \alpha+\frac{1}{2}} 
\Gamma^{\frac{1}{2}} \biggl( k + \alpha+\frac{1}{2} \biggr)
\\
&= c_L \pi^{-\frac{1}{4}} 2^{-\frac{k + \alpha}{2}-\frac{1}{2}} 
d^{\frac{k + \alpha}{2}+1} \Gamma^{\frac{1}{2}} \biggl( k + \alpha+\frac{1}{2} \biggr)
\biggl( \frac{\overline{\gamma}}{\underline{\gamma}} \biggr)^{\frac{d}{2}} 
\overline{\gamma}^{k + \alpha}
\\
& =: c_{k, \alpha} \biggl( \frac{\overline{\gamma}}{\underline{\gamma}} \biggr)^{\frac{d}{2}} 
\overline{\gamma}^{k + \alpha},
\end{align*}
where the constant
$c_{k, \alpha} := c_L \pi^{-\frac{1}{4}} 2^{-\frac{k + \alpha}{2}-\frac{1}{2}} 
d^{\frac{k + \alpha}{2}+1} \Gamma^{\frac{1}{2}} (k + \alpha+\frac{1}{2})$.
This completes the proof.
\end{proof}

\subsubsection{Proofs Related to Section \ref{sec::SError2}}

\begin{proof}[of Lemma \ref{lem::CoveringNumber}]
Let us first denote 
\begin{align*}
a & := \mathcal{N} \bigl( \lambda_A^{-1/2} B_{\widehat{\mathcal{H}}_A}, \|\cdot\|_{L_2(\mathrm{P}_{X|A})}, \varepsilon_A \bigr) \in \mathbb{N},
\\
b & := \mathcal{N} \bigl( \lambda_B^{-1/2} B_{\widehat{\mathcal{H}}_B}, \|\cdot\|_{L_2(\mathrm{P}_{X|B})}, \varepsilon_B \bigr) \in \mathbb{N}.
\end{align*}
By the definition of covering numbers, 
there exist $a$ functions $\widehat{f}_1, \ldots, \widehat{f}_a \in \lambda_A^{-1/2} B_{\widehat{\mathcal{H}}_A}$ and
$b$ functions
$\widehat{h}_1, \ldots, \widehat{h}_b \in \lambda_B^{-1/2} B_{\widehat{\mathcal{H}}_B}$ such that $\{ \widehat{f}_1, \ldots, \widehat{f}_a \}$ is an $\varepsilon_A$-cover of $\lambda_A^{-1/2} B_{\widehat{\mathcal{H}}_A}$ with respect to $\|\cdot\|_{L_2(\mathrm{P}_{X|A})}$ and $\{ \widehat{h}_1, \ldots, \widehat{h}_b \}$ is an $\varepsilon_B$-cover of $\lambda_B^{-1/2} B_{\widehat{\mathcal{H}}_B}$ with respect to $\|\cdot\|_{L_2(\mathrm{P}_{X|B})}$.
Moreover, for every function $\widehat{g}_A \in \lambda_A^{-1/2} B_{\widehat{\mathcal{H}}_A}$, there exists an $i_A \in \{ 1, \ldots, a \}$ such that 
\begin{align}\label{epsilonA}
\bigl\| \widehat{g}_A - \widehat{f}_{i_A} \bigr\|_{L_2(\mathrm{P}_{X|A})} 
\leq \varepsilon_A,
\end{align}
and for every function $\widehat{g}_B \in \lambda_B^{-1/2} B_{\widehat{\mathcal{H}}_B}$, there exists an $i_B \in \{ 1, \ldots, b \}$ such that 
\begin{align}\label{epsilonB}
\bigl\| \widehat{g}_B - \widehat{h}_{i_B} \bigr\|_{L_2(\mathrm{P}_{X|B})} 
\leq \varepsilon_B.
\end{align}
Then the definition of direct sums implies that for any $g \in B_{\mathcal{H}}$, there exists a function $\widehat{g}_A \in \lambda_A^{-1/2} B_{\widehat{\mathcal{H}}_A}$ and a function $\widehat{g}_B \in \lambda_B^{-1/2} B_{\widehat{\mathcal{H}}_B}$ such that $g = \widehat{g}_A + \widehat{g}_B$. 
This together with \eqref{epsilonA} and \eqref{epsilonB} yields
\begin{align*}
\bigl\| g - (\widehat{f}_{i_A} + \widehat{h}_{i_B}) \bigr\|_{L_2(\mathrm{P}_X)}^2 
& = \bigl\| (\widehat{g}_A - \widehat{f}_{i_A}) 
+ (\widehat{g}_B - \widehat{h}_{i_B}) \bigr\|_{L_2(\mathrm{P}_X)}^2 
\\
& = \bigl\| \widehat{g}_A - \widehat{f}_{i_A} \bigr\|_{L_2(\mathrm{P}_{X|A})}^2 
+ \bigl\| \widehat{g}_B - \widehat{h}_{i_B} \bigr\|_{L_2(\mathrm{P}_{X|B})}^2 
\\
& \leq \varepsilon_A^2 + \varepsilon_B^2 
=: \varepsilon^2.
\end{align*}
Consequently,
$\bigl\{ \widehat{f}_{i_A} + \widehat{h}_{i_B} : 
\widehat{f}_{i_A} \in \{ \widehat{f}_1, \ldots, \widehat{f}_a \} 
\text{ and } 
\widehat{h}_{i_B} \in \{ \widehat{h}_1, \ldots, \widehat{h}_b \}  \bigr\}$
is an $\varepsilon$-net of $\mathcal{H}$ with respect to $\| \cdot \|_{L_2(\mathrm{P}_X)}$. 
By the definition of covering numbers, we then get
\begin{align*}
\mathcal{N}(B_{\mathcal{H}}, \|\cdot\|_{L_2(\mathrm{P}_X)}, \varepsilon) 
\leq \mathcal{N} \bigl( \lambda_A^{-1/2} B_{\widehat{\mathcal{H}}_A}, \|\cdot\|_{L_2(\mathrm{P}_{X|A})}, \varepsilon_A \bigr) 
\cdot \mathcal{N} \bigl( \lambda_B^{-1/2} B_{\widehat{\mathcal{H}}_B}, \|\cdot\|_{L_2(\mathrm{P}_{X|B})}, \varepsilon_B \bigr),
\end{align*}
which proves the assertion.
\end{proof}

\begin{proof}[of Lemma \ref{lem::jointspace}]
Let us first denote
\begin{align*}
a & := \mathcal{N}(\eins_{\pi_h}, \|\cdot\|_{L_2(\mathrm{P}_X)}, \varepsilon)  \in \mathbb{N}, 
\\
b & := \mathcal{N}(B_{\mathcal{H}}, \|\cdot\|_{L_2(\mathrm{P}_X)}, \varepsilon)  \in \mathbb{N}.
\end{align*}
By the definition of covering numbers, 
there exist $a$ functions $f_1, \ldots, f_a \in \eins_{\pi_h}$ and $b$ functions $g_1, \ldots, g_b \in B_{\mathcal{H}}$ such that $\{ f_1, \ldots, f_a \}$ is an $\varepsilon$-cover of $\eins_{\pi_h}$ with respect to $L_2(\mathrm{P}_X)$ and $\{ g_1, \ldots, g_b \}$ is an $\varepsilon$-cover of $B_{\mathcal{H}}$ with respect to $L_2(\mathrm{P}_X)$. 
Moreover, for every function $h \in B_{\mathcal{H}} \circ \eins_{\pi_h}$, there exist an $f \in \eins_{\pi_h}$ and a $g \in B_{\mathcal{H}}$ such that $h = g \circ f$.
The definition of covering numbers implies that
for this function $f$, there exists an $i \in \{ 1, \ldots, a \}$ such that
\begin{align*}
\|f - f_i\|_{L_2(\mathrm{P}_X)} \leq \varepsilon,
\end{align*}
and for this function $g$, there exists an $j \in \{ 1, \ldots, b \}$ such that
\begin{align*}
\|g - g_j\|_{L_2(\mathrm{P}_X)} \leq \varepsilon.
\end{align*}
Consequently, we obtain
\begin{align*}
\|g \circ f - g_j \circ f_i\|_{L_2(\mathrm{P}_X)} 
& = \|g \circ f - g_j \circ f\|_{L_2(\mathrm{P}_X)} 
+ \|g_j \circ f - g_j \circ f_i\|_{L_2(\mathrm{P}_X)}
\\
& = \|(g - g_j) \circ f \|_{L_2(\mathrm{P}_X)} 
+ \|g_j \circ (f - f_i)\|_{L_2(\mathrm{P}_X)}
\\
& \leq \|f\|_{\infty} \|g - g_j\|_{L_2(\mathrm{P}_X)} 
+ \|g_j\|_{\infty} \|f - f_i\|_{L_2( \mathrm{P}_X)}
\\
& \leq (1+\|k_{\gamma}\|_{\infty})\varepsilon
\\
& \leq 2 \varepsilon,
\end{align*}
and thus the assertion is proved.
\end{proof}

The following lemma, which gives the upper bound for the entropy numbers of Gaussian kernels, follows directly from Theorem 6.27 in \cite{StCh08}. For the sake of completeness, we present the proof.

\begin{lemma}\label{thm::entropygaussian}
Let $\mathcal{X} \subset \mathbb{R}^d$, $\mathrm{P}_X$ be a distribution on $\mathcal{X}$ and $A \subset \mathcal{X}$ be such that $\mathring{A} \neq \emptyset$ and such that there exists an Euclidean ball $B \subset \mathbb{R}^d$ with radius $r_B > 0$ containing $A$, i.e., $A \subset B$. Moreover, for $0 < \gamma \leq r_B$, let $\mathcal{H}_{\gamma}(A)$ be the RKHS of the Gaussian RBF kernel $k_{\gamma}$ over $A$. Then, for all $m \in \mathbb{N}^+$, there exists a constant $c_{m,d} > 0$ such that
\begin{align*}
e_i(B_{\mathcal{H}_{\gamma}(A)}, L_2(\mathrm{P}_{X|A})) \leq c_{m,d} \sqrt{\mathrm{P}_X(A)}r_B^m\gamma^{-m}i^{-\frac{m}{d}}, \qquad i>1.
\end{align*}
\end{lemma}

\begin{proof}[of Lemma \ref{thm::entropygaussian}]
Let us consider the commutative diagram
\begin{align*}
\xymatrix{
H_{\gamma}(A) \ar[rr]^{\mathrm{id}} \ar[d]_{\mathcal{I}_B^{-1} \circ \mathcal{I}_A} & & L_2(\mathrm{P}_{X|A}) 
\\
H_{\gamma}(B) \ar[rr]_{\mathrm{id}} & & \ell_{\infty}(B) \ar[u]_{\mathrm{id}}
}
\end{align*}
where the extension operator $\mathcal{I}_A : H_{\gamma}(A) \to H_{\gamma}(\mathbb{R}^d)$ and 
the restriction operator $\mathcal{I}_B^{-1} : H_{\gamma}(\mathbb{R}^d) \to H_{\gamma}(B)$ 
given by Corollary 4.43 in \cite{StCh08} are isometric isomorphisms such that 
$\|\mathcal{I}_B^{-1} \circ \mathcal{I}_A : H_{\gamma}(A) \to H_{\gamma}(B)\| = 1$.

Let $\ell_{\infty}(B)$ be the space of all bounded functions on $B$. Then
for any $f \in \ell_{\infty}(B)$, there holds
\begin{align*}
\|f\|_{L_2(\mathrm{P}_{X|A})} 
& = \biggl(\int_{\mathcal{X}}\eins_A(x)|f(x)|^2\, d\mathrm{P}_X(x) \biggr)^{\frac{1}{2}}
\\
& \leq \|f\|_{\infty} \biggl(\int_{\mathcal{X}}\eins_A(x)\, d\mathrm{P}_X(x) \biggr)^{\frac{1}{2}} 
= \sqrt{\mathrm{P}_X(A)}
\end{align*}
and consequently 
\begin{align*}
\|\mathrm{id} : \ell_{\infty}(B) \to L_2(\mathrm{P}_{X|A})\| \leq \sqrt{\mathrm{P}_X(A)}. 
\end{align*}
This together with (A.38), (A.39) and Theorem 6.27 in \cite{StCh08} implies that for all $i \geq 1$ and $m \geq 1$, there holds
\begin{align*}
& e_i(\mathrm{id} : H_{\gamma}(A) \to L_2(\mathrm{P}_{X|A}))
\\
& \leq \|\mathcal{I}_B^{-1} \circ \mathcal{I}_A : H_{\gamma}(A) \to H_{\gamma}(B)\|
\cdot e_i(\mathrm{id} : H_{\gamma}(B) \to \ell_{\infty}(B)) 
\cdot \|\mathrm{id} : \ell_{\infty}(B) \to L_2(\mathrm{P}_{X|A})\|
\\
& \leq \sqrt{\mathrm{P}_X(A)}c_{m,d}r_B^m\gamma^{-m}i^{-\frac{m}{d}},
\end{align*}
where $c_{m,d}$ is the constant as in Theorem 6.27 in \cite{StCh08}.
\end{proof}

\begin{proof}[of Proposition \ref{lem::jointspacecoveringnumber}]
First of all, note that the restriction operator $\mathcal{I} : B_{\widehat{\mathcal{H}}_j} \to B_{\mathcal{H}_j}$ with $\mathcal{I} \widehat{f} := f$ is an isometric isomorphism. 
Inequality (A.36) in \cite{StCh08} and Lemma \ref{thm::entropygaussian} yield 
\begin{align*}
e_i \bigl( \mathrm{id} : \lambda_{2, j}^{-1/2} B_{\widehat{\mathcal{H}}_j} \to  L_2(\mathrm{P}_{X|A_j}) \bigr) 
& = 2 \lambda_{2, j}^{-1/2} e_i \bigl( \mathrm{id} :  B_{\widehat{\mathcal{H}}_j} \to  L_2(\mathrm{P}_{X|A_j}) \bigr) 
\\
& \leq 2 \lambda_{2, j}^{-1/2} \bigl\| \mathcal{I} : B_{\widehat{H}_j} \to B_{\mathcal{H}_j} \bigr\| 
\cdot e_i \bigl( \mathrm{id} :  B_{\mathcal{H}_j} \to L_2(\mathrm{P}_{X|A_j}) \bigr)
\\
& \leq 2 \lambda_{2, j}^{-1/2} a_j i^{-\frac{1}{2p}},
\end{align*}
where $a_j = \sqrt{\mathrm{P}_X(A_j)}c_{m,d}(\sqrt{d}\cdot\overline{h}_0)^m\gamma_j^{-m}$ and $p = d/(2m)$. Note that $p$ can be arbitrarily small because $m \in \mathbb{N}^+$ is sufficiently large for Gaussian RBF kernel. Then \eqref{EntropyCover} implies that for all $\varepsilon > 0$, there holds 
\begin{align*}
\ln \mathcal{N} \bigl( \lambda_{2,j}^{-1/2} B_{\widehat{\mathcal{H}}_j}, \|\cdot\|_{L_2(\mathrm{P}_{X|A_j})}, \varepsilon \bigr) 
\leq \ln(4) \bigl( 2 \lambda_{2,j}^{-1/2} a_j \bigr)^{2p} \varepsilon^{-2p}.
\end{align*}
For any $A_j \in \pi_H$ with $H\sim \mathrm{P}_H$, obviously we have $\eins_{A_j} \in \eins_{\pi_H} \in \eins_{\pi_h}$, consequently we obtain
\begin{align*}
&\ln \mathcal{N} \bigl( \lambda_{2,j}^{-1/2} B_{\widehat{\mathcal{H}}_j} \circ \eins_{A_j}, \|\cdot\|_{L_2(\mathrm{P}_{X|A_j})}, 2\varepsilon \bigr)
\\
&\leq \ln \mathcal{N} \bigl( \lambda_{2,j}^{-1/2} B_{\widehat{\mathcal{H}}_j} \circ \eins_{\pi_h}, \|\cdot\|_{L_2(\mathrm{P}_{X|A_j})}, 2\varepsilon \bigr)
\\
&= \ln \mathcal{N} \bigl( \lambda_{2,j}^{-1/2} B_{\widehat{\mathcal{H}}_j}, \|\cdot\|_{L_2(\mathrm{P}_{X|A_j})}, \varepsilon \bigr)  + \ln \mathcal{N} \bigl( \eins_{\pi_h}, \|\cdot\|_{L_2(\mathrm{P}_X)}, \varepsilon \bigr) 
\\
& \leq \ln(4) \bigl( 2 \lambda_{2,j}^{-1/2} a_j \bigr)^{2p} \varepsilon^{-2p}+\ln(K (2^d+2) (4e)^{2^d+2} (1/\varepsilon)^{2(2^d+1)}).
\end{align*}
Therefore, we have
\begin{align*}
&\ln \mathcal{N} \bigl( \lambda_{2,j}^{-1/2} B_{\widehat{\mathcal{H}}_j} \circ \eins_{A_j}, \|\cdot\|_{L_2(\mathrm{P}_{X|A_j})}, \varepsilon \bigr)
\\
&\leq \ln(4) \bigl( 4 \lambda_{2,j}^{-1/2} a_j \bigr)^{2p} \varepsilon^{-2p}+\ln(K (2^d+2) (4e)^{2^d+2} (2/\varepsilon)^{2(2^d+1)})
\\
&\leq \ln(4) \bigl( 4 \lambda_{2,j}^{-1/2} a_j \bigr)^{2p} \varepsilon^{-2p} + 2^{d+4} \ln (1/\varepsilon).
\end{align*}
where in the last step we also used the estimate
\begin{align*}
\ln \bigl( K (2^d+2) (4e)^{2^d+2} (2/\varepsilon)^{2^d+1} \bigr)
\leq 8(2^d+2) \ln(1/\varepsilon)
\leq 2^3\cdot 2^{d+1} \ln(1/\varepsilon)
\leq 2^{d+4}\ln(1/\varepsilon),
\end{align*}
which is based on the following inequalities:
\begin{align*}
\ln K 
& \leq \ln (1/\varepsilon),
\\
\ln (2^d+2) 
& \leq 2^d+2 \leq (2^d+2) \ln (1/\varepsilon)
\\
(2^d+2) \ln(4e) 
& \leq (2^d+2) \ln(e^3) 
= 3 (2^d+2)
\leq 3(2^d+2) \ln (1/\varepsilon)
\\
2(2^d+1)\ln(2/\varepsilon)
&= 2(2^d+1)(\ln(2)+\ln(1/\varepsilon))\leq 4(2^d+1)\ln(1/\varepsilon).
\end{align*}
Therefore, there holds
\begin{align}\label{equ::cover}
&\sup_{\varepsilon \in (0,1/\max\{e, K\})} \varepsilon^{2p}  \ln\mathcal{N}( \lambda_{2,j}^{-1/2} B_{\widehat{\mathcal{H}}_j} \circ \eins_{A_j}, \|\cdot\|_{L_2(\mathrm{P}_X )}, \varepsilon) 
\nonumber\\
& \leq \ln(4) \bigl( 4 \lambda_{2,j}^{-1/2} a_j \bigr)^{2p} + 2^{d+4} \varepsilon^{2p}\ln (1/\varepsilon).
\end{align}
Simple analysis shows that the right hand side of \eqref{equ::cover} is maximized at $\varepsilon^* = e^{-1/(2p)}$ and consequently we obtain
\begin{align*}
\ln\mathcal{N}( \lambda_{2,j}^{-1/2} B_{\widehat{\mathcal{H}}_j} \circ \eins_{A_j}, \|\cdot\|_{L_2(\mathrm{P}_X)}, \varepsilon) & \leq (a/\varepsilon)^{2p}
\end{align*}
with the constant $a$ is defined by 
\begin{align*}
a :=  \biggl( \ln(4) \bigl( 4 \lambda_{2,j}^{-1/2} a_j \bigr)^{2p} + \frac{2^{d+4}}{2pe} \biggr)^{\frac{1}{2p}}.
\end{align*}
By \cite[Exercise 6.8]{StCh08}, we have
\begin{align*}
e_i(\lambda_{2,j}^{-1/2} B_{\widehat{\mathcal{H}}_j} \circ \eins_{A_j}, \|\cdot\|_{L_2(\mathrm{P}_X)})
\leq 3^{\frac{1}{2p}} a i^{-\frac{1}{2p}}
\leq \biggl(3\ln(4) \bigl( 4 \lambda_{2,j}^{-1/2} a_j \bigr)^{2p} + \frac{2^{d+6}}{2pe}\biggr)^{\frac{1}{2p}} i^{\frac{1}{2p}},
\end{align*}
which holds for $\mathbb{E}_{D \sim \mathrm{P}^n}e_i(\lambda_{2,j}^{-1/2} B_{\widehat{\mathcal{H}}_j} \circ \eins_{A_j}, \|\cdot\|_{L_2(\mathrm{P}_X)})$ as well. Thus, we have
\begin{align*}
\mathbb{E}_{D \sim \mathrm{P}^n}e_i(\lambda_{2,j}^{-1/2} B_{\widehat{\mathcal{H}}_j} \circ \eins_{A_j}, \|\cdot\|_{L_2(\mathrm{P}_X)})
\leq \biggl(3\ln(4) \bigl( 4 \lambda_{2,j}^{-1/2} a_j \bigr)^{2p} + \frac{2^{d+6}}{2pe}\biggr)^{\frac{1}{2p}} i^{\frac{1}{2p}} 
:= a'_j i^{-\frac{1}{2p}}.
\end{align*}
Using $\|\cdot\|_{\ell_p^m} \leq m^{\frac{1-p}{p}}\|\cdot\|_{\ell_1^m}$, we further get
\begin{align*}
& \biggl( \sum_{j \in \mathcal{I}_H} \max\{a'_j,B \} \biggr)^{2p}
\\
& = \biggl( \sum_{j \in \mathcal{I}_H} 
\max \biggl\{ \biggl( 3 \ln(4) \bigl( 4 \lambda_{2,j}^{-1/2} a_j \bigr)^{2p} + \frac{2^{d+6}}{2pe} \biggr)^{\frac{1}{2p}},
B \biggr\} \biggr)^{2p}
\\
& \leq \biggl( \sum_{j \in \mathcal{I}_H}
\biggl( 3 \ln(4) \bigl( 4 \lambda_{2,j}^{-1/2} a_j \bigr)^{2p} 
+ \frac{2^{d+6}}{2pe} \biggr)^{\frac{1}{2p}} 
+ |\mathcal{I}_H| B \biggr)^{2p}
\\
&\leq \bigg(\sum_{j \in \mathcal{I}_H}2\biggl(3\ln(4) \bigl( 4 \lambda_{2,j}^{-1/2} a_j \bigr)^{2p}\biggr)^{\frac{1}{2p}} + 2 |\mathcal{I}_H| \biggl(\frac{2^{d+6}}{2pe}\biggr)^{\frac{1}{2p}} +|\mathcal{I}_H|B  \bigg)^{2p}
\\
&=2^{2p}\bigg(\sum_{j \in \mathcal{I}_H}\biggl(3\ln(4) \bigl( 4 \lambda_{2,j}^{-1/2} a_j \bigr)^{2p}\biggr)^{\frac{1}{2p}} + |\mathcal{I}_H| \biggl(\frac{2^{d+6}}{2pe}\biggr)^{\frac{1}{2p}} +|\mathcal{I}_H|(B/2)  \bigg)^{2p}
\\
&\leq 2^{2p} \bigg(\sum_{j \in \mathcal{I}_H}3\ln(4) \bigl( 4 \lambda_{2,j}^{-1/2} a_j \bigr)^{2p}+ |\mathcal{I}_H|^{2p} \frac{2^{d+6}}{2pe} +|\mathcal{I}_H|^{2p} \bigg(\frac{B}{2}\bigg)^{2p}\bigg)
\\
&\leq 2^{2p} 3\ln(4) 4^{2p}c_p^{2p}(\sqrt{d}\cdot \overline{h}_0)^d\sum_{j \in \mathcal{I}_H}\lambda_{2,j}^{-p} \mathrm{P}_X(A_j)^p\gamma_j^{-(d+2p)}
\\
&\phantom{=}
+ 2^{2p} |\mathcal{I}_H|^{2p} \frac{2^{d+6}}{2pe} +2^{2p} |\mathcal{I}_H|^{2p} \bigg(\frac{B}{2}\bigg)^{2p}
\\
&\leq 2^{2p}3\ln(4) 4^{2p}c_p^{2p}(\sqrt{d}\cdot \overline{h}_0)^d|\mathcal{I}_H|^{1-p}\bigg(\sum_{j \in \mathcal{I}_H}\lambda_{2,j}^{-1} \mathrm{P}_X(A_j)\gamma_j^{-\frac{d+2p}{p}}\bigg)^p
\\
&\phantom{=}
+2^{2p} |\mathcal{I}_H|^{2p} \frac{2^{d+6}}{2pe} +2^{2p} |\mathcal{I}_H|^{2p} \bigg(\frac{B}{2}\bigg)^{2p},
\end{align*}
which proves the assertion.
\end{proof}

\subsubsection{Proofs Related to Section \ref{sec::OracleKernelCk}}

\begin{proof}[of Proposition \ref{prop::oracleCk}]
Let us denote 
\begin{align} \label{def::rStar}
r^* := \inf_{f \in \mathcal{H}} \lambda_1 \underline{h}_0^q
+ \lambda_2 \|f_{\mathrm{D},\gamma}\|_{\mathcal{H}}^2
+ \mathcal{R}_{L,\mathrm{P}}(f_{\mathrm{D},\gamma}) - \mathcal{R}_{L,\mathrm{P}}^*, 
\end{align}
and for $r > r^*$, define
\begin{align*}
\mathcal{F}_r 
& := \{ f \in \mathcal{H} : \lambda_1 \underline{h}_0^q + \lambda_2 \|f\|_{\mathcal{H}}^2
+ \mathcal{R}_{L,\mathrm{P}}(f) - \mathcal{R}_{L,\mathrm{P}}^* \leq r \},
\\
\widehat{\mathcal{F}}_{j,r} 
& := \{ f \in \widehat{\mathcal{H}}_{\gamma_j} : 
\lambda_1 \underline{h}_0^q/m + \lambda_2 \|f\|_{\widehat{\mathcal{H}}_{\gamma_j}}^2
+ \mathcal{R}_{L_j,\mathrm{P}}(f) - \mathcal{R}_{L,\mathrm{P}}^* \leq r_j \},
\\
\mathcal{H}_r 
& := \{ L \circ \wideparen{f} - L \circ f_{L,\mathrm{P}}^* : f \in \mathcal{F}_r  \}.
\end{align*}
Obviously, for all $r > 0$, there exists $r_1, \ldots, r_m$ such that $\sum_{j=1}^m r_j = r$ and $\mathcal{F}_r = \bigoplus_{j=1}^m \widehat{\mathcal{F}}_{j,r}$. Moreover, the definition \eqref{def::rStar} yields
\begin{align*}
\lambda_2 \|f_{\mathrm{D},\gamma}\|_{\mathcal{H}}^2 
\leq \lambda_1 \underline{h}_0^q + \lambda_2 \|f_{\mathrm{D},\gamma}\|_{\mathcal{H}}^2
+ \mathcal{R}_{L,\mathrm{P}}(f_{\mathrm{D},\gamma}) - \mathcal{R}_{L,\mathrm{P}}^* 
\leq r 
\end{align*}
and consequently we have $\mathcal{F}_r \subset (r/\lambda_2)^{1/2}B_{\mathcal{H}}$. 
Analogously, there holds 
$\lambda_2 \|f_{\mathrm{D}_j,\gamma_j}\|_{\widehat{\mathcal{H}}_{\gamma_j}}^2 \leq r_j$ and thus $\widehat{\mathcal{F}}_{j,r} \subset (r_j/\lambda_2)^{1/2}B_{\widehat{\mathcal{H}}_{\gamma_j}}$,
which implies
\begin{align*}
\mathbb{E}_{D \sim \mathrm{P}^n} e_i(\mathcal{H}_r, L_2(\mathrm{D}))
& \leq |L|_{M,1} \mathbb{E}_{D \sim \mathrm{P}^n} e_i(\mathcal{F}_r, L_2(\mathrm{D}))
\\
& = |L|_{M,1} \sum_{j=1}^m \mathbb{E}_{\mathrm{D}_j \sim \mathbb{P}^{|D_j|}} 
e_{i/m} \bigl( \widehat{\mathcal{F}}_{j,r}, L_2(\mathrm{D}_j) \bigr)
\\
& \leq 2|L|_{M,1} \sum_{j=1}^m (r_j/\lambda_2)^{1/2} a'_j m^{\frac{1}{2p}} i^{-\frac{1}{2p}}
\\
&\leq 2 |L|_{M,1} \biggl( \frac{r}{\lambda_2} \biggr)^{1/2} m^{\frac{1}{2p}} \biggl( \sum_{j=1}^m a'_j \biggr) \cdot i^{-\frac{1}{2p}}.
\end{align*}
Moreover, for $f \in \mathcal{F}_r$, we have
\begin{align*}
\mathbb{E}_{\mathrm{P}} (L \circ \wideparen{f} - L \circ f_{L,\mathrm{P}}^*)^2 
\leq V r^{\vartheta}.
\end{align*}
Consequently, Theorem 7.16 in \cite{StCh08} applied to $\mathcal{H}_r$ shows that 
$\mathbb{E}_{D \sim \mathrm{P}^n} \mathrm{Rad}_D (\mathcal{H}_r, n) \leq \varphi_n(r)$
holds with
\begin{align*}
\varphi_n(r) 
& := \max \biggl\{ C_1(p) 2^p |L|_{M,1}^p \biggl( \frac{r}{\lambda_2} \biggr)^{p/2} m^{\frac{1}{2}}
\bigl( V r^{\vartheta})^{\frac{1-p}{2}} \biggl( \sum_{j=1}^m a'_j \biggr)^p n^{-\frac{1}{2}},
\\
& \qquad \qquad \quad
C_2(p) \bigl( 2^p |L|_{M,1}^p \bigr)^{\frac{2}{1+p}}
\biggl( \frac{r}{\lambda_2} \biggr)^{\frac{p}{1+p}} m^{\frac{1}{1+p}}
\biggl( \sum_{j=1}^m a'_j \biggr)^{\frac{2p}{1+p}} 
B^{\frac{1-p}{1+p}} n^{-\frac{1}{1+p}} \biggr\},
\end{align*}
where $C_1(p)$ and $C_2(p)$ are the constants as in \cite[Theorem 7.16]{StCh08}. 
Simple calculations show that $\varphi_n(r)$ satisfies the condition $\varphi_n(4r) \leq 2 \varphi_n(r)$. 
Moreover, using $2-p-\vartheta+\vartheta p \geq 1$, the condition $r \geq 30 \varphi_n(r)$ is satisfied if
\begin{align*}
r \geq C_p \max \biggl\{ \biggl( \frac{(\sum_{j=1}^m a'_j)^{2p} m}{\lambda_2^p n} \biggr)^{\frac{1}{2-p-\vartheta-\vartheta p}}, \frac{(\sum_{j=1}^m a'_j)^{2p}m}{\lambda_2^p n} \biggr\},
\end{align*}
where the constant $C_p$ is given by
\begin{align*}
C_p := \max \Bigl\{ \Bigl( 30 C_1(p) 2^p |L|_{M,1}^p V^{\frac{1-p}{2}} \Bigr)^{\frac{2}{2-p-\vartheta-\vartheta p}}, 
\Bigl( 30 C_2(p) (2^p |L|_{M,1}^p)^{\frac{2}{1+p}} B^{\frac{1-p}{1+p}} \Bigr)^{p+1} \Bigr\}. 
\end{align*}
If $m \leq (\sum_{j=1}^m a'_j)^{-2p} \lambda_2^p n$, then we have
\begin{align*}
\biggl( \frac{(\sum_{j=1}^m a'_j)^{2p}m}{\lambda_2^p n} \biggr)^{\frac{1}{2-p-\vartheta-\vartheta p}} 
\geq \frac{(\sum_{j=1}^m a'_j)^{2p}m}{\lambda_2^p n},
\end{align*}
which implies that
\begin{align*}
r \geq C_p \biggl( \frac{(\sum_{j=1}^m a'_j)^{2p}m}{\lambda_2^p n} \biggr)^{\frac{1}{2-p-\vartheta-\vartheta p}}.
\end{align*}
For the remaining case when $m \geq \big(\sum_{j=1}^ma_j'\big)^{-2p} \lambda_2^p n$, there holds
\begin{align*}
\lambda_1 (\underline{h}_0^*)^q + \lambda_2 \|f_{\mathrm{D},\gamma}\|_{\mathcal{H}}^2
+ \mathcal{R}_{L,\mathrm{P}}(f_{\mathrm{D},\gamma}) - \mathcal{R}_{L,\mathrm{P}}^*
& \leq \lambda_1 \overline{h}_0^q + \lambda_2 \|f_{\mathrm{D},\gamma}\|_{\mathcal{H}}^2
+ \mathcal{R}_{L,\mathrm{D}}(f_{\mathrm{D},\gamma}) + B
\\
& \leq \lambda_1 \overline{h}_0^q + \mathcal{R}_{L,\mathrm{D}}(0) + B
\\
& \leq \lambda_1 \overline{h}_0^q + 2 B \biggl( \frac{(\sum_{j=1}^m a'_j)^{2p}m}{\lambda_2^p n} 
\biggr)^{\frac{1}{2-p-\vartheta-\vartheta p}}.
\end{align*}
Using $r^* \leq \lambda_1 \underline{h}_0^q + \lambda_2 \|f_0\|_{\mathcal{H}}^2 + \mathcal{R}_{L,\mathrm{P}}(f_0) - \mathcal{R}_{L,\mathrm{P}}^*$, the assertion thus follows from Theorem 7.20 in \cite{StCh08} with $K := \max \{ 2B, 3C_p \}$.
\end{proof}

\subsubsection{Proofs Related to Section \ref{sec::Re2}}

\begin{proof}[of Theorem \ref{thm::convergenceratetree}]
First of all, we bound the approximation error by choosing an appropriate function $f_0 \in \mathcal{H}$. 
Recall that for $j \in \mathcal{I}_H$, the functions $K_j : \mathbb{R}^d \to \mathbb{R}$ is defined as in \eqref{Conv} with $\gamma_j >0$. We then define $f_0$ 
by convolving each $K_j$ with the Bayes decision function $f_{L, \mathrm{P}}^*$, that is,
\begin{align*}
f_0 (x) := \sum_{j\in \mathcal{I}_H} \eins_{A_j}(x) \cdot (K_ j * f_{L, \mathrm{P}}^*)(x), 
\qquad 
x \in \mathbb{R}^d.
\end{align*}
To show that $f_0$ is indeed a suitable function to bound the approximation error, we firstly ensure that $f_0$ is contained in $\widehat{\mathcal{H}}_k$, and then derive bounds for both, the regularization term and the excess risk of $f_0$. 
By Proposition 4.46 in \cite{StCh08}, since $f_{L,\mathrm{P}}^* \in L_2(\mathbb{R}^d)$, we obtain that for every $j \in \mathcal{I}_H$, there holds
\begin{align*}
(K_j * f_{L, \mathrm{P}}^*)_{|A_j} \in \mathcal{H}_{\gamma_j}(A_j)
\end{align*}
with
\begin{align}
\|\eins_{A_j} f_0\|_{\widehat{\mathcal{H}}_{\gamma_j}(A_j)} 
& = \|\eins_{A_j} (K_j * f_{L, \mathrm{P}}^*)\|_{\widehat{\mathcal{H}}_{\gamma_j}(A_j)}
\nonumber\\
& = \bigl\| (K_j * f_{L, \mathrm{P}}^*)_{|A_j} \bigr\|_{\mathcal{H}_{\gamma_j}(A_j)}
\nonumber\\
& \leq (\gamma_j \sqrt{\pi})^{-\frac{d}{2}} 
(2^{k+1} - 1) 
\|f_{L, \mathrm{P}}^*\|_{L_2(\mathbb{R}^d)}.
\label{equ::f0norm}
\end{align}
This implies
\begin{align*}
f_0 = \sum_{j \in \mathcal{I}_H}\eins_{A_j}(K_j*f_{L, \mathrm{P}}^*) \in \mathcal{H}.
\end{align*}
Moreover, Theorem \ref{thm::convolution} yields
\begin{align}
\mathcal{R}_{L, \mathrm{P}}(f_0) - \mathcal{R}_{L, \mathrm{P}}^* 
& = \|f_0 - f_{L, \mathrm{P}}^*\|_{L_2(\mathrm{P}_{X})}^2
\nonumber\\
& = \biggl\| \sum_{j\in \mathcal{I}_H} \eins_{A_j}(K_j * f_{L, \mathrm{P}}^*) - f_{L, \mathrm{P}}^* \biggr\|_{L_2(\mathrm{P}_X)}^2 
\nonumber\\
& \leq c_{k, \alpha}^2 \biggl( \frac{\overline{\gamma}}{\underline{\gamma}} \biggr)^d \overline{\gamma}^{2(k+\alpha)},
\label{equ::f0excess}
\end{align}
where $c_{k, \alpha}$ is a constant only depending on $k$ and $\alpha$.

Next, we derive a bound for $\|L \circ f_0\|_{\infty}$. 
Using Theorem 2.3 in \cite{Optimal13}, we obtain that for any $x \in \mathcal{X}$, there holds
\begin{align*}
|f_0(x)| 
& = \biggl| \sum_{j \in \mathcal{I}_H} \eins_{A_j}(x) \cdot (K_j * f_{L, \mathrm{P}}^*)(x) \biggr| 
\\
& \leq \sum_{j\in \mathcal{I}_H} \eins_{A_j}(x) |K_j * f_{L, \mathrm{P}}^*(x)| 
\\
& \leq (2^{k+1} - 1) \|f_{L, \mathrm{P}}^*\|_{L_{\infty}(\mathbb{R}^d)}
\end{align*}
and consequently we have
\begin{align}
\|L \circ f_0\|_{\infty} 
& = \sup_{(x, y) \in \mathcal{X} \times \mathcal{Y}} |L(y, f_0(x))|
\nonumber\\
& \leq \sup_{(x, y) \in \mathcal{X} \times \mathcal{Y}} \bigl( M^2 + 2 M |f_0(x)| + |f_0(x)|^2 \bigr)
\nonumber\\
& \leq 4^{k+1} \max \bigl\{ M^2, \|f_{L, \mathrm{P}}^*\|_{L_{\infty}(\mathbb{R}^d)}^2 \bigr\} =: B_0.
\label{equ::B0}
\end{align}
Proposition \ref{lem::jointspacecoveringnumber} together with Proposition \ref{prop::oracleCk} yields 
\begin{align*}
& \lambda_1 (\underline{h}_0^*)^{q}+ \lambda_2\|f_{D,\gamma}\|_{\mathcal{H}}^2+\mathcal{R}_{L,\mathrm{P}}(f_{D,\gamma})-\mathcal{R}_{L,\mathrm{P}}^*
\\
& \lesssim \lambda_1\underline{h}_0^{q}+\underline{h}_0^{-d}\lambda_{2,j}\gamma_j^{-d}+\overline{\gamma}^{2(k+\alpha)}+\overline{h}_0^{-d} \lambda_{2,j}^{-p} \gamma_j^{-(d+2p)}n^{-1}+n^{-1}.
\end{align*}
Choosing 
\begin{align*}
\overline{h}_{0,n} := n^0, 
\qquad 
\gamma_{n,j} := n^{-\frac{1}{2(k+\alpha)+d}}, 
\qquad 
\lambda_{1,n} := n^{-\frac{1}{2(k+\alpha)+d}}, 
\qquad 
\lambda_{2,n,j} := n^{-1},
\end{align*}
we obtain
\begin{align*}
\lambda_1 (\underline{h}_0^*)^{q}+ \lambda_2\|f_{D,\gamma}\|_{\mathcal{H}}^2+\mathcal{R}_{L,\mathrm{P}}(f_{D,\gamma})-\mathcal{R}_{L,\mathrm{P}}^*
\lesssim n^{-\frac{2(k+\alpha)}{2(k+\alpha)+d}+\xi},
\end{align*}
where $\xi = p+\frac{2p}{2(k+\alpha)+d}$ can be arbitrarily small with $p$ being infinitesimal. This proves the assertion.
\end{proof}

\begin{proof}[of Theorem \ref{thm::convergencerateforest}]
Let $f_{\mathrm{D}, \gamma,\mathrm{E}}$ 
be the kernel histogram transform ensembles given by (\ref{LSVMEHR}).
Using Jensen's inequality, we have  
\begin{align*}
\mathcal{R}_{L,\mathrm{P}}(f_{\mathrm{D}, \gamma,\mathrm{E}}) -  \mathcal{R}_{L,\mathrm{P}}^* 
&= \int_{\mathcal{X}} \biggl( \frac{1}{T}\sum_{t=1}^T f_{\mathrm{D}, \gamma, H_t}  - f_{L, \mathrm{P}}^* \biggr)^2 d \mathrm{P}_X\\
&\leq \frac{1}{T}\sum_{t=1}^T \int_{\mathcal{X}} \biggl( f_{\mathrm{D}, \gamma, H_t} - f_{L, \mathrm{P}}^* \biggr)^2 d \mathrm{P}_X\\
&= \frac{1}{T}\sum_{t=1}^T \biggl(\mathcal{R}_{L,\mathrm{P}}(f_{\mathrm{D}, \gamma, H_t}) - \mathcal{R}_{L,\mathrm{P}}^* \biggr).
\end{align*}
Then the union bound together with Theorem \ref{thm::convergenceratetree} yields
\begin{align*}
& \mathrm{P} \Bigl( \mathcal{R}_{L,\mathrm{P}}(f_{\mathrm{D},  \gamma,\mathrm{E}} ) 
- \mathcal{R}_{L,\mathrm{P}}^*  > c \cdot n^{- \frac{2\alpha}{2\alpha+d}+\xi} \Bigr) 
\\
& \leq \sum_{t=1}^T \mathrm{P} \Bigl( \mathcal{R}_{L,\mathrm{P}}(f_{\mathrm{D},\gamma,H_t} ) - \mathcal{R}_{L,\mathrm{P}}^*  > c \cdot n^{- \frac{2(k+\alpha)}{2(k+\alpha)+d}+\xi} 
\Bigr)
\leq T e^{-\tau}
\end{align*}
where the constant $c$ is as in Theorem \ref{thm::convergenceratetree}. 
As a result, there holds
\begin{align*}
\mathcal{R}_{L,\mathrm{P}}(f_{\mathrm{D}, \gamma, \mathrm{E}} ) 
- \mathcal{R}_{L,\mathrm{P}}^*  
\leq c \cdot n^{- \frac{2(k+\alpha)}{2(k+\alpha)+d}+\xi}
\end{align*}
with probability $\nu_n$ at least $1-3e^{-\tau}$, where $c$ is a constant depending on $M$, $k$, $\alpha$, $p$, and $T$.
\end{proof}

\section{Conclusion}\label{sec::conclusion}

By conducting a statistical learning treatment, this paper studies the large-scale regression problem with histogram transform estimators. Based on partition induced by random histogram transform and various different kinds of embedded regressors, this nonparametric strategy provides an effective solution taking full advantage of large diversity of the random histogram transform, the nature of ensemble learning, and the efficiency of vertical methods. By decomposing the error term into approximation error and estimation error, the insights from the theoretical perspective are threefold: 
First, different regression estimators NHT and KHT are applied when the Bayes decision function $f_{L,\mathrm{P}}^*$ is assumed to satisfy different H\"{o}lder continuity assumptions. Secondly, almost optimal convergence rates are established within the regularized empirical risk minimization framework for NHTE in $C^{0,\alpha}$ and for KHTE in $C^{k,\alpha}$ with $k \geq 2$. Thirdly, for the space $C^{1,\alpha}$, the lower bound established in Theorem \ref{thm::optimalForest} illustrates the exact benefits of ensembles over single estimator. Last but not least, several numerical simulations are conducted to offer evidence to support our theoretical results and comparative real-data experiments with other state-of-the-art regression estimators demonstrate the accuracy of our algorithm. In this paper, we explain the phenomenon that ensemble estimators outperform single ones in the space $C^{1,\alpha}$ with respect to constant embedded regressors, from the perspective of learning rate. And we're now exploring other possible interpretations, which applies to more general function space such as $C^{k,\alpha}$ and smoother regressors such as SVMs, for this phenomenon from other aspects, information theory, for instance.

\small{}

\end{document}